%% file: main.tex
\documentclass{article}
\usepackage{preprint,times}

\usepackage[utf8]{inputenc}
\usepackage[T1]{fontenc}
\usepackage{xurl}
\usepackage{booktabs}
\usepackage{wrapfig}
\usepackage{amsfonts}
\usepackage{nicefrac}
\usepackage{microtype}
\usepackage{xcolor}
\usepackage{amsthm,thmtools}

\usepackage{etoolbox,verbatim}

\input{preamble/_preamble_includes}
\title{Conditional Flow Models Convert Epistemic Uncertainty into Apparent Aleatoric Variance}

\author{%
  Navodita Sharma\textsuperscript{1}\thanks{Corresponding author. Email: \href{mailto:navodita@mit.edu}{\texttt{navodita@mit.edu}}.}\quad
  Abhishek Shetty\textsuperscript{2}\quad
  Max Simchowitz\textsuperscript{3}\thanks{Equal advising.}\quad
  Pulkit Agrawal\textsuperscript{1}\footnotemark[2]\\[0.5em]
  \normalfont\small\textsuperscript{1}Massachusetts Institute of Technology\\
  \normalfont\small\textsuperscript{2}Georgia Institute of Technology\\
  \normalfont\small\textsuperscript{3}Carnegie Mellon University
}
\date{}

\begin{document}

\maketitle

\begin{abstract}

\input{body/abstract}

\end{abstract}

\input{body/intro}
\input{body/related}
\input{body/prelim}
\input{body/body}
\input{body/conclusion}

\bibliographystyle{references}
\bibliography{refs}

\appendix

\input{body/app_smoothing.tex}
\input{body/app_coverage.tex}
\input{body/app_variance_inflation.tex}
\input{body/app_related_work.tex}

\end{document}

%% file: preamble/_preamble_includes.tex
\input{preamble/color_defs}

\input{preamble/header}

\input{preamble/theorem_styles}

\input{preamble/characters}

\input{preamble/general_macros}

\input{preamble/algnames}

%% file: preamble/color_defs.tex
\usepackage{xcolor}
\usepackage{tcolorbox}
\tcbuselibrary{skins,breakable}

\tcbset{
  aibox/.style={
      width=\linewidth,
      top=6pt,
      bottom=0pt,
      left=1.5pt,
      right=1.5pt,
      colback=blue!60!gray!5,
      colframe=black,
      colbacktitle=black,
      enhanced,
      center,
      attach boxed title to top left={yshift=-0.1in,xshift=0.15in},
      boxed title style={boxrule=0pt,colframe=white,},
    }
}
\newtcolorbox{AIbox}[2][]{aibox,title=#2,#1}

\definecolor{primalcolor}{HTML}{A60000}
\definecolor{contrarycolor}{HTML}{00A6A6}
\definecolor{darkcontrarycolor}{HTML}{004C4C}
\definecolor{lightblue}{HTML}{2970CC}
\definecolor{lightpurple}{HTML}{673147}
\definecolor{ForestGreen}{HTML}{FF5733}
\definecolor{myred}{HTML}{AA4A44}
\definecolor{hyppurple}{HTML}{800080}

\newcommand{\linkcolor}{darkcontrarycolor}
\newcommand{\urlcolor}{darkcontrarycolor}
\newcommand{\citecolor}{darkcontrarycolor}

\newcommand{\thmcolordark}{red!30!black}

%% file: preamble/header.tex
\usepackage[T1]{fontenc}
\usepackage{courier}
\usepackage{capt-of}
\usepackage[font=small,labelfont=bf]{caption}

\usepackage{natbib}
\usepackage{scalefnt}
\usepackage{amssymb,upgreek,bm}
\usepackage{mathtools}
\usepackage[english]{babel}  
\usepackage{multirow,tcolorbox,tikz-cd,turnstile,xspace,xparse}
\usepackage{relsize,breakcites,appendix}
\usepackage{longtable,tablefootnote,booktabs,float,makecell}
\usepackage[normalem]{ulem}
\usepackage{tabu}
\usepackage[shortlabels]{enumitem} 
\usepackage{footnote}
\usepackage{boxedminipage}
\usepackage{multirow,nicefrac}
\usepackage{verbatim} %

\usepackage[colorlinks=true,linkcolor=\linkcolor,urlcolor=\urlcolor,citecolor=\citecolor,breaklinks,hypertexnames=false]{hyperref}

\usepackage{prettyref}

\usepackage[capitalise,nameinlink]{cleveref}

    \usepackage{amsthm}

\usepackage{thm-restate}

    \usepackage{mathrsfs}
    \usepackage{subcaption}

\usepackage{placeins}

\DeclareMathAlphabet{\mathbfsf}{\encodingdefault}{\sfdefault}{bx}{n}

\numberwithin{equation}{section}

%% file: preamble/theorem_styles.tex
\usepackage{thmtools}

\makeatletter
\@ifundefined{newcounteralias}{}{%
  \renewcommand\thmt@autorefsetup{%
    \@xa\def\csname\thmt@envname autorefname\@xa\endcsname\@xa{\thmt@thmname}%
  }%
}
\makeatother

\Crefname{equation}{Eq.}{Eqs.}
\Crefname{assumption}{Assumption}{Assumptions}
\Crefname{condition}{Condition}{Conditions}
\Crefname{claim}{Claim}{Claims}
\Crefname{property}{Property}{Properties}
\Crefname{construction}{Construction}{Constructions}

\declaretheoremstyle[
    headformat=\normalfont\textcolor{\thmcolordark}{\bfseries\NAME\,\NUMBER}\NOTE,%
    notefont={\normalfont\textcolor{\thmcolordark}{\bfseries}}, 
    notebraces={}{},
    bodyfont=\normalfont\itshape,
    spaceabove = 6pt,
    spacebelow = 6pt,
    ]{coloredthmversion}

\declaretheoremstyle[
    headformat=\normalfont\textcolor{\thmcolordark}{\bfseries\NAME\,\NUMBER}\NOTE,%
    bodyfont=\normalfont\itshape,
    spaceabove = 6pt,
    spacebelow = 6pt,
    ]{coloredthm}

\declaretheoremstyle[
    headformat=\normalfont\textcolor{\thmcolordark}{\bfseries\NAME\,\NUMBER}\NOTE,%
    bodyfont=\normalfont,
    spaceabove = 6pt,
    spacebelow = 6pt,
    ]{coloreddef}

    \theoremstyle{coloredthmversion}

  \theoremstyle{coloredthm}
  \newtheorem{theorem}{Theorem}
  \newtheorem{lemma}{Lemma}[section]
  \newtheorem{corollary}{Corollary}[section]
  \newtheorem{proposition}[lemma]{Proposition}

\newtheorem*{thminformal*}{Informal Theorem}

    \theoremstyle{coloreddef}
    \newtheorem{definition}{Definition}[section]
    
    \newtheorem{remark}{Remark}[section]

\newtheorem{assumption}{Assumption}[section]
\newtheorem{condition}{Condition}[section]

\makeatletter
\newcommand{\neutralize}[1]{\expandafter\let\csname c@#1\endcsname\count@}
\makeatother

    \newtheoremstyle{named}{}{}{\itshape}{}{\bfseries}{}{.5em}{\Cref{#3} {\normalfont (informal)} }{}
    \theoremstyle{named}
    
    \theoremstyle{plain}

\newtheorem*{theorem*}{Theorem}
\newtheorem*{lemma*}{Lemma}
\newtheorem*{corollary*}{Corollary}
\newtheorem*{proposition*}{Proposition}
\newtheorem*{claim*}{Claim}
\newtheorem*{fact*}{Fact}
\newtheorem*{observation*}{Observation}
\newtheorem*{definition*}{Definition}
\newtheorem*{remark*}{Remark}
\newtheorem*{example*}{Example}

%% file: preamble/characters.tex
\def\ddefloop#1{\ifx\ddefloop#1\else\ddef{#1}\expandafter\ddefloop\fi}
\def\ddef#1{\expandafter\def\csname bb#1\endcsname{\ensuremath{\mathbb{#1}}}}
\ddefloop ABCDEFGHIJKLMNOPQRSTUVWXYZ\ddefloop

\def\ddefloop#1{\ifx\ddefloop#1\else\ddef{#1}\expandafter\ddefloop\fi}
\def\ddef#1{\expandafter\def\csname frak#1\endcsname{\ensuremath{\mathfrak{#1}}}}
\ddefloop ABCDEFGHIJKLMNOPQRSTUVWXYZ\ddefloop

\def\ddefloop#1{\ifx\ddefloop#1\else\ddef{#1}\expandafter\ddefloop\fi}
\def\ddef#1{\expandafter\def\csname fr#1\endcsname{\ensuremath{\mathfrak{#1}}}}
\ddefloop ABCDEFGHIJKLMNOPQRSTUVWXYZ\ddefloop

\def\ddefloop#1{\ifx\ddefloop#1\else\ddef{#1}\expandafter\ddefloop\fi}
\def\ddef#1{\expandafter\def\csname eul#1\endcsname{\ensuremath{\EuScript{#1}}}}
\ddefloop ABCDEFGHIJKLMNOPQRSTUVWXYZ\ddefloop

\def\ddefloop#1{\ifx\ddefloop#1\else\ddef{#1}\expandafter\ddefloop\fi}
\def\ddef#1{\expandafter\def\csname scr#1\endcsname{\ensuremath{\mathscr{#1}}}}
\ddefloop ABCDEFGHIJKLMNOPQRSTUVWXYZ\ddefloop

\def\ddefloop#1{\ifx\ddefloop#1\else\ddef{#1}\expandafter\ddefloop\fi}
\def\ddef#1{\expandafter\def\csname b#1\endcsname{\ensuremath{\mathbf{#1}}}}
\ddefloop ABCDEFGHIJKLMNOPQRSTUVWXYZ\ddefloop

\def\ddefloop#1{\ifx\ddefloop#1\else\ddef{#1}\expandafter\ddefloop\fi}
\def\ddef#1{\expandafter\def\csname bhat#1\endcsname{\ensuremath{\hat{\mathbf{#1}}}}}
\ddefloop ABCDEFGHIJKLMNOPQRSTUVWXYZ\ddefloop

\def\ddefloop#1{\ifx\ddefloop#1\else\ddef{#1}\expandafter\ddefloop\fi}
\def\ddef#1{\expandafter\def\csname btil#1\endcsname{\ensuremath{\tilde{\mathbf{#1}}}}}
\ddefloop ABCDEFGHIJKLMNOPQRSTUVWXYZ\ddefloop

\def\ddefloop#1{\ifx\ddefloop#1\else\ddef{#1}\expandafter\ddefloop\fi}
\def\ddef#1{\expandafter\def\csname bst#1\endcsname{\ensuremath{\mathbf{#1}^\star}}}
\ddefloop ABCDEFGHIJKLMNOPQRSTUVWXYZ\ddefloop

\def\ddefloop#1{\ifx\ddefloop#1\else\ddef{#1}\expandafter\ddefloop\fi}
\def\ddef#1{\expandafter\def\csname bst#1\endcsname{\ensuremath{\mathbf{#1}^\star}}}
\ddefloop abcdeghijklmnopqrstuvwxyz\ddefloop

\def\ddefloop#1{\ifx\ddefloop#1\else\ddef{#1}\expandafter\ddefloop\fi}
\def\ddef#1{\expandafter\def\csname bhat#1\endcsname{\ensuremath{\hat{\mathbf{#1}}}}}
\ddefloop abcdefghijklmnopqrstuvwxyz\ddefloop

\def\ddefloop#1{\ifx\ddefloop#1\else\ddef{#1}\expandafter\ddefloop\fi}
\def\ddef#1{\expandafter\def\csname b#1\endcsname{\ensuremath{\mathbf{#1}}}}
\ddefloop abcdeghijklnopqrstuvwxyz\ddefloop

\def\ddefloop#1{\ifx\ddefloop#1\else\ddef{#1}\expandafter\ddefloop\fi}
\def\ddef#1{\expandafter\def\csname barb#1\endcsname{\ensuremath{\bar{\mathbf{#1}}}}}
\ddefloop abcdefghijklmnopqrstuvwxyz\ddefloop

\def\ddef#1{\expandafter\def\csname c#1\endcsname{\ensuremath{\mathcal{#1}}}}
\ddefloop ABCDEFGHIJKLMNOPQRSTUVWXYZ\ddefloop
\def\ddef#1{\expandafter\def\csname h#1\endcsname{\ensuremath{\widehat{#1}}}}
\ddefloop ABCDEFGHIJKLMNOPQRSTUVWXYZ\ddefloop
\def\ddef#1{\expandafter\def\csname hc#1\endcsname{\ensuremath{\widehat{\mathcal{#1}}}}}
\ddefloop ABCDEFGHIJKLMNOPQRSTUVWXYZ\ddefloop
\def\ddef#1{\expandafter\def\csname t#1\endcsname{\ensuremath{\widetilde{#1}}}}
\ddefloop ABCDEFGHIJKLMNOPQRSTUVWXYZ\ddefloop
\def\ddef#1{\expandafter\def\csname tc#1\endcsname{\ensuremath{\widetilde{\mathcal{#1}}}}}
\ddefloop ABCDEFGHIJKLMNOPQRSTUVWXYZ\ddefloop

%% file: preamble/general_macros.tex
\newcommand{\eqstyle}{\textstyle}

\usepackage{tcolorbox}
\tcbuselibrary{skins,breakable}

\tcbset{
  aibox/.style={
      width=\linewidth,
      top=6pt,
      bottom=0pt,
      left=1.5pt,
      right=1.5pt,
colback=blue!60!gray!5,
colframe=black,
      colbacktitle=black,
      enhanced,
      center,
      attach boxed title to top left={yshift=-0.1in,xshift=0.15in},
      boxed title style={boxrule=0pt,colframe=white,},
    }
}

\DeclareMathOperator*{\argmin}{arg\,min}

\let\Pr\relax
\DeclareMathOperator{\Pr}{\mathbb{P}}

\newcommand{\Exp}{\mathbb{E}}

\newcommand{\ballkr}[1][r]{\cB_{k}(r)}

\newcommand{\rmd}{\mathrm{d}}

\newcommand{\True}{\mathsf{True}}

\DeclareMathSymbol{\shortminus}{\mathbin}{AMSa}{"39}

\newcommand{\R}{\mathbb{R}}

%% file: body/abstract.tex
Conditional Flow Models (CFMs) are increasingly being used in place of deterministic predictors across machine learning, including for tasks that were classically formulated as regression. Given their success in such tasks, we ask the question: \textit{what useful information do the conditional distributions produced by CFMs carry beyond what a point predictor would provide}?

We identify a systematic phenomenon: in many natural regimes, CFM samples routinely have larger variance than the true conditional distribution, even in nearly deterministic regression problems. This effect is most pronounced  where the model has not resolved the dependence on the conditioning variable, and thus the additional variance can be viewed as converting \emph{epistemic uncertainty} in fitting the ground truth conditional distribution into apparent \emph{aleatoric variance} in the generated samples. Through controlled synthetic and real-data experiments, and comparisons with explicit smoothing estimators, we show that the induced variance tracks model underfitting of complex or hard-to-fit targets, and decreases with additional training or data.
We uncover surprising structure in this excess conditional variance that seems unique to CFMs and absent in the other standard probabilistic models that we test.
We provide a theoretical account of the phenomenon in a simplified setting and demonstrate a consequent practical use-case of conditional flow-matching: designing high-coverage predictive sets that cover the ground-truth regression target.

%% file: body/intro.tex
\section{Introduction}

The role of generative models in machine learning has expanded rapidly. 
Rather than returning a single point prediction, these models produce a conditional distribution over possible outputs. 
Perhaps surprisingly, they have gained traction even in settings where predicting a single value, rather than distribution over values, suffices. These setting, which we will call \emph{point-prediction} problems,  encompass predicting high-performing actions in robotics \citep{chi2023diffusion},  accurate estimates of a policy value function in reinforcement learning \citep{agrawalla2025floqtrainingcriticsflowmatching}, and numerous other  scientific estimation  and time-series problems \citep{klein2025cellflow, shen2023non, tashiro2021csdi}.

In this work, we focus on conditional flow-matching models (CFMs), which produce samples by integrating a learned velocity field from base noise to the output space, and suggest that these models confer a complementary benefit: they can \emph{\textbf{naturally express estimation uncertainty}}, even in point-prediction problem in which observations are nearly \emph{noiseless}. %

In the idealized setting, the optimal flow $v^\star$ transports the base distribution to the true conditional law \citep{lipman2023flow,albergo2023building}. 
In practice, however, the flow field is parametrized by a neural network $v_\theta$ that is learned imperfectly, due to a combination of limited training data, network capacity, and incomplete optimization (e.g.\ early stopping, popular in the robot learning community). 
Thus the learned conditional distribution is shaped not only by the target distribution, but also by the inductive biases caused by fitting the ground-truth vector-field learning with a neural network. In particular, neural networks tend to oversmooth high-frequency features during learning \citep{rahaman2019spectral,xu2020frequency, tancik2020fourier}, causing the learned vector-field $v_\theta$ to underfit the ground-truth $v^\star$. %

We demonstrate, both across domains and in natural theoretical setups, that under this common smoothing behavior, conditional flow models \emph{convert epistemic uncertainty into apparent aleatoric variance}. That is, the approximation error or uncertainty about the dependence of the generated variable on the conditioning variable is expressed as a larger spread in the learned conditional distribution. We show that this excess spread has substantial structure, resulting in the distributions learned by conditional flow models being approximately calibrated. This structure is unique to CFMs, and absent in the standard probabilistic models that we tested.

Beyond scientific interest, this has actionable consequences. One can train conditional flow-matching models to fit potentially \emph{deterministic target functions} and, as long as the model does not overfit, use the spread of their resulting samples to define confidence sets containing the ground-truth function, with no conformal correction. 

Across the application domains we study, including RL policy evaluation, cell perturbation prediction, and robot policy learning, predictive sets built this way are approximately calibrated empirically in the underfitting regime, a property the variance-inflation phenomenon enables. A practitioner receives an uncertainty estimate ``for free'' alongside the point prediction.

\begin{figure}[h!]
    \centering
    \includegraphics[width=1.0\textwidth]{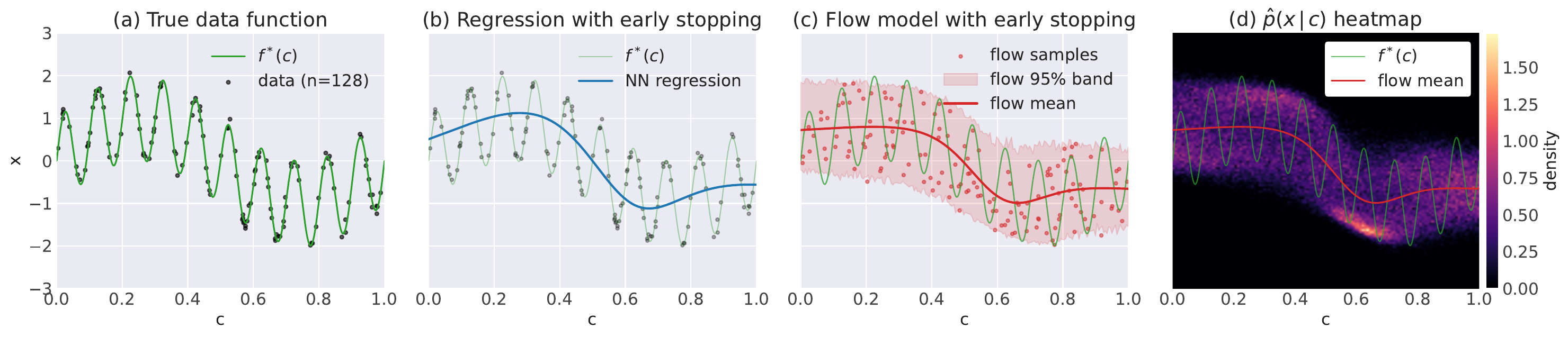}
    \caption{Flow models vs Regression during underfitting. } \label{fig:flow_vs_reg}
\end{figure}

More concretely, we make the following contributions:
\begin{itemize}[leftmargin=6.5pt]
    \item We identify a \textbf{surprisingly general} relationship between learning error in the conditional flow field and the resulting generated samples: we empirically observe that when a learned CFM underfits the ground-truth velocity field as a function of the conditioning variable, its generated samples have conditional variance larger than that of the true distribution, in expectation over the conditioning variable. This stands in sharp contrast to deterministic regression, where underfitting manifests only as oversmoothing of the predictive mean.
    \item We trace this variance inflation to a \textbf{single mechanism}: smoothing bias in vector-field learning. All our tested inductive biases that smooth the learned conditional vector field along the conditioning variable, including early stopping and explicit regularization, produce the same effect. We demonstrate this empirically on synthetic regression problems, controlled smoothing estimators, and high-dimensional offline RL benchmarks. %
    We uncover \textbf{surprising structure} in this excess conditional variance that seems unique to CFMs and absent in four standard probabilistic models that we test. We identify an empirical consequence of this structured inflated variance -- the distribution learned by CFMs is approximately calibrated in the underfitting regime. 
    \item We propose an idealized setting of regularized smoothing to \textbf{theoretically study} the variance inflation phenomenon. Under an equivalent-kernel approximation to regularized vector-field learning, the learned conditional distribution is a local mixture of true conditionals around the query point. This yields a clean bias--variance relation: on average, the learned conditional variance is at least the true conditional variance plus the squared bias of the learned conditional mean. This simplified setting serves as a useful explanatory model to understand the variance inflation phenomenon and its structure empirically observed in neural CFMs. We also discuss the validity of this approximation to neural CFMs and the regimes where the approximation fails.
    
    \item We give a \textbf{practical method} for leveraging this variance inflation in CFMs for predictive coverage. Across three application domains, policy evaluation, cell perturbation prediction, and behavior cloning for robotics, predictive sets built directly from CFM samples are approximately calibrated in the underfitting regime, and are competitive with or tighter than conformal-prediction baselines.
\end{itemize}

%% file: body/related.tex
\newcommand{\atsign}{\makeatletter @ \makeatother }
\section{Related Work}

\paragraph{Flow matching and conditional transport.}
Flow matching and related stochastic-interpolant and rectified-flow methods learn continuous-time transports by regressing a time-dependent vector field along a path between base noise and data \citep{lipman2023flow, liu2023flow, albergo2025stochasticinterpolantsunifyingframework}. Conditional flow matching indexes this transport by side information, yielding a model of the conditional law of an output given a conditioning variable. Conditional Flow Models (CFMs) have been applied broadly in robotics \citep{chisari2024learningroboticmanipulationpolicies, rouxel2024flowmatchingimitationlearning, zhang2024flowpolicyenablingfastrobust} and scientific modeling \citep{klein2025cellflow, miller2024flowmm}.

\paragraph{Inductive biases and generalization in diffusion and flow models.}
Several recent works study why diffusion or score-based models generalize rather than simply memorizing the training data. Recent work has attributed this behavior to several sources of bias: variance in the denoising score-matching target \citep{vastola2025generalization}, architectural bias \citep{kamb2025analytic}, and selective underfitting of the learned score \citep{song2025selective}. These papers are closest in spirit to ours, however, they focus on unconditional generation, whereas we study conditional generation and ask how underfitting with respect to the conditioning variable changes the conditional law. We discuss in detail how our work differs from existing literature on inductive biases of flow/diffusion models in Appendix \ref{app:relation_prior_work}.

\paragraph{Implicit regularization, spectral bias, and smoothing.}
Our explanation is based on smoothing bias in the learned conditional vector field. Neural networks are known to preferentially fit low-frequency or smoother structure early in training, a phenomenon often called spectral bias \citep{rahaman2019spectral,xu2020frequency}. In wide-network regimes, training dynamics can be approximated by kernel gradient flow, connecting early stopping and finite-time training to classical regularization in function space \citep{jacot2018neural}.
Related ideas appear throughout nonparametric regression, smoothing splines, kernel ridge regression, and inverse problems, where regularized estimators can be described by equivalent kernels or linear smoothing operators. %
We build on this viewpoint, but study a different object: not the bias of a point estimator, but the distribution obtained after integrating a smoothed conditional vector field.

%% file: body/prelim.tex
\section{Preliminaries}

\newcommand{\Lfm}{\cL_{\textnormal{\textsc{fm}}}}
\newcommand{\Lhatfm}{\hat\cL_{\textnormal{\textsc{fm}}}}
\newcommand{\ddt}{\frac{\rmd}{\rmd t}}
\newcommand{\Xspace}{\R^{d_{x}}}
\newcommand{\Cspace}{\R^{d_{c}}}

We study conditional generative modeling. 
Given samples from a joint distribution $p^\star$ over random variables $(X,C) \in \Xspace \times \Cspace$, where $X$ is the generated variable and $C$ is the conditioning variable, conditional generative models learn to generate new samples from the conditional distribution of $X$ given $C = c$, denoted $p^\star(\cdot \mid c)$. Flow models represent such distributions as \textbf{continuous normalizing flows} (CNF) and ODEs transporting a simple initial distribution to a terminal distribution approximating $p^\star$. Here, we consider the conditional variant.

\begin{definition}[Continuous Normalizing Flow] \citep{chen2018neural} Given a joint distribution $p_1$ over $(X,C)$, an initial distribution $p_0$ over $(X,C)$, and conditional vector field $v: [0,1] \times \Cspace \times \Xspace \to \Xspace$ mapping $(t,c,x) \mapsto v_t(c,x)$,
we say that $(v,p_0)$ is a \textbf{conditional continuous normalization flow} (CCNF) for $p_1$ if, for every $c \in \Cspace$, the random process $(X_t)_{t \in [0,1]}$ defined by the 
\begin{align}
    \eqstyle X_0 \sim p_0(\cdot \mid C = c), \quad \ddt X_t = v_t(c,X_t) \label{eqw:CNF}
\end{align}
satisfies $X_1 \sim p_1(\cdot \mid c)$. We let $\Phi_t^v(X_0\mid c)$ denote the deterministic function, called a flow-map, mapping $X_0$ to $X_t$ in \Cref{eqw:CNF}. 
Given $(p_0,v)$, we let $p_t^{v}(\cdot \mid c,p_0)$ denote the distribution of $X_t \mid c$ marginalized over $X_0 \sim p_0(\cdot \mid c)$ under \Cref{eqw:CNF}. Note then that, by definition, $(v,p_0)$ is a CCNF for $p_1^v(\cdot \mid \cdot, p_0)$. We suppress dependence on $p_0$ when clear from context.
\end{definition}

\textbf{Flow-matching} \citep{lipman2023flow, albergo2023building, liu2023flow} provides a generic recipe for learning (C)CNFs from data. Let $p_0$ and differentiable coefficients $\alpha_t,\beta_t:[0,1] \to \R$ satisfying $\alpha_0 = \beta_1 = 1$ and $\beta_0 = \alpha_1 = 0$ both be given. The induced (linear) \textbf{stochastic interpolant} is the random variable $X_t = \alpha_t X_0 + \beta_t X_1$, where $X_0 \sim p_0(\cdot \mid c)$ and $X_1 \sim p^\star(\cdot \mid c)$. Flow-matching is based on the following observation:
\begin{proposition}[\cite{albergo2025stochasticinterpolantsunifyingframework}]\label{prop:Lfm}Define the vector field $v^\star$ via $v_t^\star(c,x) = \Exp[\dot X_t \mid X_t=x, C=c]$, where $\dot X_t = \dot \alpha_t X_0 + \dot \beta_t X_1$. Then, $(v^\star, p_0)$ is a CNF for $p^\star$. Moreover, $v^\star$ satisfies $v^\star \in \argmin_{v} \Lfm(v)$, where the $\Lfm$ is the flow-matching loss
\begin{align}
    \eqstyle \Lfm(v) := \Exp_{(X,C) \sim p^\star} \Exp_{t \sim \mathrm{Unif}([0,1]),X_0 \sim p_0(\cdot \mid c)}\Exp[\|v_t(C,X_t) - \dot X_t\|^2]
\end{align}
\end{proposition}
Motivated by \Cref{prop:Lfm}, flow matching trains a velocity field $v_\theta = (v_{\theta,t}(c,x))$ by (approximately) minimizing an empirical approximation to the flow matching loss $\Lfm$.

\textbf{Smoothing/Underfitting Regime.} In this work, we focus on checkpoints whose learned function can be viewed as a smoothed approximation of the true function.
Neural networks tend to learn broad trends before finer variation \citep{rahaman2019spectral}. Early stopping and regularization can leave that finer variation unresolved. Even the checkpoint selected using best validation loss can be from this regime (Fig. \ref{fig:prob_methods_density_val_selected} shows a simple example), when the target function is too complex to learn exactly with the available data and model capacity. Since practical use-cases of CFMs involve modeling complex conditional dependencies, the selected model can lie in the smoothing regime. Our central finding is that, in this regime, CFMs have some empirical coverage properties. These can be useful at checkpoints selected using standard selection rules: since even when the learned conditional mean cannot capture the full dependence on the conditioning input, practitioners can rely on the fact that the learned distribution will be approximately calibrated, due to this inductive bias of CFMs.

%% file: body/body.tex
\section{Structured Variance Inflation in Underfit Conditional Flows}
\label{sec:variance_inflation}

In this section, we uncover three surprising phenomena about underfit conditional flow models: First, the resulting samples have \textbf{larger conditional variance} than the ground-truth distribution (in expectation over the conditioning variable), even in nearly deterministic regression problems. Second, this variance inflation is a consequence of \textbf{smoothing bias} in the learned vector field: all tested mechanisms that smooth the conditional vector field along the conditioning variable produce the same effect. Third, this inflated conditional variance is \textbf{highly structured}, and this structure results in underfit CFMs being approximately calibrated. We demonstrate the phenomenon on a synthetic 1D problem (Figure \ref{fig:flow_vs_reg}), establish that it is driven by smoothing (Figure \ref{fig:reg_trends}), and investigate its structure and coverage consequences (Section~\ref{sec:excess_var_structure}), where we show that the standard probabilistic regression baselines we study do not exhibit the same behavior.

\begin{figure}[h]
    \centering
    \begin{subfigure}[t]{1.0\textwidth}
        \centering
        \includegraphics[width=1.0\textwidth]{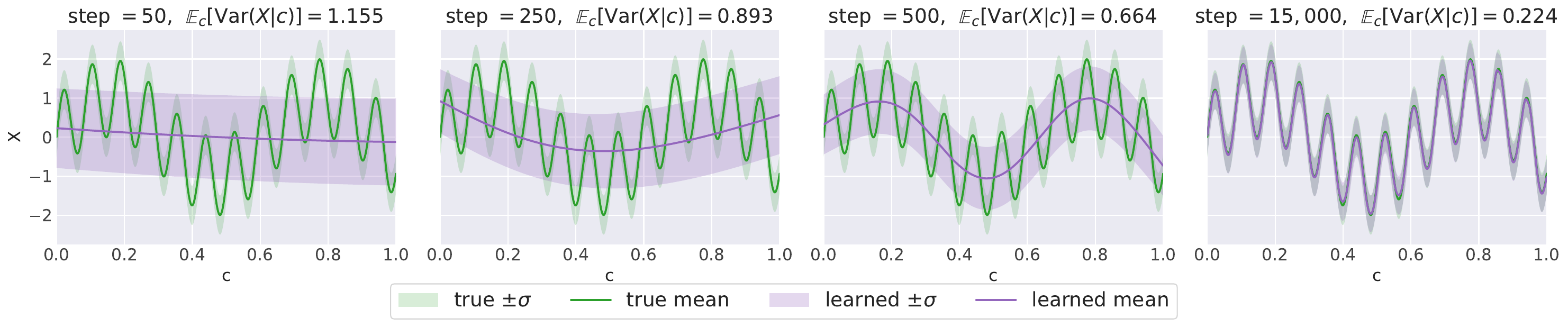}
        \caption{NN vector field: [Left to Right] Distributions learned with increasing training time}
    \end{subfigure}
    \begin{subfigure}[t]{1.0\textwidth}
        \centering
        \includegraphics[width=1.0\textwidth]{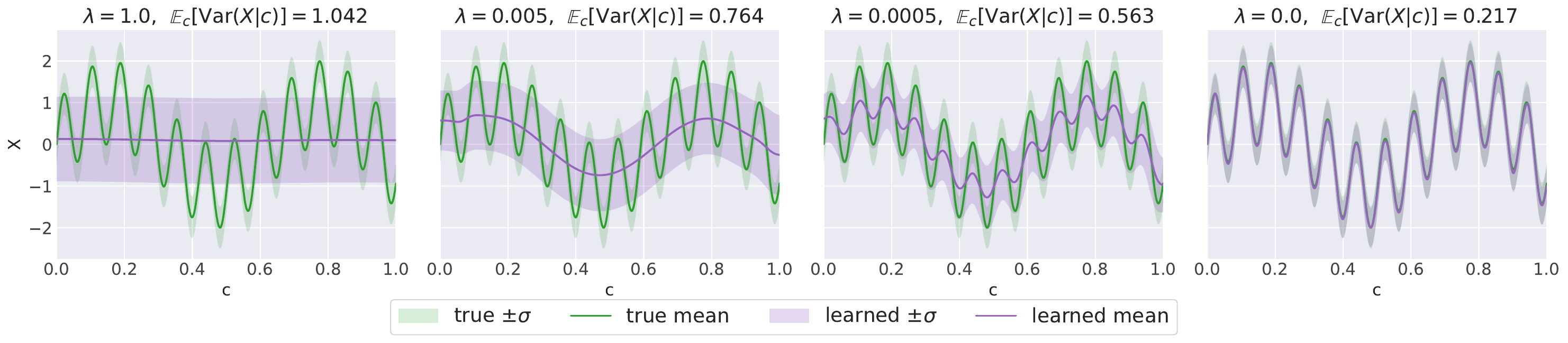}
        \caption{NN vector field with explicit regularization penalty $\lambda (\partial_c v )^2$}
    \end{subfigure}
    \begin{subfigure}[t]{1.0\textwidth}
        \centering
        \includegraphics[width=1.0\textwidth]{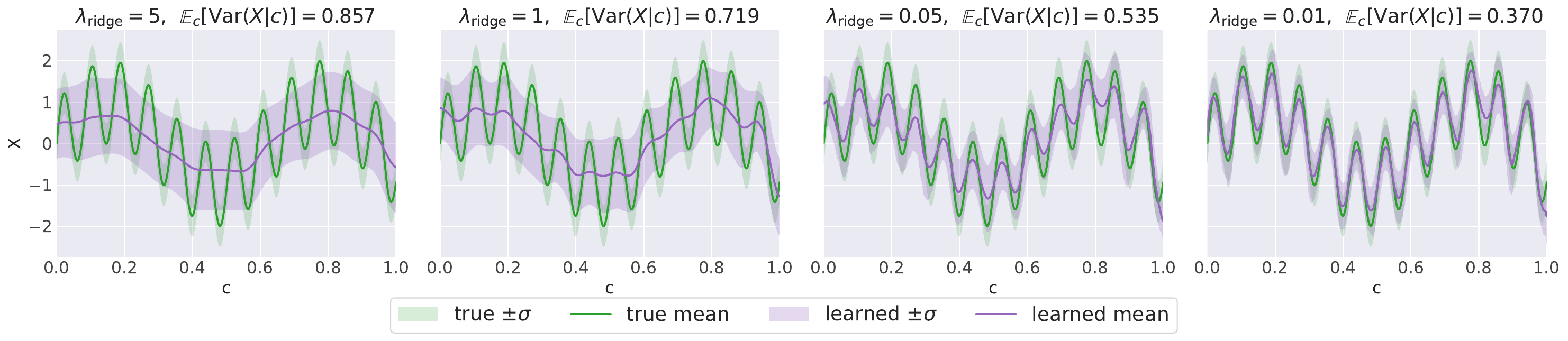}
        \caption{Vector field learning with kernel ridge regression with penalty $\lambda$}
    \end{subfigure}
    \caption{Smoothing bias in vector-field learning induces variance inflation. In all three cases, the average learned conditional variance decreases monotonically as smoothing decreases.\\True data: $C \sim \mathcal{U}([0,1]), \, X \mid C=c \sim \mathcal{N}\!\left(\sin(10c) + \sin(75c),\, \sigma^2 \right), \, \sigma=0.5$.} \label{fig:reg_trends}
\end{figure}

Figure \ref{fig:flow_vs_reg}, given in the introduction, illustrates a one-dimensional synthetic regression problem in which the target function contains both low and high-frequency structure. When a neural-network regressor is trained with early stopping, underfitting leads to oversmoothing of the predictive mean; the learned function captures the coarse trend but fails to recover the finer oscillations. 

For a conditional flow model trained under the same regime, the learned conditional mean is similarly oversmoothed, but unlike the deterministic regressor, the flow also learns substantial additional variance, visible in the wide 95\% band around the mean. 
Even in a nearly deterministic problem, underfitting in the conditional vector field appears not only as bias in the learned mean, but also as extra variance in the generated conditional distribution. 
Panel (d) shows the full learned conditional density $\hat p(x \mid c)$, which is broad and roughly bimodal in regions where $f^\star$ (the conditional mean) has high-frequency oscillations within a smoothing neighborhood; the flow's conditional distribution effectively averages over the local oscillations of $f^\star$. This finding is not specific to this example, and Appendix \ref{sec:add_synthetic} shows that similar phenomena arise across a number of synthetic target functions.

\textbf{Why do we call this uncertainty epistemic?} The excess variance arises from unresolved dependence of the target on the conditioning variable: the learned model smooths over distinctions between nearby conditions that the true conditional mean resolves. This uncertainty is therefore reducible and model-dependent, rather than irreducible randomness in the data-generating process. As the model resolves more of this structure, for example with additional data or training, both the smoothing bias and the excess variance decrease. However, in a CFM this unresolved variation appears as spread within the learned conditional distribution, and hence as apparent aleatoric variance. From the learned distribution alone, this model-induced spread is not distinguishable from genuine aleatoric variation.

\subsection{Variance inflation tracks smoothing bias}

Figure \ref{fig:reg_trends} reveals that the phenomenon in Figure \ref{fig:flow_vs_reg} is not tied to a single estimator or training recipe. We compare three mechanisms that induce smoothing in the learned conditional vector field: early stopping of a neural network which biases training towards low-frequency components of the target \citep{yao2007early, rahaman2019spectral, tancik2020fourier}, an
\emph{explicit roughness penalty} $\lambda \,(\partial_c v)^2$ in the
conditioning variable as used in classic non-parametric regression \citep{wahba1990spline, green1994nonparametric}, and kernel ridge regression with an RKHS-norm regularization penalty $\lambda$ \citep{caponnetto2007optimal}.

All these methods induce smoothing of the conditional mean function, as can be observed from the plotted learned mean at different regularization strengths (Figure \ref{fig:reg_trends}). In all three cases, this smoothing is accompanied by another consistent effect: the learned conditional variance increases as the learned mean becomes smoother. The average learned conditional variance (reported in each panel title) decreases monotonically as smoothing decreases.

Our key finding, corroborated across additional functions in Appendix \ref{sec:add_synthetic},  is that the key object is the s\textit{moothing bias in vector-field learning}, not flow matching alone. Flow matching simply provides a mechanism through which this smoothing bias is expressed as conditional spread. 
We verify that this phenomenon also transfers to higher dimensions by experimenting on a high-dimensional benchmark used in Offline RL in Appendix \ref{sec:policy_eval_pca}.

\begin{figure}[h]
    \centering
    \includegraphics[width=0.6\linewidth]{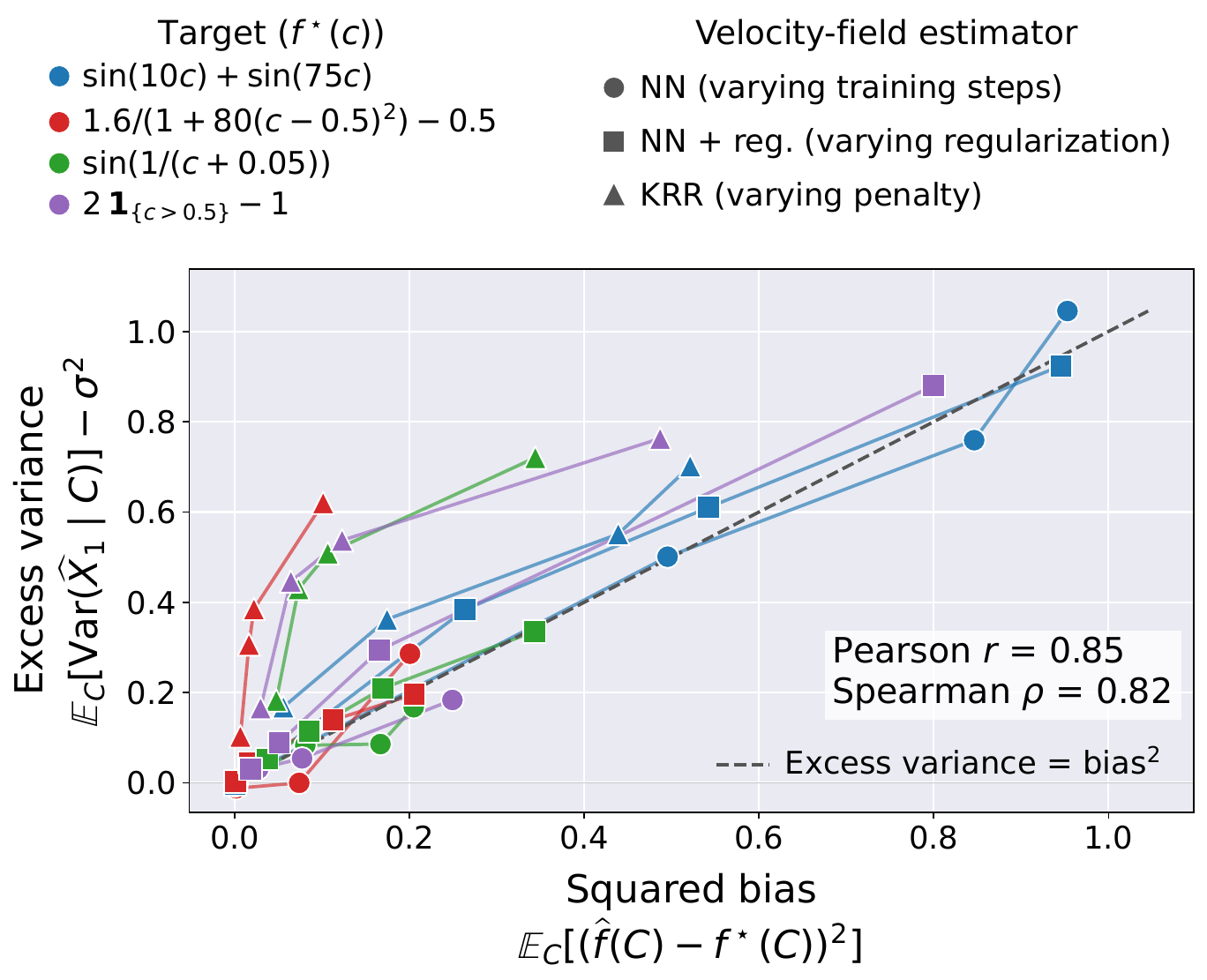}
    \caption{Expected excess conditional variance vs expected squared bias across the tested functions, velocity-field estimators, and smoothing levels.}
    \label{fig:bias_variance_corr}
\end{figure}

Also note in Figure \ref{fig:reg_trends} that the induced variance is often largest precisely when the learned mean is most biased. Figure~\ref{fig:bias_variance_corr} summarizes this relationship (between bias in the mean function and excess mean conditional variance) across the functions and velocity-field estimators considered in Figures~\ref{fig:reg_trends}, \ref{fig:sin_inverse}, \ref{fig:peak}, and~\ref{fig:step}. Each point corresponds to a particular function, estimator, and smoothing (regularization) level. Both excess conditional variance and squared bias of the learned conditional mean are averaged over conditioning inputs. Pooling these settings, the two quantities are strongly positively correlated (Pearson $r=0.85$, Spearman $\rho=0.82$), with points generally lying near or above the line $\text{excess variance}=\text{squared bias}$.

\subsection{Investigating the structure of excess variance}
\label{sec:excess_var_structure}

We now investigate the structure of the excess variance learned by CFMs, and compare with four standard probabilistic methods: Gaussian NLL regression, MC Dropout \citep{gal2016dropout}, deep ensembles \citep{lakshminarayanan2017simple}, and variational BNN \citep{blundell2015weight}, which can also absorb unresolved (epistemic) information into predictive uncertainty in the smoothing regime. We train on the synthetic task as in Figure \ref{fig:reg_trends}, and select checkpoints with comparable expected squared bias in their conditional means (Figure \ref{fig:prob_methods_density}).

All four baselines concentrate their inflated predictive spread around the learned smoothed conditional mean. CFMs instead place substantial mass near the oscillation envelope, which is the high density region of the true distribution. This reveals surprising structure in how CFMs distribute excess variance, beyond simply increasing its magnitude; and the baselines lack this structure. In Appendix \ref{app:prob_methods_additional_exp} we provide additional figures showing how these learned distributions evolve across training, and also compare distributions selected using best-validation checkpoint rule.

\begin{figure}[h!]
    \centering
    \includegraphics[width=1.0\textwidth]{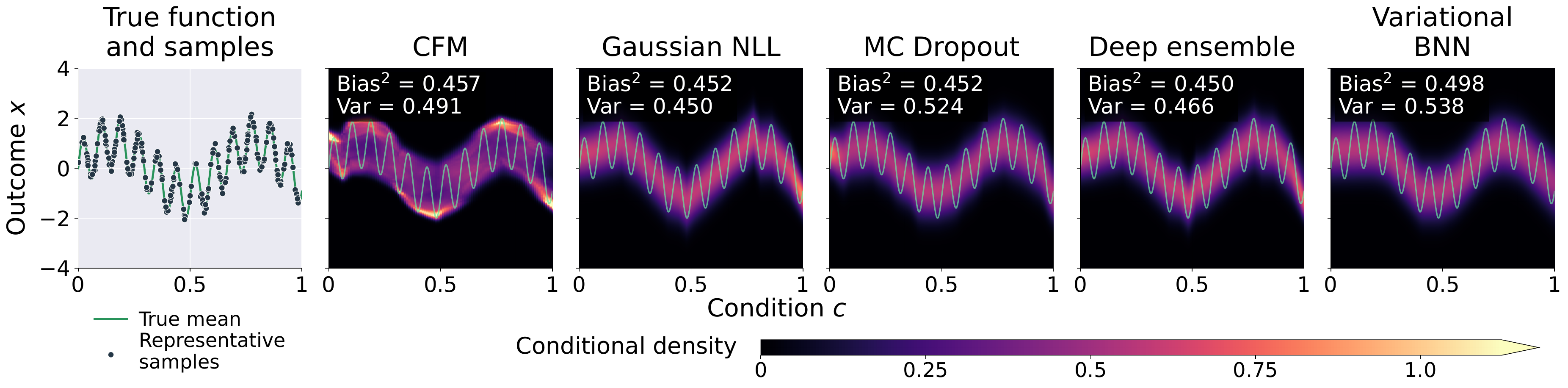}
    \caption{CFM assigns more mass to regions of higher density under the true distribution; baselines learn centered distributions. Training data: $C \sim \mathcal{U}([0,1]), \, X \mid C=c \sim \mathcal{N}\!\left(\sin(10c) + \sin(75c),\, \sigma^2 \right), \, \sigma=0.1$} \label{fig:prob_methods_density}
\end{figure}

On further analysis, we uncover another interesting empirical property: \textit{this excess variance structure results in CFMs being approximately calibrated throughout the smoothing regime} (Figure \ref{fig:synth_coverage_comparison}). None of the classical methods show this property.

\textbf{Predictive coverage.} For a probabilistic model $\hat p_\theta(x\mid c)$ and nominal level $\alpha\in(0,1)$, let $C_{\theta,\alpha}(c)$ denote a predictive set containing approximately $\alpha$ probability mass. We measure empirical coverage on $\mathcal D=\{(x_i,c_i)\}_{i=1}^{|\mathcal D|}$ as
\(
\widehat{\mathrm{Cov}}_\alpha(\mathcal D)
=
\frac{1}{|\mathcal D|}
\sum_{i=1}^{|\mathcal D|}
\mathbf 1\{x_i\in C_{\theta,\alpha}(c_i)\}.
\)
For all five methods, we construct $C_{\theta,\alpha}(c)$ as the central $\alpha$-mass interval of the learned conditional distribution, and report coverage at $\alpha\in\{0.75,0.90,0.95\}$. We also report Calibration Error (CE) metric (\ref{app:coverage_metrics}, lower is better) which summarizes agreement between nominal and achieved coverage across a wider grid of $\alpha$-values. Figure \ref{fig:synth_coverage_comparison} shows that CFM coverage stays near nominal, whereas the baselines deviate significantly. We report all experimental details in Appendix \ref{app:synthetic_experimental_details}.

\begin{figure}[h!]
    \centering
    \includegraphics[width=1.0\textwidth]{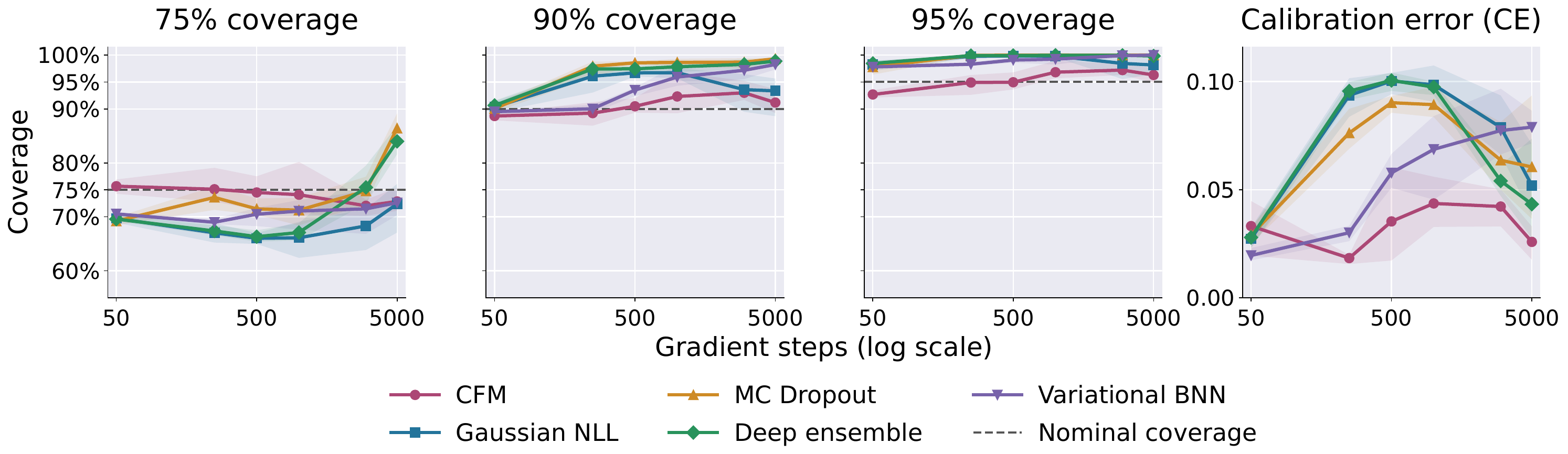}
    \caption{CFM stays approximately calibrated, baselines deviate significantly. CFM has lowest consistent CE.} \label{fig:synth_coverage_comparison}
\end{figure}

\section{An Idealized Model for Variance Inflation} \label{sec:theory}

\begin{figure}[h!]
    \centering
    \includegraphics[width=1.0\textwidth]{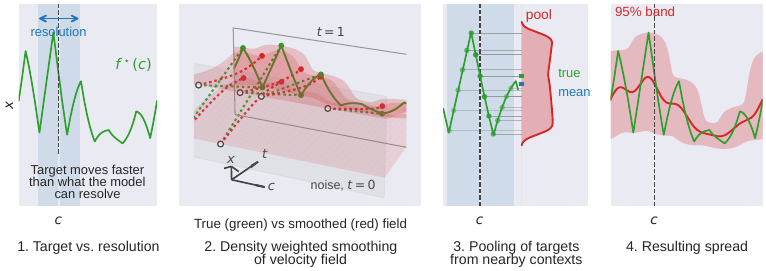}
    \caption{Underfitting of the velocity field leads to local mixing of conditional distributions, not just smoothing of the mean.} \label{fig:smoothing_local_mixture}
\end{figure}

Due to finite data, capacity constraints on the learned network, and optimization choices (e.g. early stopping), $v_\theta$ is an imperfect approximation of $v^\star$. Because the precise form of this approximation is difficult to characterize for neural function approximators, we study the following idealized population-level \textbf{regularized flow-matching problem}.

\begin{align}
v_\lambda \in \arg\min_{v\in\mathcal V} \Bigg\{ \mathbb E\left[\left\|v_T(C,X_T)-\dot X_T\right\|^2\right] +
\lambda\mathbb E\left[\left\|(L_c v_T)(C,X_T)\right\|^2\right]\Bigg\}.
\label{eq:reg_fm}
\end{align}

Here $\mathcal V$ is the admissible function space of conditional vector fields $v:[0,1]\times\Cspace\times\Xspace\to\Xspace$ with finite loss and regularization terms. For almost every $(t,x)$, the slice $c\mapsto v_t(c,x)$ belongs to the domain of $L_c$; admissibility is imposed separately on these slices. The operator $L_c$ is a linear operator inducing smoothing in the \emph{conditioning variable}. Specifically, for each fixed $(t, x)$, $L_c$ acts on the function $c \mapsto v_t(c, x)$ while treating $t$ and $x$ as fixed. This matches the regime of interest: the learned vector field may depend flexibly on the flow state $x$ and on time $t$, but its dependence on the conditioning variable $c$ is smoothed. Concrete examples include first-derivative penalties $L_c = \nabla_c$, higher-order smoothness penalties $L_c = \nabla_c^k$, and RKHS-type penalties on $c \mapsto v_t(c, x)$ \citep{wahba1990spline}. 
It is well-known that such regularized regression problems have equivalent kernel representations \citep{scholkopf2001generalized}. This objective serves as an idealized model of neural-network smoothing bias; related regularized models have been studied for conditional diffusion models \citep{pfrommer2025conditionaldiffusionmodelactually}. Appendix \ref{app:smoothing} provides further justification for studying this model.

Since $L_c$ only induces smoothing in the conditioning variable $c$ and treats $t$ and $x$ as fixed, \Cref{eq:reg_fm} can be optimized independently for each slice corresponding to a fixed $(t,x)$. Hence, up to a multiplicative constant independent of $g$, the regularized least-squares problem for the slice $g(c):=v_t(c,x)$ is
\begin{align}
\hat g_{t,x}\in \arg\min_g
\left\{
\int_{\mathcal C}\left(\|g(c)-v^\star_t(c,x)\|^2 + \lambda\,\|(L_cg)(c)\|^2\right) \,p^\star_t(x\mid c)P_C(dc)\right\}.
\end{align}
The weight $p^\star_t(x\mid c)P_C(dc)$ follows from the training distribution and reflects how frequently each conditioning input contributes training examples at this state and time. For Tikhonov-type penalties, the solution is a linear smoother in the regression target, with some equivalent kernel $K_{\lambda,t,x}$ depending on $(t,x)$ \citep{silverman1984spline} :
\begin{align}
v_{\lambda,t}(c, x) = \hat g_{t,x}(c)=\int K_{\lambda,t,x}(c,c')\,v^\star_t(c',x)\,p^\star_t(x\mid c')P_C(dc'). \label{eq:slice_solution}
\end{align}
For the constant-preserving smoothers considered here, the equivalent kernel is normalized so that \(
\int K_{\lambda,t,x}(c,c')\,p^\star_t(x\mid c')P_C(dc')=1
\). The factor $p^\star_t(x\mid c')$ is the design density of the slice problem, not a separate modeling choice. The starting point of our analysis is a structural assumption on the kernel $K_{\lambda,t,x}$.

\begin{restatable}[Factorization of the kernel]{assumption}{kernelFactorization}
\label{assum:kernel_factorization}
The exact kernel $K_{\lambda,t,x}$ depends on $(t,x)$ through the design density $\rho_{t,x}(c):=p^\star_t(x\mid c)p_C(c)$. We make the following equivalent-kernel approximation, that the spatial shape of the kernel depends weakly on $(t,x)$:
\begin{align}
K_{\lambda,t,x}(c,c') = a_{t,x}(c)\,K'_\lambda(c,c'),
\label{eq:kernel_factorize}
\end{align}
for some normalization $a_{t,x}(c)$ and a kernel $K'_\lambda$ that depends on $\lambda$ but not on $(t,x)$. We also assume $K_{\lambda,t,x}(c,c')\geq 0$ and $a_{t,x}(c)>0$.
\end{restatable}

The factorization \eqref{eq:kernel_factorize} is an equivalent-kernel approximation of the kind used for smoothing splines and other roughness-penalized estimators \citep{silverman1984spline,wahba1990spline,green1994nonparametric}. The exact kernel $K_{\lambda,t,x}$ depends on $(t,x)$ only through the design density $\rho_{t,x}(c):=p^\star_t(x\mid c)p_C(c)$, so replacing it by a common kernel is justified when $\rho_{t,x}$ is locally stable over the smoothing length scale of the penalty. Appendix~\ref{app:smoothing} discusses this assumption in detail, along with its regime of validity (small $t$, large path variance, or locally slowly varying conditional mean). Appendix \ref{app:ek_validity} numerically tests this representation in a synthetic task by fitting a shared kernel satisfying our assumptions to neural velocity fields. The fitted representation closely approximates the learned fields over much of the trajectory, with discrepancies increasing near \(t=1\). %
 
Under Assumption~\ref{assum:kernel_factorization}, and mild moment and regularity conditions (Assumptions~\ref{assum:zero_mean}, \ref{assum:path_regularity}, and~\ref{assum:symmetric_smoothing}), the conditional distribution generated by $v_\lambda$ satisfies the following bias--variance relation.

\begin{theorem}[Excess Variance and Covariance]
\label{thm:main}
Suppose Assumption~\ref{assum:kernel_factorization} and Assumptions~\ref{assum:zero_mean}, \ref{assum:path_regularity}, and \ref{assum:symmetric_smoothing} hold. Let $\widehat X_1$ denote the endpoint sample generated by integrating
$v_\lambda$ at conditioning $C$, and let $X_1$ denote the true target.
\begin{align}
\mathbb{E}_C\!\left[
\mathrm{Cov}(\widehat X_1 \mid C)
\right]
\succeq
\mathbb{E}_C\!\left[
\mathrm{Cov}(X_1 \mid C)
\right]
+
\mathbb{E}_C\!\left[
\left( \mathbb{E}[\widehat X_1 \mid C]
-
\mathbb{E}[X_1 \mid C]\right)^{\otimes 2}
\right].
\label{eq:main_bound_vector_trace}
\end{align}
\end{theorem}

The proof proceeds by showing that, under Assumption~\ref{assum:kernel_factorization} and the stated regularity conditions, the conditional distribution generated by $v_\lambda$ can be written as
\(
\widehat p_1^\lambda(x\mid c)
=
\int p^\star(x\mid c')\,W_\lambda(dc'\mid c),
\)
for some $W_\lambda$ which is a Markov kernel on the conditioning space (Appendix~\ref{app:smoothing}, Eq.~(\ref{eq:endpoint_mixture})). Thus, the distribution induced by $v_\lambda$ is a local mixture of the true conditional distributions. This interpretation also qualitatively reflects the structure of the distribution learned by the CFM in Figure~\ref{fig:prob_methods_density}. Taking moments of this mixture and applying the smoothing-operator properties gives the bias--variance relation in Theorem~\ref{thm:main}. The full proof is deferred to Appendix~\ref{app:smoothing}

The comparison above is in the Loewner order. Equivalently, in every direction \(u\), the average learned conditional variance of \(u^\top\widehat X_1\) is at least the corresponding true conditional variance plus the average squared directional bias of the learned conditional mean.
In the scalar case, Theorem~\ref{thm:main} says that, on average over the conditioning variable, the learned conditional variance is at least the true conditional variance plus the squared bias of the learned conditional mean. 
Thus, smoothing bias incurred in fitting the conditional mean is reflected as additional spread in the generated samples. The relation in \Cref{eq:main_bound_vector_trace} reflects the trends observed empirically in Figure \ref{fig:bias_variance_corr}, that is, expected excess conditional variance $\gtrsim$ expected squared bias.

\section{Coverage Properties of Conditional Flow Models} \label{sec:coverage}

In this section, we test whether the coverage properties of CFMs observed in Section \ref{sec:excess_var_structure} extend to more real-world applications. We consider three application domains: Monte-Carlo policy evaluation, cell perturbation prediction, and behavior cloning. We focus on the \emph{underfitting regime}, the regime where the smoothing bias of neural networks holds.

\textbf{Coverage and interval definitions.}
For a CFM $\hat p_\theta(x\mid c)$ and nominal level $\alpha\in(0,1)$, we define predictive set $C_{\theta,\alpha}(c)$ and empirical coverage $\widehat{\mathrm{Cov}}_\alpha(\mathcal D)$ on $\mathcal D=\{(x_i,c_i)\}_{i=1}^{|\mathcal D|}$ same as in Section \ref{sec:excess_var_structure}.
We report coverage at $\alpha\in\{0.75,0.90,0.95\}$. For scalar outputs (policy evaluation), we construct $C_{\theta,\alpha}(c)$ as the central $\alpha$-mass interval of the learned conditional distribution; for high-dimensional outputs (cell-pertubation, behavior cloning), we use sample-based KNN predictive sets \citep{wang2022probabilisticconformalpredictionusing}.
To complement coverage at the three reported nominal levels, we also report \textbf{mean absolute calibration error (CE)}, an aggregate measure across a broader grid of levels. CE is the average absolute gap between nominal and observed coverage; lower values indicate better calibration (Appendix~\ref{app:coverage_metrics}).
Full constructions are given in Appendix~\ref{app:coverage_details}.

\textbf{i) Monte-Carlo Policy Evaluation.}
We use three OGBench environments \citep{park2025ogbenchbenchmarkingofflinegoalconditioned}. The CFM models the distribution of scalar Monte-Carlo returns $x$ conditioned on the state-action pair $c=(s,a)$. Figure~\ref{fig:coverage_antmaze} reports results for \texttt{Antmaze-large}; \texttt{Cube-Single} and \texttt{Humanoidmaze-medium} are deferred to Appendix~\ref{app:additional_coverage_exp}.

\textbf{ii) Single-cell perturbation prediction.}
We evaluate on Norman, PBMC cytokines, and SciPlex3 from scPerturb \citep{peidli2024scperturb} datasets, using the CellFlow architecture \citep{klein2025cellflow}. The CFM models the distribution of perturbed-cell states $x$, with control cells as the source distribution and a pretrained perturbation embedding as the condition $c$. Figure~\ref{fig:coverage_bio} reports Norman results; PBMC and SciPlex3 are deferred to Appendix~\ref{app:additional_coverage_exp}.

\textbf{iii) Behavior cloning.}
We evaluate a flow-matching policy on PushT \citep{chi2023diffusion}. The CFM models the distribution of short action chunks $x$ conditioned on the robot's observation history $c$.

Figures~\ref{fig:coverage_antmaze}--\ref{fig:coverage_bc} and \ref{fig:coverage_cube}--\ref{fig:coverage_sciplex} report our results. We plot CFM training and validation losses, empirical coverage metrics, predictive-set sizes, and utility metrics. We include the loss curves to identify when the trained model begins to overfit, and indicate the training regime over which we expect coverage to remain reliable. This region is shaded in the figures, and determined heuristically using a validation-loss based early-stopping rule, described in Appendix~\ref{app:underfitting_cutoffs}.
In practice, models are generally not trained far into the overfitting regime.
Here, however, we deliberately train for longer to demonstrate coverage trends across regimes. Set sizes are interval widths for scalar outputs and KNN radii for high-dimensional outputs. Validation log-likelihood (or task success rate for behavior cloning) is tracked to assert that high coverage is not obtained from a degenerate predictive distribution.

\paragraph{Baselines.}
We compare against six baselines. For scalar outputs, \textbf{det + SC} applies split conformal prediction to a deterministic regressor with the same conditioning backbone as the CFM, while \textbf{det + CQR} uses conformalized quantile regression to obtain locally adaptive intervals. \textbf{Conf CFM} conformalizes the raw CFM predictive sets, whereas \textbf{CFM mean + SC} applies split conformal prediction to the CFM conditional mean. We additionally compare against \textbf{Deep Evidential Regression (DER)} \citep{amini2020deep} for scalar outputs and \textbf{Heteroscedastic Gaussian Regression (HGR)}, using a diagonal Gaussian model for high-dimensional outputs. DER and HGR are reported for fully trained models. Complete experimental details are given in Appendix~\ref{app:coverage_details}. The takeaways are as follows: 

\begin{figure}[h!]
    \centering
    \includegraphics[width=1.0\textwidth]{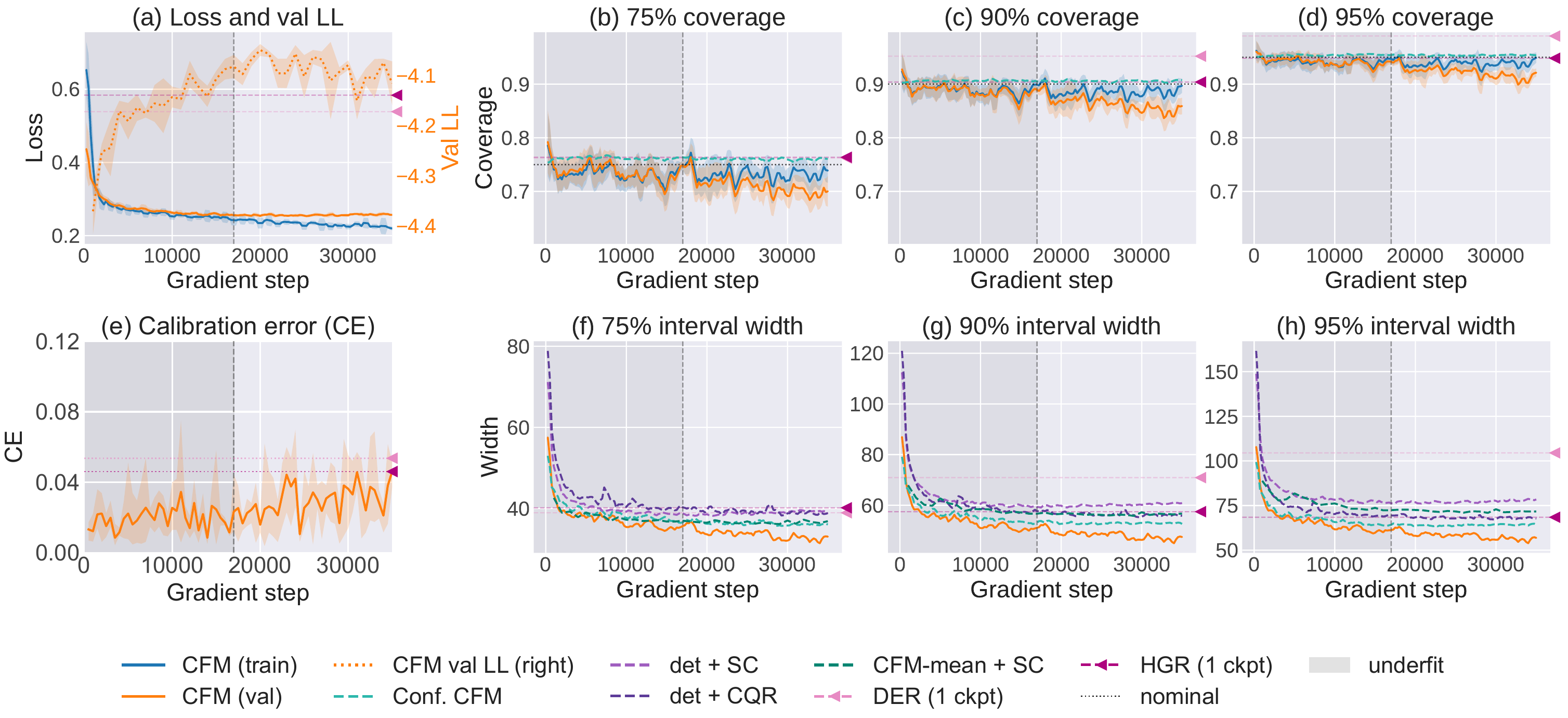}
    \caption{Policy evaluation: Antmaze-Large. (a) Train \& val loss, and val log-likelihood. (b–d) Coverage at nominal levels 75\%, 90\%, 95\%. (e) Mean absolute calibration error (CE; lower is better). (f–h) Mean predictive-interval widths
    } \label{fig:coverage_antmaze}
\end{figure}

\begin{figure}[h!]
        \centering
        \includegraphics[width=1.0\textwidth]{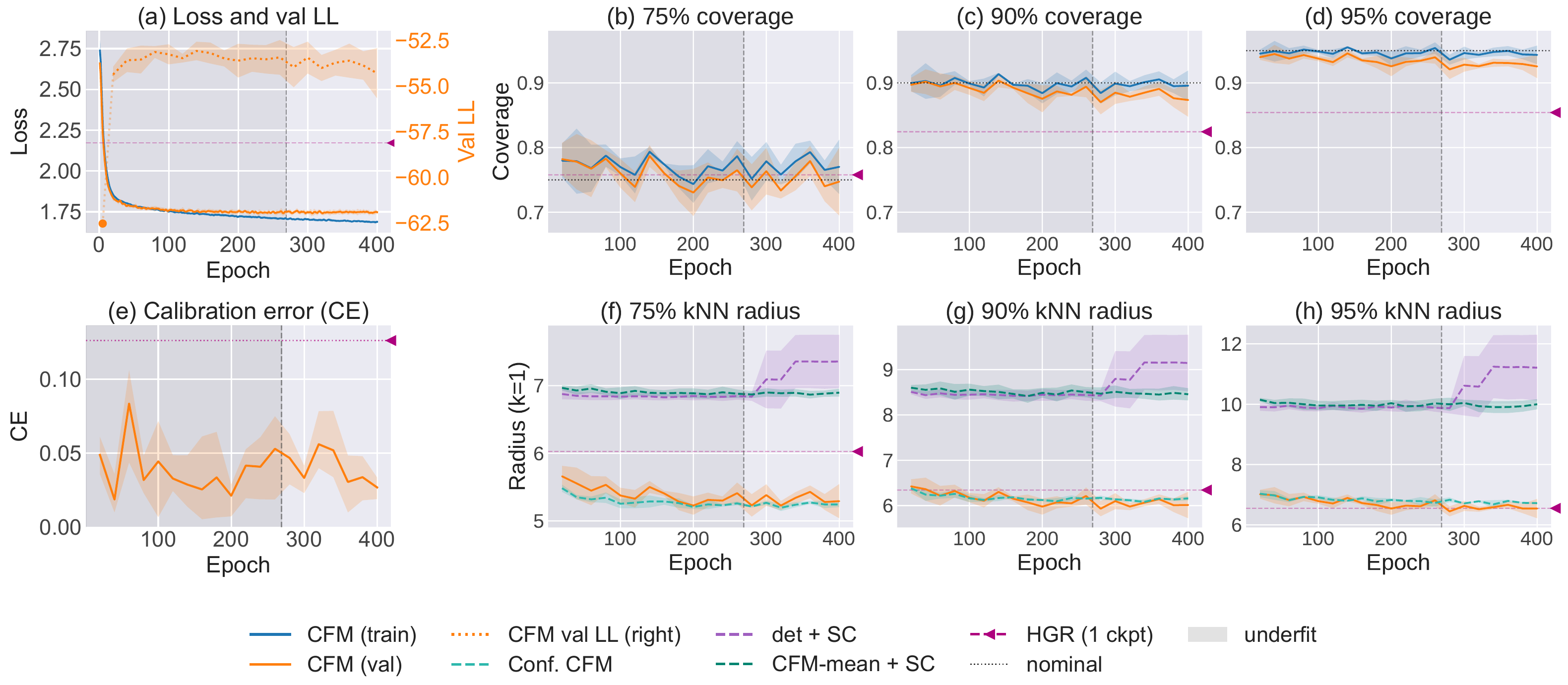}
    \caption{Cell-pertubation dataset: Norman. Layout follows Fig. \ref{fig:coverage_antmaze}; (f–h) report mean KNN radii.} \label{fig:coverage_bio}
\end{figure}

\begin{figure}[h!]
    \centering
    \includegraphics[width=1.0\textwidth]{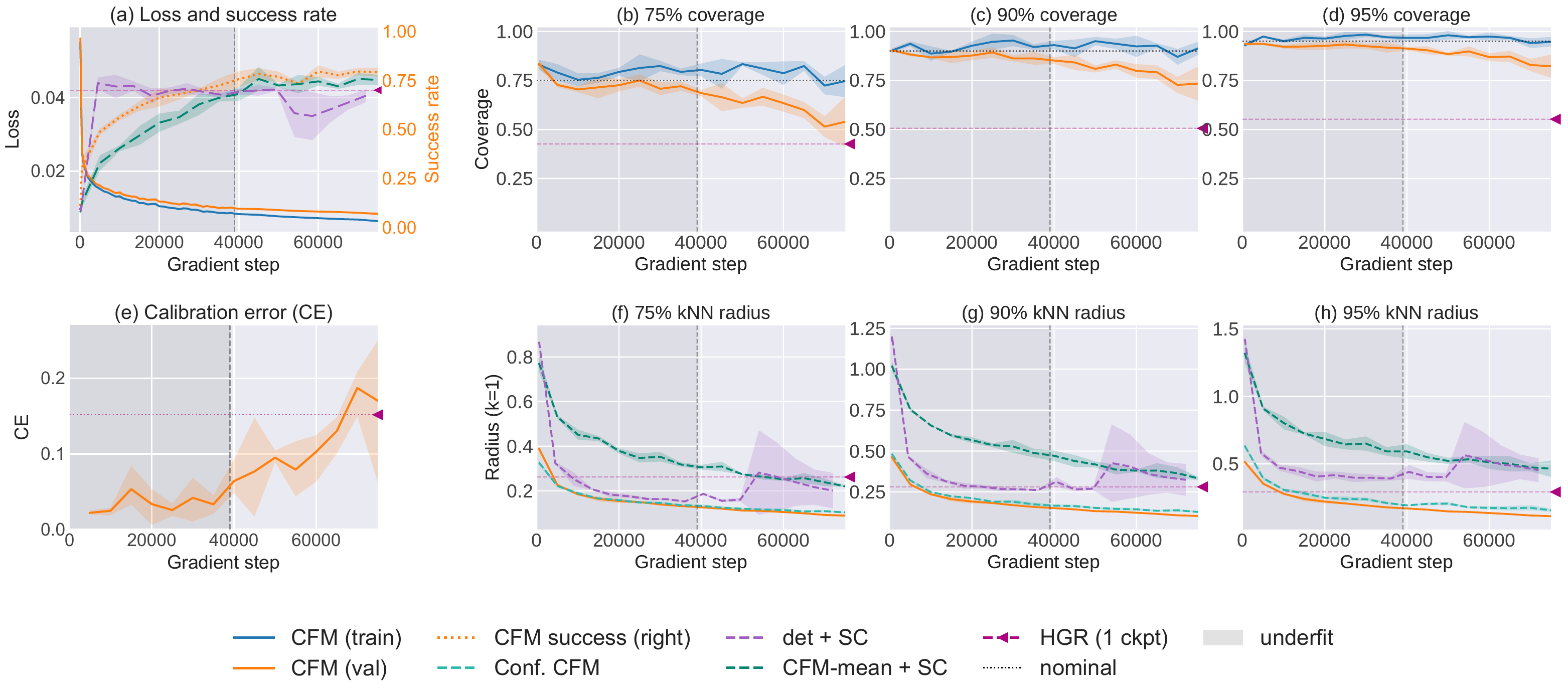}
    \caption{Behavioral Cloning. Layout follows Fig. \ref{fig:coverage_bio}; (a) reports task success rate on the right axis.} \label{fig:coverage_bc}
\end{figure}

\begin{enumerate}[leftmargin=6.5pt]
    \item \textbf{Coverage in the underfitting regime is approximately calibrated.} There are two takeaways here: i) Up till the point CFM train and val loss curves diverge (underfitting regime), coverage of CFM on val data stays approx. $\geq$ nominal level ($\widehat{\mathrm{Cov}}_\alpha(\mathcal D) \gtrsim \alpha$), and dips afterwards. ii) Interestingly, coverage of CFMs on the training data stays $\geq$ nominal throughout training. This is because the training data is always 'in-distribution' for the model. Once the model starts to overfit, the training data stays in-distribution w.r.t to the model's learned distribution, whereas the validation does not. Calibration Error of CFM also generally stays much lower than the baseline(s).
    \item \textbf{Variance tracks pointwise prediction error.} For scalar-output tasks we also plot pointwise predictive variance against squared prediction error in Figure~\ref{fig:policy_variance_error} in Appendix~\ref{app:additional_coverage_exp}. Examples with larger flow variance also tend to have larger squared prediction error for the flow mean, with positive Spearman correlation, showing that the variance is larger where the model's mean is less accurate. %
    \item \textbf{CFM based coverage sets are tight.} Throughout the regime of interest (before overfitting sets in): i) CFM-based coverage widths / KNN-radii are tighter than split-conformals. This is because split-conformals are not locally adaptive, and also shows that CFM-based coverage sets are. ii) CFM coverage set sizes are competitive with CQR and other baselines that produce close-to nominal coverage. iii) CFM and conformalized CFM sets have very similar sizes, showing that additional conformalization changes the sets only minimally.
\end{enumerate}

\FloatBarrier

%% file: body/conclusion.tex
\section{Conclusion and Limitations}

We identify a surprising inductive bias of conditional flow models: smoothing bias in vector-field learning causes unresolved dependence on the conditioning variable to appear as excess predictive variance, converting epistemic uncertainty into apparent aleatoric variance. We observe this phenomenon across synthetic and real-world tasks, and provide an idealized model to study it. We empirically find that this excess variance has approximate coverage properties, and, \textit{at any point} in training before overfitting, predictive sets constructed directly from CFM samples achieve approximately nominal coverage without explicit calibration training.

\paragraph{Limitations.}
Our theory analyzes an idealized population-level regularized smoothing model and does not establish the variance-inflation result for arbitrary neural-network training. Although the theoretical result characterizes the variance inflation phenomenon, it does not imply formal predictive coverage guarantees. The coverage properties studied here are empirical, and developing stronger theoretical coverage guarantees is left for future work.

%% file: body/app_smoothing.tex
\newpage

\section{Theoretical Model of Variance Inflation}
\label{app:smoothing}

This section provides the formal setup and proof of \Cref{thm:main} from \Cref{sec:theory}. The argument proceeds in three steps.
\begin{itemize}
    \item First, an equivalent-kernel approximation (\Cref{sec:eq_kernel}) reduces the regularized flow-matching solution to a density-weighted local average of the true velocity field.
    \item Second, this representation implies that the endpoint distribution generated by the regularized flow is a local mixture of true conditionals (\Cref{lem:local_mixture}).
    \item Third, taking moments of this mixture gives an explicit decomposition of the learned covariance into true conditional covariance and unresolved variation of the conditional mean. A short operator argument then shows that, on average, this unresolved variation dominates the squared smoothing bias, yielding \Cref{thm:main}.
\end{itemize}

\paragraph{Setup.}
We restate the regularized flow-matching problem from \Cref{sec:theory}: we study the population-level optimum
\begin{align}
v_\lambda \in \arg\min_{v \in \mathcal V}
\left\{
\mathbb{E}\!\left[\| v_T(C, X_T) - \dot X_T \|^2\right]
+ \lambda\, \mathbb{E}\!\left[\| (L_c v_T)(C, X_T) \|^2\right]
\right\},
\label{eq:reg_fm_appendix}
\end{align}
as a proxy for the inductive bias of a learned conditional flow. Here $\mathcal V$ is the admissible function space specified after \eqref{eq:reg_fm}, $\lambda \geq 0$ controls the strength of regularization, and $L_c$ is a linear operator acting only on the conditioning variable: for each fixed $(t, x)$, $L_c$ acts on the function $c \mapsto v_t(c, x)$ while treating $t$ and $x$ as fixed. Examples include derivative penalties such as $L_c = \nabla_c$, higher-order smoothness penalties, and RKHS-type penalties. This captures the regime of interest: the learned vector field may depend flexibly on the flow state $x$ and time $t$, but its dependence on the conditioning variable $c$ is smoothed.

\paragraph{Justification for studying this setup:}
Prior work has established that neural networks exhibit a spectral or smoothness bias during training, preferentially fitting smoother and lower-frequency components before finer-scale structure \citep{daniely2017sgdlearnsconjugatekernel, jacot2018neural, rahaman2019spectral, xu2020frequency, Lee_2020}. Equation~\eqref{eq:reg_fm} captures this behavior in an idealized population-level model. Closely related, \citet{pfrommer2025conditionaldiffusionmodelactually} study a special case with \(L_c=\nabla_c^2\) as a model of smoothness bias in conditional diffusion models. More broadly, underfitting of learned velocity and score fields has been studied extensively in the diffusion and flow-matching literature. Prior work has linked smoothness and selective underfitting of these fields to generalization, interpolation, and hallucination in generative models \citep{yoon2023diffusion,scarvelis2023closed,aithal2024understanding,song2025selective}. Finally, in Appendix~\ref{app:ek_validity} we numerically evaluate how well the factorized equivalent-kernel representation in Equations~\eqref{eq:slice_solution} and~\eqref{eq:kernel_factorize} approximates trained neural CFMs on synthetic one-dimensional regression tasks.

For brevity, throughout this appendix we write $\hat v := v_\lambda$ for the regularized population solution.

For any vector field $v$, the endpoint distribution generated at condition $c$ is the law of $\Phi_1^v(c,Z)$ with $Z\sim \rho_0$. We write its conditional mean and covariance as
\begin{align}
m_v(c) := \mathbb E_Z[\Phi_1^v(c,Z)], \qquad
\Sigma_v(c) := \mathrm{Cov}_Z(\Phi_1^v(c,Z)).
\end{align}
When the generated variable is scalar, we write $s_v^2(c):=\Sigma_v(c)$. In the unregularized realizable population limit, $\Sigma_v(c)$ matches the true conditional covariance of $X_1\mid C=c$. The central question of this section is what happens to this endpoint covariance when the vector field is regularized or underfit in the $c$-direction. We show that the learned endpoint distribution can have substantially larger spread than the true conditional distribution, with the excess spread reflecting unresolved variation of the true conditional mean over the model's effective smoothing neighborhood.

Let $\mu^\star(c):=\mathbb E[X_1\mid C=c]$ denote the true conditional mean, define the residual $\xi:=X_1-\mu^\star(C)$, and write $\Sigma^\star(c):=\mathrm{Cov}(X_1\mid C=c)$ for the true conditional covariance.

\begin{assumption}[Target moments]
\label{assum:zero_mean}
The target $X_1\in\mathbb R^{d_x}$ has finite second moments. The base variable $X_0\sim\rho_0$ is independent of $(C,\xi)$, and $p^\star_0(\cdot\mid c)=\rho_0$ does not depend on $c$.
\end{assumption}

These definitions imply
\begin{align}
X_1 = \mu^\star(C)+\xi, \qquad
\mathbb E[\xi\mid C]=0, \qquad
\mathrm{Cov}(\xi\mid C)=\Sigma^\star(C).
\end{align}
In the scalar case $d_x=1$, we write $\mu^\star=f^\star$ and $\Sigma^\star=\sigma^2$. We make no parametric distributional assumption about $\xi$.

\subsection{Assumptions: Equivalent-kernel model of regularized flow matching, and regularity conditions}
\label{sec:eq_kernel}

The regularized objective in \eqref{eq:reg_fm_appendix} penalizes variation of $c \mapsto v_t(c,x)$. For each fixed $(t,x)$, this is a supervised regression problem in $c$ whose population design distribution is the conditional law of $C\mid X_t=x$, i.e., proportional to $p^\star_t(x\mid c')P_C(dc')$. Classical equivalent-kernel calculations for regularized least squares show that the fitted value at $c$ is a local average of the regression target \citep{silverman1984spline}, which gives \Cref{eq:slice_solution}. We now simplify this equivalent-kernel representation into the form used in our analysis.

\paragraph{Simplifying the Kernel.}
We first restate \Cref{assum:kernel_factorization} for convenience.

\kernelFactorization*

We now show that \Cref{assum:kernel_factorization} implies \Cref{assum:density_kernel}, and then motivate \Cref{assum:kernel_factorization} in the light of prior literature.

\begin{lemma}[Density-weighted equivalent kernel]
\label{assum:density_kernel}
For each regularization level $\lambda$, there exists a Markov kernel $W_\lambda(dc' \mid c)$ on the conditioning space such that the regularized population velocity in~\eqref{eq:reg_fm} satisfies
\begin{align}
\hat v_t(c, x) \;=\;
\frac{\int v^\star_t(c', x)\, p^\star_t(x \mid c')\, W_\lambda(dc' \mid c)}
     {\int p^\star_t(x \mid c')\, W_\lambda(dc' \mid c)}.
\label{eq:eq_kernel}
\end{align}
\end{lemma}

\begin{proof}
Substituting \eqref{eq:kernel_factorize} into the slice solution \Cref{eq:slice_solution} and absorbing $a_{t,x}(c)$ into the normalization, we obtain
\begin{align}
\hat v_t(c,x) =
\frac{\int K'_\lambda(c,c')\,v^\star_t(c',x)\,p^\star_t(x\mid c')\,p_C(c')\,dc'}
{\int K'_\lambda(c,c')\,p^\star_t(x\mid c')\,p_C(c')\,dc'}.
\label{eq:density_weighted_v}
\end{align}
The form in~\eqref{eq:density_weighted_v} is exactly the representation in \Cref{assum:density_kernel}, with $W_\lambda$ obtained concretely by normalizing $K'_\lambda(c,c')p_C(c')$ in $c'$.
\end{proof}

\textbf{Motivating \Cref{assum:kernel_factorization}.}
The exact kernel $K_{\lambda,t,x}$ depends on $(t,x)$ only through the design density $\rho_{t,x}(c):=p^\star_t(x\mid c)p_C(c)$, so replacing it by a common kernel is justified when $\rho_{t,x}$ is locally stable over the smoothing length scale of the penalty. 
Equivalently, using the normal-equation resolvent $R_\rho:=(M_\rho+\lambda L_c^\ast L_c)^{-1}$, the second-resolvent identity
\begin{align}
R_{\rho_{t,x}}-R_{\bar\rho}=-R_{\rho_{t,x}}\big(M_{\rho_{t,x}}-M_{\bar\rho}\big)R_{\bar\rho}
\label{eq:resolvent_identity}
\end{align}
shows that the exact kernels are stable whenever the perturbation $M_{\rho_{t,x}}-M_{\bar\rho}$ is small in the operator norm induced by the regularized problem. This is the standard Fredholm/resolvent picture behind Tikhonov-regularized least squares and kernel ridge regression \citep{cucker2002mathematical,caponnetto2007optimal,kress2014linear,engl1996regularization}. For the Gaussian-like paths we consider, the approximation is expected to be best for small $t$, large path variance, or locally slowly varying conditional mean $\mu^\star$. It is weakest near $t=1$ in low-noise regions, where $p^\star_t(x\mid c)$ sharply localizes around level sets of the conditional mean.\\
Also, in \Cref{app:ek_validity} we numerically test how well \Cref{eq:eq_kernel} approximates the velocity field learned by a neural network, on a synthetic task. This is equivalent to verifying how well \Cref{eq:slice_solution} and \Cref{eq:kernel_factorize} together approximate a neural CFM's learned velocity field.

\begin{remark} The learned vector field representation \Cref{eq:eq_kernel} obtained from \Cref{assum:density_kernel} is stronger than saying that $\hat v_t(c,x)$ is an ordinary pointwise smoother of $v^\star_t(c,x)$: ordinary smoothing in $c$ at fixed $(t,x)$ would replace the conditional mean by a smoothed mean and would not create the same excess endpoint variance. The density-weighted form in \eqref{eq:density_weighted_v} is structurally different: it locally blends the conditional probability paths themselves, weighting by how likely each $c'$ is to have produced the observed $x$ under the true model. This blending is what produces the variance inflation.
\end{remark}

We conclude this subsection with two assumptions on the resulting kernel $W_\lambda$. First, we require mild regularity conditions ensuring that the mixed path above is the path generated by $\hat v$.

\begin{assumption}[Path regularity and uniqueness]
\label{assum:path_regularity}
For $W_\lambda$ in \Cref{assum:density_kernel}, the mixture density
\begin{align}
\hat p_t^\lambda(x\mid c):=\int p^\star_t(x\mid c')\,W_\lambda(dc'\mid c)
\end{align}
is positive and continuously differentiable in $(t,x)$, differentiation under the $W_\lambda$-integral is valid, and the multivariate continuity equation
\begin{align}
\partial_t p_t(x\mid c)+\nabla_x\cdot\!\left(p_t(x\mid c)\hat v_t(c,x)\right)=0
\end{align}
with initial law $\rho_0$ admits a unique weak solution.
\end{assumption}

Second, we assume reversibility and positive semidefiniteness of the kernel --- two standard properties of self-adjoint smoothing operators on $L^2(P_C)$. In particular, for example, they are satisfied by the Markov kernels arising from Tikhonov regularization with symmetric quadratic penalties \citep{engl1996regularization}.

\begin{assumption}[Symmetric positive smoothing]
\label{assum:symmetric_smoothing}
The kernel $W_\lambda$ preserves $P_C$ and is reversible:
\begin{align}
P_C(dc)\,W_\lambda(dc'\mid c)=P_C(dc')\,W_\lambda(dc\mid c').
\end{align}
Moreover, the smoothing operator
\[
(S_\lambda g)(c):=\int g(c')\,W_\lambda(dc'\mid c)
\]
is positive semidefinite as an operator on $L^2(P_C)$.
\end{assumption}

\subsection{Formal Main Results}

We are now ready to state our formal main results, deferring proofs to \Cref{ssec:proofs}. First, we show that under our assumptions, \textbf{smoothing over the flows directly translates to smoothing over the conditionals}. This is the key bridge between regularization of the vector field and variance inflation in the generated distribution.

\begin{proposition}[Regularization induces a local mixture]
\label{lem:local_mixture}
Suppose \Cref{assum:zero_mean,assum:kernel_factorization,assum:path_regularity} hold. Then the conditional path $\hat p_t^\lambda(\cdot\mid c)$ defined in \Cref{assum:path_regularity} satisfies the continuity equation
\begin{align}
\partial_t \hat p_t^\lambda(x\mid c)
+
\nabla_x\cdot\!\left(
\hat p_t^\lambda(x\mid c)\hat v_t(c,x)
\right)
=0,
\end{align}
with initial law $\hat p_0^\lambda(\cdot\mid c)=\rho_0$. Consequently, the endpoint law generated by $\hat v$ is the $W_\lambda$-mixture of true conditionals:
\begin{align}
\hat p_1^\lambda(x\mid c)
=
\int p^\star(x\mid c')\,W_\lambda(dc'\mid c).
\label{eq:endpoint_mixture}
\end{align}
\end{proposition}

At this point, the flow-dynamics part of the argument is complete: the regularized flow generates a local mixture of the true conditional distributions. The remainder of the variance analysis follows from the moments of this mixture.

Let
\begin{align}
C'_\lambda\mid C=c\sim W_\lambda(\cdot\mid c),
\qquad
X'_1\sim p^\star(\cdot\mid C'_\lambda).
\end{align}
By \Cref{lem:local_mixture},
\begin{align}
\hat X_1\mid C=c \overset{d}{=} X'_1.
\end{align}
Therefore, by the laws of total expectation and total covariance,
\begin{align}
\hat\mu_\lambda(c)
&:=
\mathbb E[\hat X_1\mid C=c]
=
\mathbb E[\mu^\star(C'_\lambda)\mid C=c]
=
(S_\lambda\mu^\star)(c),
\label{eq:vector_mean_mixture}
\\
\hat\Sigma_\lambda(c)
&:=
\mathrm{Cov}(\hat X_1\mid C=c)
\nonumber\\
&=
\mathbb E[\Sigma^\star(C'_\lambda)\mid C=c]
+
\mathrm{Cov}(\mu^\star(C'_\lambda)\mid C=c).
\label{eq:vector_cov_mixture}
\end{align}

The decomposition in \eqref{eq:vector_cov_mixture} already exposes the variance-inflation mechanism. The first term is the true conditional covariance averaged over the smoothing neighborhood, while the second is the variation of nearby conditional means that the smoothed model has failed to resolve.

\paragraph{Variance in the Scalar Case.}
For simplicity, first consider scalar outputs. Write $\mu^\star=f^\star$ and $\Sigma^\star=\sigma^2$, and define
\begin{align}
\hat f_\lambda(c)
&:=
\mathbb E[f^\star(C'_\lambda)\mid C=c]
=
(S_\lambda f^\star)(c),
\\
\delta_\lambda(c)
&:=
\hat f_\lambda(c)-f^\star(c),
\end{align}
together with
\begin{align}
s^2_\lambda(c)
&:=
\mathrm{Var}(f^\star(C'_\lambda)\mid C=c),
\\
\bar\sigma^2_\lambda(c)
&:=
\mathbb E[\sigma^2(C'_\lambda)\mid C=c].
\end{align}
Specializing \eqref{eq:vector_mean_mixture}--\eqref{eq:vector_cov_mixture} to the scalar case gives
\begin{align}
\mathbb E[\hat X_1\mid C=c]
&=
f^\star(c)+\delta_\lambda(c),
\\
\mathrm{Var}(\hat X_1\mid C=c)
&=
s^2_\lambda(c)+\bar\sigma^2_\lambda(c).
\label{eq:scalar_moments}
\end{align}

Equation~\eqref{eq:scalar_moments} makes the source of the additional variance explicit. The term $s_\lambda^2(c)$ is the variance of $f^\star$ over the smoothing neighborhood $W_\lambda(\cdot\mid c)$, i.e., the unresolved variation of the regression function over the model's effective receptive field around $c$. The term $\bar\sigma^2_\lambda(c)$ is the local average of the irreducible noise variance $\sigma^2$ over the same neighborhood. In the unregularized limit $(\lambda\to 0, W_\lambda\to \delta_c)$, the first term vanishes and the second collapses to $\sigma^2(c)$, recovering the correct conditional variance.

Thus, the learned conditional variance exceeds the local average noise level by exactly $s_\lambda^2(c)$. The remaining question is whether this unresolved variation is large enough to account for the squared bias $\delta_\lambda^2(c)$ induced by the same smoothing operator.

\begin{proposition}[Average excess variance dominates squared bias]
\label{cor:variance_dominates_bias}
Suppose $d_x=1$. Under \Cref{assum:zero_mean,assum:symmetric_smoothing,assum:kernel_factorization,assum:path_regularity},
\begin{align}
\mathbb E_C\!\left[
s^2_\lambda(C)-\delta^2_\lambda(C)
\right]
=
2\,\langle
S_\lambda f^\star,
(I-S_\lambda)f^\star
\rangle_{L^2(P_C)}
\ge 0.
\label{eq:variance_bias_operator}
\end{align}
\end{proposition}

\paragraph{Average scalar excess variance.}
\Cref{cor:variance_dominates_bias} can be equivalently phrased directly in terms of the conditional variance learned by the flow. Let
\begin{align}
\hat V_\lambda(c)
&:=
\mathrm{Var}(\hat X_1\mid C=c)
=
s^2_\lambda(c)+\bar\sigma^2_\lambda(c),
\\
V^\star(c)
&:=
\mathrm{Var}(X_1\mid C=c)
=
\sigma^2(c),
\end{align}
and
\begin{align}
\mathrm{Bias}_\lambda(c)
:=
\mathbb E[\hat X_1\mid C=c]
-
\mathbb E[X_1\mid C=c]
=
\delta_\lambda(c).
\end{align}
Combining \eqref{eq:scalar_moments} with \Cref{cor:variance_dominates_bias} and using
\(
\mathbb E[\bar\sigma^2_\lambda(C)]
=
\mathbb E[\sigma^2(C)]
\)
by stationarity,
\begin{align}
&\mathbb E_C[\hat V_\lambda(C)]
-
\mathbb E_C[V^\star(C)]
-
\mathbb E_C[\mathrm{Bias}_\lambda(C)^2]
\nonumber\\
&\qquad =
\mathbb E_C[
s^2_\lambda(C)-\delta_\lambda^2(C)
]
\ge 0.
\end{align}
Equivalently,
\begin{align}
\mathbb E_C[\hat V_\lambda(C)]
\ge
\mathbb E_C[V^\star(C)]
+
\mathbb E_C[\mathrm{Bias}_\lambda(C)^2].
\label{eq:average_scalar_variance_budget}
\end{align}
Thus, on average over the conditioning variable, the learned conditional variance is at least the true conditional variance plus the squared smoothing bias of the learned conditional mean.

This gives a precise sense in which the conditional flow converts epistemic uncertainty into apparent aleatoric variance. The smoothing operator $S_\lambda$ captures unresolved variation of the regression function $f^\star$ in the conditioning variable, which the model cannot represent given its inductive bias. This unresolved variation manifests as bias $\delta_\lambda(c)$ in the learned conditional mean; \eqref{eq:average_scalar_variance_budget} shows that the same smoothing also induces enough additional conditional spread to cover the squared bias on average. Underfit conditional flows therefore do not merely smooth their predictions toward a less informative estimate; they also produce stochastic outputs whose extra variance carries the unresolved dependence on $C$.

\paragraph{Vector Case.}
The vector case follows by applying the same scalar argument along every direction in the output space, and admits the same interpretation.

\begin{proposition}[Vector-valued excess covariance]
\label{cor:vector_variance_budget}
Suppose $X_1\in\mathbb R^{d_x}$. Let
\begin{align}
b_\lambda(c):=\hat\mu_\lambda(c)-\mu^\star(c).
\end{align}
Under \Cref{assum:zero_mean,assum:path_regularity,assum:kernel_factorization,assum:symmetric_smoothing},
\begin{align}
\mathbb E_C[\hat\Sigma_\lambda(C)]
\succeq
\mathbb E_C[\Sigma^\star(C)]
+
\mathbb E_C[b_\lambda(C)b_\lambda(C)^\top],
\label{eq:vector_covariance_budget}
\end{align}
where $\succeq$ denotes the Loewner order.
\end{proposition}

\subsection{Deferred Proofs}
\label{ssec:proofs}

We first prove the local-mixture representation that connects smoothing of the velocity field to smoothing of the generated conditional distribution.

\begin{proof}[Proof of \Cref{lem:local_mixture}]
For every fixed $c'$, the exact conditional path satisfies the multivariate continuity equation
\begin{align}
\partial_t p^\star_t(x\mid c')
+
\nabla_x\cdot\!\left(
p^\star_t(x\mid c')v^\star_t(c',x)
\right)
=0.
\end{align}
Integrating this identity against $W_\lambda(dc'\mid c)$ and exchanging the integral with the derivatives, by \Cref{assum:path_regularity}, gives
\begin{align}
\partial_t \hat p_t^\lambda(x\mid c)
+
\nabla_x\cdot\!\left(
\int p^\star_t(x\mid c')v^\star_t(c',x)\,
W_\lambda(dc'\mid c)
\right)
=0.
\end{align}
By \Cref{assum:density_kernel},
\begin{align}
\int p^\star_t(x\mid c')v^\star_t(c',x)\,
W_\lambda(dc'\mid c)
=
\hat p_t^\lambda(x\mid c)\hat v_t(c,x).
\end{align}
The initial law is $\rho_0$ since $p^\star_0(\cdot\mid c')=\rho_0$ does not depend on $c'$. Uniqueness in \Cref{assum:path_regularity} then identifies $\hat p_t^\lambda$ as the law generated by $\hat v$. Evaluating at $t=1$ gives \eqref{eq:endpoint_mixture}.
\end{proof}

\begin{proof}[Proof of \Cref{cor:variance_dominates_bias} (scalar case)]
The proof compares two effects of the same smoothing operator: the unresolved variation $s_\lambda^2$ created by averaging nearby values of $f^\star$, and the squared bias $\delta_\lambda^2$ created by moving $f^\star$ toward $S_\lambda f^\star$. Their difference reduces to a quadratic form involving $S_\lambda(I-S_\lambda)$.

By definition,
\[
s^2_\lambda(c)
=
(S_\lambda f^{\star 2})(c)
-
(S_\lambda f^\star)(c)^2.
\]
Taking expectation over $C\sim P_C$ and using stationarity of $W_\lambda$,
\begin{align}
\mathbb E[s^2_\lambda(C)]
=
\mathbb E[f^\star(C)^2]
-
\mathbb E[(S_\lambda f^\star)(C)^2].
\label{eq:Es2}
\end{align}
Similarly, since
\[
\delta_\lambda(c)
=
(S_\lambda f^\star)(c)-f^\star(c),
\]
we have
\begin{align}
\mathbb E[\delta^2_\lambda(C)]
=
\mathbb E[(S_\lambda f^\star)(C)^2]
-
2\,\mathbb E[(S_\lambda f^\star)(C)f^\star(C)]
+
\mathbb E[f^\star(C)^2].
\label{eq:Ebias2}
\end{align}
Subtracting \eqref{eq:Ebias2} from \eqref{eq:Es2} and writing inner products in $L^2(P_C)$ gives
\begin{align}
\mathbb E[s^2_\lambda(C)-\delta^2_\lambda(C)]
&=
2\langle S_\lambda f^\star,f^\star\rangle
-
2\langle S_\lambda f^\star,S_\lambda f^\star\rangle
\nonumber\\
&=
2\langle
S_\lambda f^\star,
(I-S_\lambda)f^\star
\rangle.
\end{align}

It remains to show that this quadratic form is nonnegative. By \Cref{assum:symmetric_smoothing}, $S_\lambda$ is self-adjoint on $L^2(P_C)$ by reversibility and positive semidefinite by assumption. As a Markov operator preserving $P_C$, it is also a contraction. Hence its spectrum lies in $[0,1]$, and so
\(
S_\lambda(I-S_\lambda)
\)
is positive semidefinite. Therefore,
\begin{align}
\langle
S_\lambda f^\star,
(I-S_\lambda)f^\star
\rangle
=
\langle
f^\star,
S_\lambda(I-S_\lambda)f^\star
\rangle
\ge 0.
\end{align}
\end{proof}

\begin{proof}[Proof of \Cref{cor:vector_variance_budget} (vector case)]
The vector result follows by applying the scalar argument to every one-dimensional projection of the conditional mean.

By \eqref{eq:vector_cov_mixture},
\begin{align}
\hat\Sigma_\lambda(c)
=
\mathbb E[\Sigma^\star(C'_\lambda)\mid C=c]
+
A_\lambda(c),
\qquad
A_\lambda(c)
:=
\mathrm{Cov}(\mu^\star(C'_\lambda)\mid C=c).
\end{align}
Since $W_\lambda$ preserves $P_C$,
\begin{align}
\mathbb E_C[
\mathbb E[\Sigma^\star(C'_\lambda)\mid C]
]
=
\mathbb E_C[\Sigma^\star(C)].
\end{align}
It therefore remains to show
\begin{align}
\mathbb E_C[A_\lambda(C)]
\succeq
\mathbb E_C[b_\lambda(C)b_\lambda(C)^\top].
\end{align}

Fix $u\in\mathbb R^{d_x}$ and define
\(
g_u(c):=u^\top\mu^\star(c).
\)
Then
\begin{align}
u^\top A_\lambda(c)u
&=
\mathrm{Var}(g_u(C'_\lambda)\mid C=c),
\\
u^\top b_\lambda(c)
&=
(S_\lambda g_u)(c)-g_u(c).
\end{align}
Applying the scalar calculation from \Cref{cor:variance_dominates_bias} to $g_u$ gives
\begin{align}
u^\top\!\left(
\mathbb E_C[A_\lambda(C)]
-
\mathbb E_C[b_\lambda(C)b_\lambda(C)^\top]
\right)u
=
2\langle
S_\lambda g_u,
(I-S_\lambda)g_u
\rangle_{L^2(P_C)}
\ge 0,
\end{align}
where the final inequality follows from \Cref{assum:symmetric_smoothing}. Since this holds for every $u$, the Loewner inequality follows:
\begin{align}
\mathbb E_C[A_\lambda(C)]
\succeq
\mathbb E_C[b_\lambda(C)b_\lambda(C)^\top].
\end{align}
Combining this with the covariance decomposition in \eqref{eq:vector_cov_mixture} gives \eqref{eq:vector_covariance_budget}.
\end{proof}

We emphasize that both \eqref{eq:average_scalar_variance_budget} and \eqref{eq:vector_covariance_budget} are average statements over $C$. They do not assert that the scalar excess variance dominates squared bias pointwise, nor that the learned covariance pointwise dominates the true covariance in every heteroscedastic problem. The empirical observations in \Cref{sec:coverage} go somewhat beyond this, showing that pointwise predicted variance correlates strongly with pointwise squared error of the flow's mean.

\subsection{Validity of the Equivalent-Kernel Assumption (\texorpdfstring{\Cref{assum:kernel_factorization}}{Assumption \ref*{assum:kernel_factorization}})} \label{app:ek_validity}

In this section, we numerically test how well \Cref{eq:eq_kernel} approximates the velocity field learned by a neural network, on a synthetic task. This is equivalent to verifying how well \Cref{eq:slice_solution} and \Cref{eq:kernel_factorize} together approximate a neural CFM's learned velocity field. Let $\hat{v}_\theta(c, x_t, t)$ denote a neural CFM, and $v^*_t, p^*_t$ denote the true velocity field and distribution respectively at time $t$. Specifically, we want to test if there exists a kernel $W$ such that $\tilde v_{W,t}(c, x) \;=\;
\frac{\int v^\star_t(c', x)\, p^\star_t(x \mid c')\, W(dc' \mid c)}{\int p^\star_t(x \mid c')\, W(dc' \mid c)}$ closely approximates $\hat{v}_\theta$.

To assess the validity of the equivalent-kernel assumption, we use a 1D Gaussian target:
\[
X_1 \mid C = c \sim \mathcal{N}(f^\star(c), \sigma^2), \qquad
X_0 \sim \mathcal{N}(0, \sigma_0^2),
\]
with linear path $X_t = (1-t)X_0 + tX_1$, giving
\[
X_t \mid C = c \sim \mathcal{N}\big(t f^\star(c), s_t^2\big), \qquad
s_t^2 = (1-t)^2 \sigma_0^2 + t^2 \sigma^2 .
\]

The factorized kernel approximation is most accurate when the slice-design density
$\rho_{t,x}(c') = p_t^\star(x \mid c')\, p_C(c')$ varies little over the smoothing bandwidth
$h_\lambda$. With $p_C$ locally constant, the path mean shifts by
$\approx t h_\lambda |f^{\star\prime}(c)|$ over this neighborhood while the path std is $s_t$,
giving the heuristic condition
\[
\frac{t h_\lambda |f^{\star\prime}(c)|}
{\sqrt{(1-t)^2 \sigma_0^2 + t^2 \sigma^2}} \ll 1 .
\]

For our experiments, we use
$f^\star(c) = \sin(2\pi c) + \sin(10\pi c)$, $\sigma = 0.05$, and
train a CFM $\hat{v}_\theta(c, x_t, t)$, analyzing checkpoints
$N_{\mathrm{steps}} \in \{1000, 3000, 5000\}$ (heavy/moderate/mild underfitting).

To test the equivalent-kernel (EK) approximation, we discretize $c$ on a uniform grid
$\{c_j\}_{j=1}^N$ and define
\[
\tilde{v}_{\phi,t}(c_i, x)
= \frac{\sum_j W_{\phi,ij}\, v_t^\star(c_j, x)\, p_t^\star(x \mid c_j)}
       {\sum_j W_{\phi,ij}\, p_t^\star(x \mid c_j)} ,
\]
where $v^\star_t$ and $p_t^\star$ are available in closed form.
We parameterize $W_\phi$ as a learned-embedding Gibbs kernel,
\[
K_{\phi,ij}
=
\exp\left(
-\frac{\|h_\phi(c_i)-h_\phi(c_j)\|^2}{2\sigma_K^2}
\right),
\qquad
W_\phi = D_\phi K_\phi D_\phi,
\]
where $h_\phi:[0,1]\to\mathbb{R}^2$ is a small MLP and $\sigma_K$ is learned.
The positive diagonal matrix $D_\phi$ is obtained by symmetric Sinkhorn normalization so that
$W_\phi\mathbf{1}=\mathbf{1}$. This construction ensures
\[
W_\phi=W_\phi^\top,\qquad
W_\phi\succeq0,\qquad
(W_\phi)_{ij}\geq0,\qquad
W_\phi\mathbf{1}=\mathbf{1}.
\]

Denoting $\hat{v}_{\theta,t}=\hat{v}_\theta(\cdot,\cdot,t)$, we fit $\phi$ by minimizing
\[
\mathcal{L}(\phi)
=
\mathbb{E}_{t,C,X_t\sim \hat p_t(\cdot\mid C)}
\left[
\left(
\tilde{v}_{\phi,t}(C,X_t)
-
\hat{v}_{\theta,t}(C,X_t)
\right)^2
\right],
\]
where $\hat p_t(\cdot\mid c)$ is the marginal induced by integrating the learned field
$\hat v_\theta$ from the Gaussian base distribution up to time $t$.
Thus, we test whether the EK approximation explains the trained field on the region of
state space actually visited by the learned model.

Let $\hat{p}_1(\cdot \mid c), \tilde{p}_1(\cdot \mid c)$ be the endpoint distributions obtained by
integrating $\hat{v}_\theta, \tilde{v}_\phi$ from a shared $\mathcal{N}(0,1)$ base, with
means/variances $\hat{f}, \hat{\sigma}^2$ and $\tilde{f}, \tilde{\sigma}^2$ respectively.
Writing
\[
\langle \cdot \rangle
=
\mathbb{E}_{C, X_t \sim \hat p_t(\cdot \mid C)}[\cdot],
\]
we report the field error
\[
E_{\mathrm{field}}(t)
=
\frac{
\left\langle
(\tilde{v}_{\phi,t}-\hat{v}_{\theta,t})^2
\right\rangle
}{
\left\langle
\hat{v}_{\theta,t}^2
\right\rangle
}
\]
as a function of $t$ in \Cref{fig:ek_error}, Column 1.
We also plot $\tilde{v}_{\phi,t}(c,x_t)$ and $\hat{v}_{\theta,t}(c,x_t)$ as a function of $c$
at $x_t=0$ for $t=0.25,0.5,0.7,0.9$ in \Cref{fig:ek_error}, Columns 2--5.
As expected, \Cref{fig:ek_error} shows that $\tilde{v}_{\phi}$ is a good approximation of
$\hat{v}_{\theta}$ over much of the learned trajectory, with the approximation becoming weaker
in some regions as $t\rightarrow1$.
In \Cref{fig:ek_samples}, we also plot the endpoint distributions obtained by integrating
$\hat{v}_{\theta}$ and $\tilde{v}_{\phi}$. Interestingly, the endpoint distributions look very
similar, and the vector-field approximation error close to $t=1$ does not affect the integral
significantly. Finally, in \Cref{fig:ek_smoothing} we visualize how both $\hat{v}_\theta$ and
$\tilde{v}_\phi$ look like density-weighted smoothed functions of the true velocity field
$v^\star$.

It is intuitive that a deterioration of the approximation near t=1 need not substantially affect the endpoint distribution. Standard ODE stability bounds control endpoint discrepancies through the time-integrated velocity-field error, weighted by the sensitivity of the remaining dynamics \citep{benton2023error}. Consequently, discrepancies concentrated in a short terminal interval can have a limited accumulated effect when these quantities remain controlled. In this sense, the approximation is less restrictive for characterizing endpoint distribution properties (which we use it for) than for accurately describing the full velocity field.

\begin{figure}[h!]
    \centering
    \includegraphics[width=1.0\textwidth]{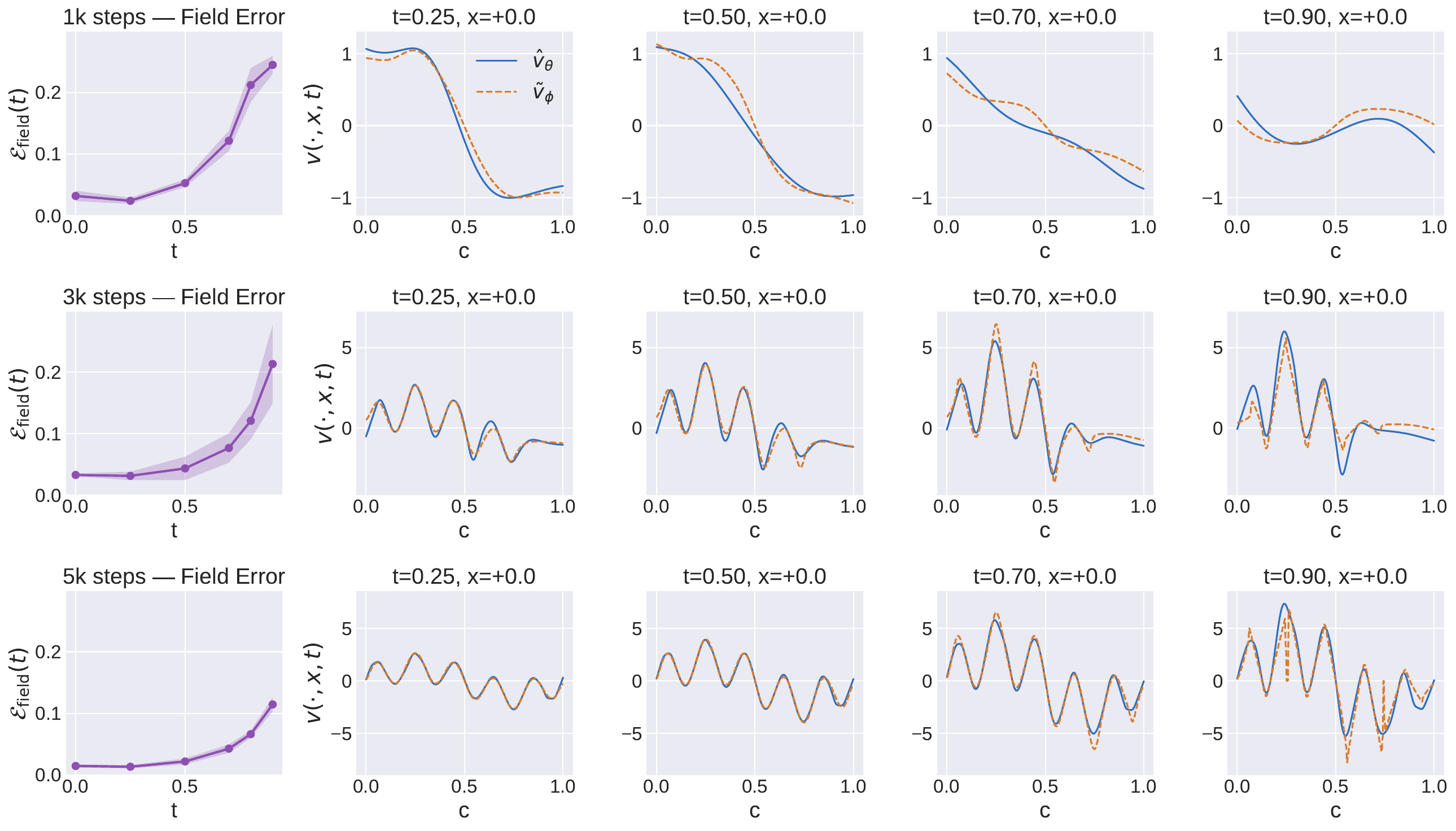}
    \caption{Kernel approximation error and velocity field visualization} 
    \label{fig:ek_error}
\end{figure}

\begin{figure}[h!]
    \centering
    \includegraphics[width=1.0\textwidth]{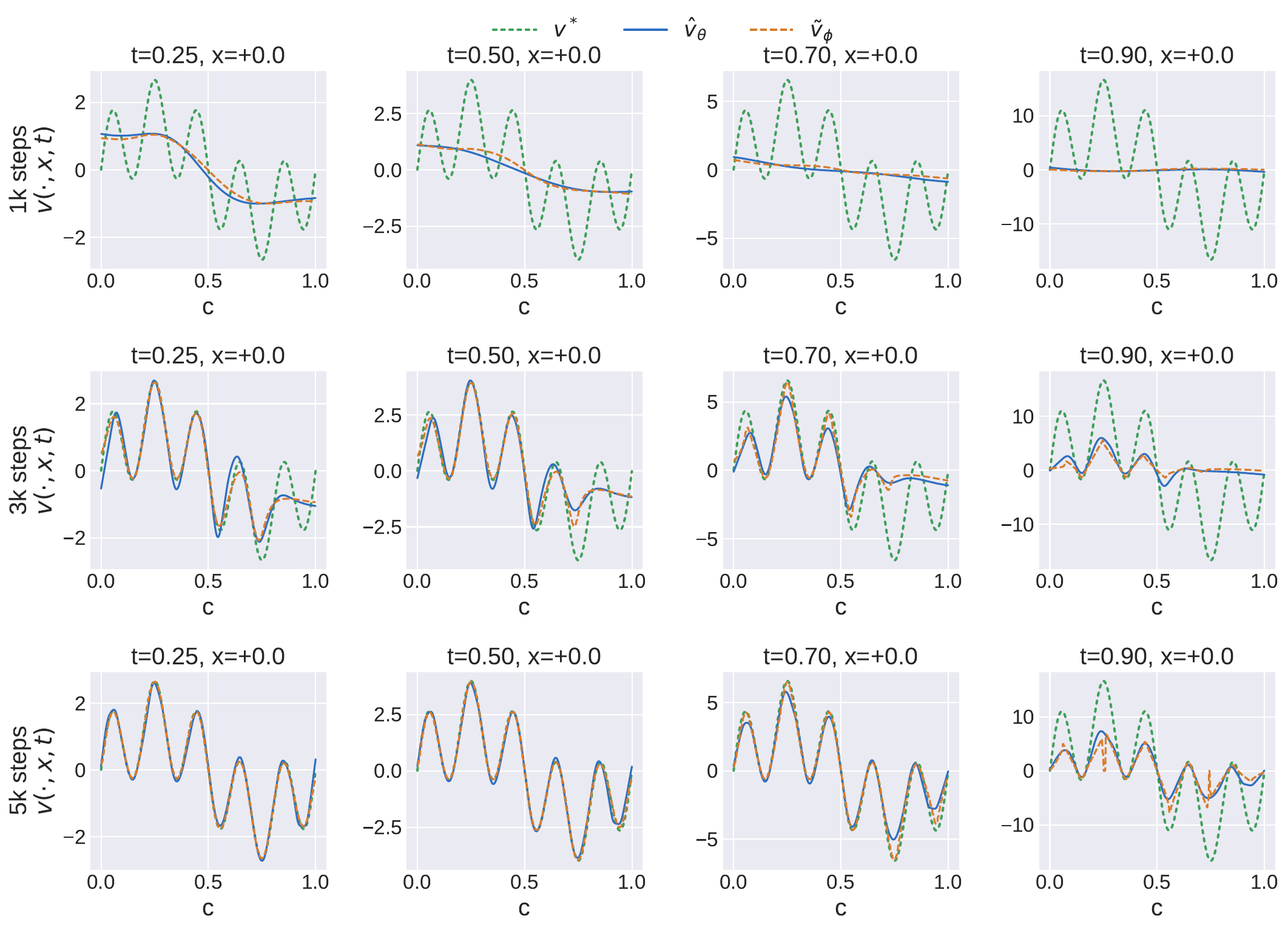}
    \caption{Visualization of density-weighted smoothing in the learned velocity field and the kernel approximation} 
    \label{fig:ek_smoothing}
\end{figure}

\begin{figure}[h!]
    \centering
    \includegraphics[width=1.0\textwidth]{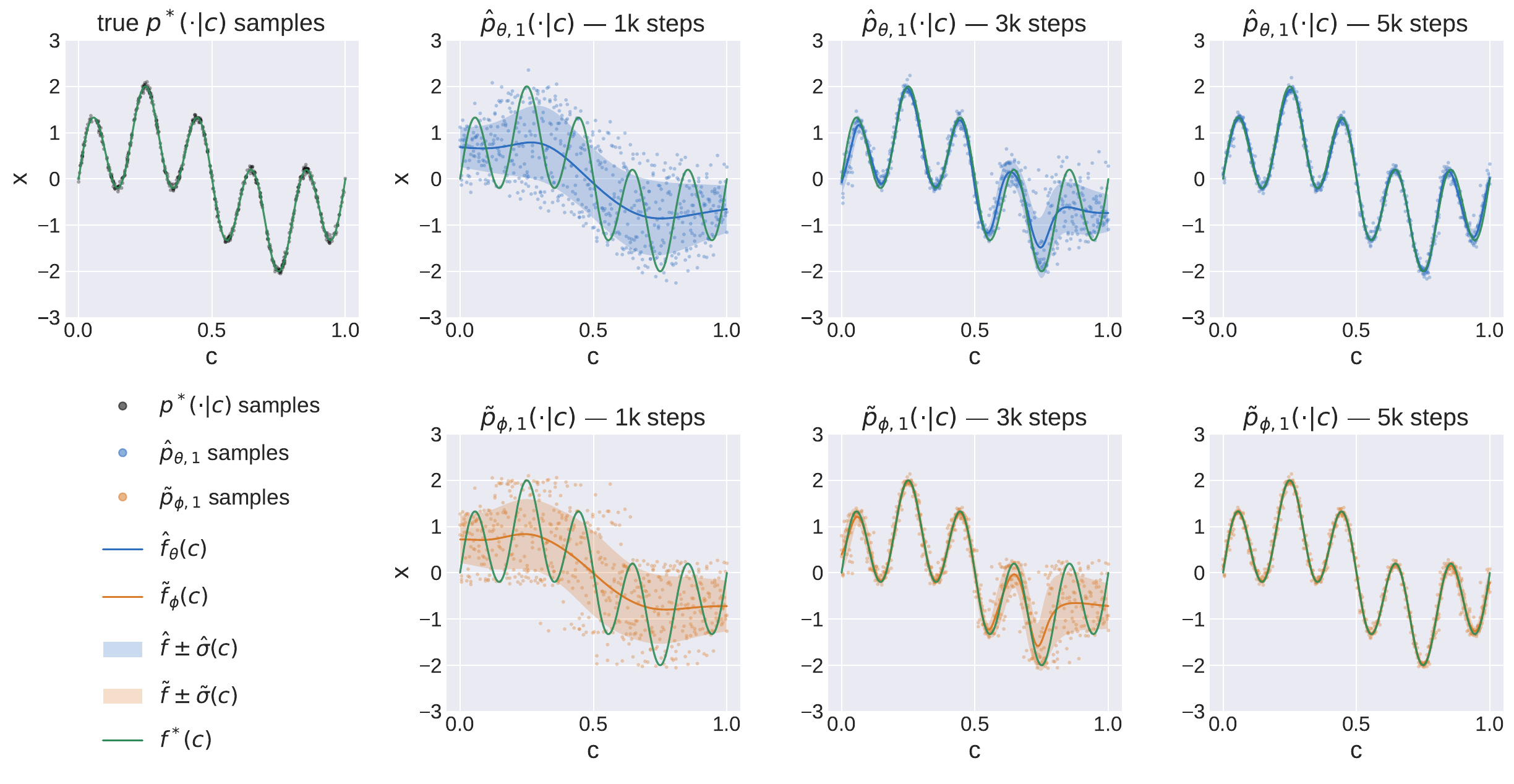}
    \caption{Effect of kernel approximation on the endpoint distribution} 
    \label{fig:ek_samples}
\end{figure}

%% file: body/app_coverage.tex
\section{Coverage Experiments} \label{app:coverage_details}

This section provides full constructions of the predictive sets and conformal baselines used in Section~\ref{sec:coverage}, along with experimental details.

\subsection{Coverage Set Creation}

\paragraph{Flow predictive intervals (scalar outputs)} \label{app:flow_intervals}

In the Monte-Carlo policy-evaluation experiments, the target variable is one-dimensional. We write each example as $(x_i, c_i)$, where $c_i$ is the conditioning input and $x_i \in \mathbb{R}$ is the target value. Given a trained conditional flow with flow map $\Phi^\theta_1(\cdot \mid c)$ and base distribution $\rho_0$, the learned conditional distribution at condition $c$ is the law of
\[
\Phi^\theta_1(Z \mid c), \qquad Z \sim \rho_0.
\]

Because the output is one-dimensional, flow trajectories do not cross. The map $z \mapsto \Phi^\theta_1(z \mid c)$ is therefore order-preserving, and central predictive intervals can be obtained by transporting quantiles of the base distribution. For a nominal coverage level $\alpha \in (0, 1)$, let
\[
z^-_\alpha = \operatorname{Quantile}_{(1-\alpha)/2}(\rho_0),
\qquad
z^+_\alpha = \operatorname{Quantile}_{(1+\alpha)/2}(\rho_0).
\]
We define the raw flow predictive interval by
\begin{align}
C^{\mathrm{flow}}_{\theta, \alpha}(c)
=
\left[ \Phi^\theta_1(z^-_\alpha \mid c), \, \Phi^\theta_1(z^+_\alpha \mid c) \right].
\label{eq:flow_interval}
\end{align}
Equivalently, this is the central $\alpha$-mass interval of the learned one-dimensional conditional distribution. The corresponding interval width is
\begin{align}
W^{\mathrm{flow}}_\alpha(c)
=
\Phi^\theta_1(z^+_\alpha \mid c) - \Phi^\theta_1(z^-_\alpha \mid c).
\end{align}
For a dataset $\mathcal D$, we report the empirical coverage $\widehat{\mathrm{Cov}}_\alpha(\mathcal D)$ and the average interval width
\begin{align}
\overline W^{\mathrm{flow}}_\alpha(\mathcal D)
=
\frac{1}{|\mathcal D|}
\sum_{(x_i, c_i) \in \mathcal D}
W^{\mathrm{flow}}_\alpha(c_i).
\end{align}

\paragraph{KNN predictive sets (high-dimensional outputs)} \label{app:knn_sets}

For high-dimensional outputs, central quantile intervals are not well-defined, and the standard conformal baselines (split conformal, CQR) do not generalize naturally to multivariate outputs. We instead use a sample-based KNN predictive set, similar to what is proposed in \citet{wang2022probabilisticconformalpredictionusing}.

For a condition $c$, draw $\tilde x_1, \ldots, \tilde x_M \sim \hat p_\theta(\cdot \mid c)$. For any candidate point $x$, define $d_k(x; c)$ as the distance from $x$ to its $k$-th nearest neighbor among the generated samples. We compute the same score for the generated samples themselves using leave-one-out nearest neighbors, and define
\[
\tau_{\theta, \alpha}(c)
=
\mathrm{Quantile}_\alpha\!\left( \{ d_k(\tilde x_m; c) \}_{m=1}^M \right).
\]
The KNN predictive set is
\begin{align}
C^{\mathrm{KNN}}_{\theta, \alpha}(c)
=
\{ x : d_k(x; c) \le \tau_{\theta, \alpha}(c) \}.
\label{eq:knn_set}
\end{align}
A held-out target $x_i$ is counted as covered if $d_k(x_i; c_i) \le \tau_{\theta, \alpha}(c_i)$. Throughout, we use $k = 1$ for computing coverage metrics.

\subsection{Coverage Metrics}
\label{app:coverage_metrics}

In this section, we describe the coverage metrics reported in Figures \ref{fig:coverage_antmaze}-\ref{fig:coverage_bc}, \ref{fig:coverage_cube}-\ref{fig:coverage_sciplex}.

\paragraph{Coverage and set size.}
We report empirical coverage, the fraction of observed targets contained in the predictive sets, at nominal levels $\alpha\in\{0.75,0.90,0.95\}$.
We report this alongside average interval widths for scalar outputs and average KNN radii for multivariate outputs, to assess whether coverage is achieved with informative sets.

\paragraph{Mean absolute calibration error (CE).}
Coverage at the three displayed nominal levels provides only a partial picture of calibration.
We therefore report CE to summarize agreement between nominal and achieved coverage across a wider grid of levels, following \citet{cui2020calibrated,chung2021uncertainty}.

Let $C_{\theta,\alpha}(c)$ denote the model's predictive set at nominal level $\alpha$.
For levels $\alpha_1,\ldots,\alpha_J$, define
\[
\mathrm{CE}_{\mathrm{pop}}(\theta)
=
\frac{1}{J}
\sum_{j=1}^{J}
\left|
\Pr_{(C,X)\sim p^\star}
\!\left[X\in C_{\theta,\alpha_j}(C)\right]
-\alpha_j
\right|.
\]
On an evaluation dataset $\mathcal D=\{(x_i,c_i)\}_{i=1}^{N}$, we estimate this quantity using empirical coverage:
\[
\widehat{\mathrm{CE}}(\mathcal D;\theta)
=
\frac{1}{J}
\sum_{j=1}^{J}
\left|
\frac{1}{N}
\sum_{i=1}^{N}
\mathbf 1\!\left\{x_i\in C_{\theta,\alpha_j}(c_i)\right\}
-\alpha_j
\right|.
\]
Smaller values indicate closer agreement between nominal and observed coverage, averaged over conditioning inputs and the evaluated levels.

\paragraph{Scalar outputs.}
We use the central predictive intervals defined above, bounded by the model's $(1-\alpha)/2$ and $(1+\alpha)/2$ quantiles.
CE therefore summarizes the same interval-coverage property shown in the coverage panels, across additional nominal levels.

\paragraph{Multivariate outputs.}
We use the KNN predictive sets defined above.
A target is covered when its nearest-neighbor distance to the predictive sample cloud does not exceed the threshold for the corresponding nominal level.
For CFMs, the cloud is generated by the learned conditional flow; for HGR, it is sampled from the predicted diagonal Gaussian.
We reuse each cloud across the evaluated levels and apply the same CE formula.

\subsection{Baselines} \label{app:conformal_baselines}

All conformal methods use a held-out calibration set $\mathcal D_{\mathrm{cal}}$ disjoint from the training set. The deterministic baselines use the same backbone architecture as the flow's conditioning encoder and are trained on the same training set with matched optimizer, learning rate, and number of gradient steps. Split conformal prediction is applied either to a deterministic regressor (\textbf{det + SC}) or to the conditional mean of the trained CFM (\textbf{CFM mean + SC}). \textbf{Conf CFM} instead applies conformal calibration directly to the raw predictive set produced by the CFM.

\paragraph{Split conformal prediction.}
Let $\hat f(c)$ denote a point predictor. For \textbf{det + SC}, $\hat f(c)$ is the output of the deterministic regressor, whereas for \textbf{CFM mean + SC}, $\hat f(c)$ is the conditional mean of the trained CFM.

On $\mathcal D_{\mathrm{cal}}$, we compute absolute residual scores
\[
r_i = |x_i-\hat f(c_i)|.
\]
Let $q^{\mathrm{split}}_\alpha$ be the $\alpha$-quantile of $\{r_i\}$, with the standard finite-sample correction. The split-conformal interval is
\[
C^{\mathrm{split}}_\alpha(c)
=
\left[
\hat f(c)-q^{\mathrm{split}}_\alpha,\,
\hat f(c)+q^{\mathrm{split}}_\alpha
\right],
\]
with width \(
W^{\mathrm{split}}_\alpha(c)
=
2q^{\mathrm{split}}_\alpha.
\)
Every input receives the same expansion, so split conformal achieves marginal coverage but is not locally adaptive.

For high-dimensional outputs, \textbf{CFM mean + SC} uses the analogous construction with Euclidean residual scores
\[
r_i
=
\left\|x_i-\hat f(c_i)\right\|_2.
\]
The resulting predictive set is the conformal ball
\[
C^{\mathrm{CFM\text{-}mean+SC}}_\alpha(c)
=
\left\{
x:
\left\|x-\hat f(c)\right\|_2
\leq
q^{\mathrm{split}}_\alpha
\right\}.
\]

\paragraph{Conformalized quantile regression.}
For \textbf{det + CQR}, at each coverage level $\alpha$, we train two predictors $\hat\ell_\alpha(c)$ and $\hat u_\alpha(c)$ using quantile losses for the lower and upper tails, with target quantile levels $(1-\alpha)/2$ and $(1+\alpha)/2$, respectively. Before conformal calibration, these define the interval \(
[\hat\ell_\alpha(c),\hat u_\alpha(c)].
\)

On $\mathcal D_{\mathrm{cal}}$, we compute CQR scores
\[
s_i
=
\max
\left\{
\hat\ell_\alpha(c_i)-x_i,\,
x_i-\hat u_\alpha(c_i)
\right\}.
\]
Let $q^{\mathrm{CQR}}_\alpha$ be the $\alpha$-quantile of $\{s_i\}$, with the standard finite-sample correction. The conformalized CQR interval is
\[
C^{\mathrm{CQR}}_\alpha(c)
=
\left[
\hat\ell_\alpha(c)-q^{\mathrm{CQR}}_\alpha,\,
\hat u_\alpha(c)+q^{\mathrm{CQR}}_\alpha
\right],
\]
with width
\[
W^{\mathrm{CQR}}_\alpha(c)
=
\hat u_\alpha(c)-\hat\ell_\alpha(c)
+
2q^{\mathrm{CQR}}_\alpha.
\]
Unlike split conformal prediction, CQR can produce different interval widths at different conditioning inputs.

\paragraph{Conformalized CFM.}
For scalar outputs, let
\[
C_{\theta,\alpha}(c)
=
\left[
\hat\ell_{\theta,\alpha}(c),
\hat u_{\theta,\alpha}(c)
\right]
\]
denote the raw CFM predictive interval. On $\mathcal D_{\mathrm{cal}}$, we compute the conformity scores
\[
s_i^{\mathrm{CFM}}
=
\max
\left\{
\hat\ell_{\theta,\alpha}(c_i)-x_i,\,
x_i-\hat u_{\theta,\alpha}(c_i)
\right\}.
\]
Let $q^{\mathrm{CFM}}_\alpha$ be their $\alpha$-quantile, with the standard finite-sample correction. The conformalized CFM interval is
\[
C^{\mathrm{ConfCFM}}_{\theta,\alpha}(c)
=
\left[
\hat\ell_{\theta,\alpha}(c)-q^{\mathrm{CFM}}_\alpha,\,
\hat u_{\theta,\alpha}(c)+q^{\mathrm{CFM}}_\alpha
\right].
\]

For high-dimensional outputs, we conformalize the KNN predictive set. Let $\tau_{\theta,\alpha}(c)$ denote the raw CFM KNN threshold defined above. We compute calibration scores
\[
s_i^{\mathrm{CFM}}
=
d_k(x_i;c_i)-\tau_{\theta,\alpha}(c_i)
\]
and let $q^{\mathrm{CFM}}_\alpha$ be their conformal quantile. The resulting predictive set is
\[
C^{\mathrm{ConfCFM}}_{\theta,\alpha}(c)
=
\left\{
x:
d_k(x;c)
\leq
\tau_{\theta,\alpha}(c)
+
q^{\mathrm{CFM}}_\alpha
\right\}.
\]
Thus, Conf CFM applies a conformal correction to the boundary of the raw CFM predictive set without changing the underlying learned distribution.

\paragraph{Deep Evidential Regression.}
We use the standard Deep Evidential Regression method of \cite{amini2020deep} as an additional probabilistic uncertainty baseline. DER is evaluated only for scalar outputs, and its predictive intervals are obtained from its learned predictive distribution. We report DER at its best validation checkpoint.

\paragraph{Heteroscedastic Gaussian Regression.}
Heteroscedastic Gaussian Regression models the conditional distribution as
\[
p_\psi(x\mid c)
=
\mathcal N
\left(
x;
\mu_\psi(c),
\Sigma_\psi(c)
\right).
\]
For scalar outputs, the network predicts a conditional mean $\mu_\psi(c)$ and variance $\sigma_\psi^2(c)>0$, and is trained by minimizing the Gaussian negative log-likelihood
\[
\mathcal L_{\mathrm{HGR}}
=
\frac{1}{2}
\left[
\frac{
\left(x-\mu_\psi(c)\right)^2
}{
\sigma_\psi^2(c)
}
+
\log \sigma_\psi^2(c)
\right],
\]
up to an additive constant. The central $\alpha$-mass predictive interval is
\[
C^{\mathrm{HGR}}_{\psi,\alpha}(c)
=
\left[
\mu_\psi(c)
+
z_{(1-\alpha)/2}\sigma_\psi(c),
\,
\mu_\psi(c)
+
z_{(1+\alpha)/2}\sigma_\psi(c)
\right],
\]
where $z_q$ denotes the $q$-quantile of the standard normal distribution.

For $d$-dimensional outputs, HGR predicts a diagonal conditional covariance,
\[
\Sigma_\psi(c)
=
\operatorname{diag}
\left(
\sigma_{\psi,1}^2(c),
\ldots,
\sigma_{\psi,d}^2(c)
\right),
\]
and minimizes
\[
\mathcal L_{\mathrm{HGR}}
=
\frac{1}{2}
\sum_{j=1}^{d}
\left[
\frac{
\left(x_j-\mu_{\psi,j}(c)\right)^2
}{
\sigma_{\psi,j}^2(c)
}
+
\log \sigma_{\psi,j}^2(c)
\right].
\]

\paragraph{Computing $k$-NN radii across methods.}
For high-dimensional outputs, we express all methods' predictive-region sizes on the same $k$-NN radius scale used for the CFM. At each condition $c$, the sample-based radius is the $\alpha$-quantile of the leave-one-out $k$-nearest-neighbor distances among $K$ samples from the predictive distribution; we report its average over conditions. We hold $K$ and $k$ fixed across methods within each task to ensure a common sampling resolution. This preserves the CFM's original radius statistic while avoiding an ellipsoidal approximation to its predictive region.

For the CFM, we sample from the learned conditional flow; for HGR, from the predicted diagonal Gaussian. HGR radii are computed separately at each condition because Euclidean $k$-NN distances depend on the full shape of the diagonal covariance, not only on the geometric mean of its marginal standard deviations.

For the point-predictor split-conformal baselines (\textbf{det + SC} and \textbf{CFM mean + SC}), no per-condition sampling is required. At each nominal level $\alpha$, we interpret the conformal $L_2$ ball of radius $R_\alpha$ as the central $\alpha$-mass region of an isotropic Gaussian $\mathcal N(\mu(c),s_\alpha^2 I_D)$, where $\mu(c)$ is the corresponding point prediction and $D$ is the output dimension. This fixes
\[
s_\alpha
=
\frac{R_\alpha}{\sqrt{\chi^2_\alpha(D)}},
\]
where $\chi^2_\alpha(D)$ denotes the $\alpha$-quantile of the chi-squared distribution with $D$ degrees of freedom. Since pairwise distances are translation-invariant and scale linearly with $s_\alpha$, the corresponding radius is
\[
\tau_\alpha
=
\frac{R_\alpha}{\sqrt{\chi^2_\alpha(D)}}\,
q_{\alpha,k}(K,D),
\]
where $q_{\alpha,k}(K,D)$ is the reference radius obtained by applying the same leave-one-out $k$-NN quantile procedure to a cloud of $K$ standard Gaussian samples in $D$ dimensions. We estimate this reference quantity once by simulation for each $(\alpha,k,K,D)$ and reuse it for both ball-based baselines. The distance rescaling is exact under the isotropic-Gaussian representation; only the reference quantity is estimated numerically. This representation is used solely to express the ball radius on the common $k$-NN scale.

For \textbf{Conf CFM}, we report the sample-based CFM threshold after adding the conformal correction defined above, again averaged over conditions.

\subsection{Datasets and Training Setup}

All the plots are created using 3 random seeds. Dataset specific details and training setup are given below.

\subsubsection{Biology datasets}\label{app:bio_details}

We evaluate conditional flow matching on three single-cell perturbation benchmarks from the harmonized scPerturb collection \citep{peidli2024scperturb}: \textbf{Norman}, \textbf{SciPlex3}, and \textbf{PBMC cytokines}. In each dataset, a cell is represented by a PCA-reduced gene-expression vector $x \in \mathbb{R}^d$, and each sample is annotated with a perturbation condition $c$ together with a control condition. The perturbation prediction problem is cast as conditional generation: the source distribution is the population of control cells, the target distribution is the population of cells under the perturbation condition, and the conditioning variable $c$ is a pretrained embedding of the perturbation. With the source distribution of control cells fixed, the flow learns the distribution of perturbed-cell states conditioned on $c$.

\paragraph{Norman (Perturb-seq).}
The Norman et al. \citep{norman2019} Perturb-seq dataset consists of single-cell RNA-seq profiles of K562 cells subjected to CRISPRa-induced overexpression of one or two genes. We use the preprocessed AnnData release distributed by the authors of biolord \citep{piran2024biolord}, which contains $\sim$89k cells across $229$ single and double-gene perturbation conditions plus unperturbed controls; this is the same release used by the CellFlow benchmark \citep{klein2025cellflow}. Of biolord's predefined train/test partitions, we use the  ``train'' perturbation conditions ($\sim$42k cells) for our flow-matching experiments and the in-distribution cells are then randomly split into a training set and a validation set (80-20 split). The $7353$ unperturbed control cells form the pool from which $x_0$ samples are drawn at training time. The condition $c \in \mathbb{R}^{2560}$ is the concatenation of two ESM2-650M \citep{lin2023esm2} protein embeddings, one per perturbed gene. The target $x \in \mathbb{R}^{50}$ is the $50$-component PCA projection of the log-normalized expression vector, with PCA fit on all cells. 

The $x_0$ distribution is control-pool $+$ Gaussian noise: $x_0 = z + \sigma\,\epsilon$ with $z\sim\mathrm{Uniform}(\text{control pool})$, $\epsilon\sim\mathcal{N}(0,I)$, and $\sigma=1.0$. The velocity field is a CellFlow-style \citep{klein2025cellflow} three-branch MLP: a condition encoder ($[1024,1024]\!\to\!256$), a cell-state encoder ($[512,512,512]\!\to\!512$), and a sinusoidal time embedding ($128$d) followed by $[512,512,512]\!\to\!512$, fused by a $[1024,1024,1024]\!\to\!50$ decoder. Activations are SiLU throughout. We train with Adam at learning rate $10^{-3}$, batch size $1024$. Sampling at evaluation uses $100$ Euler steps.

\paragraph{PBMC (cytokine stimulation).}
The PBMC cytokine atlas released by Parse Biosciences \citep{parsebio_pbmc} contains single-cell RNA-seq profiles of peripheral blood mononuclear cells from 12 donors stimulated with 90 cytokines plus PBS controls ($\sim$9.7M cells over 2{,}000 highly-variable genes). We access this dataset through the CellFlow data loader \citep{klein2025cellflow} and restrict to a single donor (Donor 1), which contains $\sim$1.3m cells: $\sim$90k PBS-stimulated controls and $\sim$1.2m cytokine-stimulated cells across the 90 cytokines. To keep the per-condition cell count balanced, we subsample uniformly at random to at most $1{,}000$ cells per cytokine condition (and the same cap for controls), yielding $91{,}000$ cells which are split into a training set and a validation set ($80$--$20$ split). The $1000$ subsampled PBS controls form the pool from which $x_0$ samples are drawn at training time. The condition $c \in \mathbb{R}^{2560}$ is the ESM2 \citep{lin2023esm2} protein embedding of the cytokine. The target $x \in \mathbb{R}^{100}$ is the $100$-component PCA projection of the log-normalized expression vector, with PCA fit on all cells.

The $x_0$ distribution is again control-pool $+$ Gaussian noise with standard deviation $\sigma=10.0$. The velocity field used a condition encoder ($[1024,1024]\!\to\!256$), a cell-state encoder ($[512,512,512]\!\to\!512$), and a sinusoidal time embedding ($1024$d) followed by $[1024,1024,1024]\!\to\!1024$, fused by a $[1024,1024,1024]\!\to\!100$ decoder. We use SiLU activations, train with Adam at learning rate $10^{-3}$, batch size $1024$. Sampling at evaluation uses $1{,}000$ Euler steps.

\paragraph{SciPlex (small-molecule perturbations).}
The SciPlex3 dataset of Srivatsan et al.\ \citep{srivatsan2020} profiles single-cell transcriptomes of three cancer cell lines (A549, MCF7, K562) exposed to ${\sim}188$ small-molecule drugs at four doses each, with DMSO as the vehicle control. We restrict to the A549 lung-adenocarcinoma line, retrieve the dataset through the \texttt{pertpy} library \citep{heumos2024pertpy} via \texttt{pertpy.data.sciplex3\_raw()} and follow the preprocessing of the CellFlow benchmark \citep{klein2025cellflow}.
After filtering, the dataset contains $703$ in-distribution conditions covering $187$ drugs, with $\sim$136k perturbed cells and $\sim$3k DMSO control cells, which are split into train-val datasets (80-20 split). The condition $c \in \mathbb{R}^{1025}$ is $[\text{Morgan-1024-fingerprint}(\text{drug}),\,\log_{10}(\text{dose in nM})]$.
The target $x \in \mathbb{R}^{100}$ is the $100$-component PCA projection of the log-normalized expression vector, with PCA fit on all cells.

Similar to the previous two datasets, th $x_0$ distribution is constructed by adding gaussian noise with standard-deviation $= 10.0$ on the control pool. The velocity field architecture is the same as that used in the Norman experiments. Activations are SiLU throughout. We train with Adam at learning rate $10^{-3}$, batch size $1024$, for $2{,}000$ epochs. Sampling at evaluation uses $1000$ Euler steps.

\subsubsection{Policy Evaluation}
\label{app:policy_eval_dataset_details}

We construct three policy-evaluation datasets on top of the OGBench \citep{park2025ogbenchbenchmarkingofflinegoalconditioned} offline RL benchmark, one per environment. For each
  environment, we take the trained floq agent of \citet{agrawalla2025floqtrainingcriticsflowmatching} as the policy $\pi$ to be evaluated, and distill its true return $Q^{\pi}(s,a)$
  into a flow-matching predictor by Monte-Carlo: for each $(s,a)$ pair we average the discounted return over $n_{\mathrm{MC}}$ independent on-policy rollouts of $\pi$ from $(s,a)$,
  giving a stochastic target $\hat Q(s,a) \in \mathbb{R}$. The condition is $c=[s,a]$ and the target is $x=\hat Q(s,a)$.

  The per-environment details are:

  \begin{table}[h]
  \centering
  \small
  \setlength{\tabcolsep}{4pt}
  \begin{tabular}{lccccccccc}
  \toprule
  Dataset & $n_{\mathrm{MC}}$ & $H$ & $\gamma$ & $\dim s$ & $\dim a$ & $\dim c$ & $|\mathcal{D}|$ \\
  \midrule
  Cube-single          & 20 & 200  & 0.99  & 28 & 5  & 33 & 200{,}000 \\
  AntMaze-large        & 20 & 1000 & 0.998 & 29 & 8  & 37 & 200{,}000 \\
  Humanoid-Maze-medium & 20 & 2000 & 0.999 & 69 & 21 & 90 & 200{,}000 \\
  \bottomrule
  \end{tabular}

  \caption{$n_{\mathrm{MC}}$ is the number of independent on-policy rollouts averaged into $\hat Q(s,a)$, $H$ is the fixed rollout horizon, and $\gamma$ is the discount factor.}
  \label{tab:floq-datasets}
  \end{table}

  \paragraph{$x_0$ distribution.}
  All three FloQ runs use an isotropic Gaussian prior $x_0 \sim \mathcal{N}(0, 1)$ in the 1-D target space.

\paragraph{Architecture and training.}
For all three FloQ datasets we use the same MLP velocity field with hidden widths $[512,512,512,512]$ and ReLU activations, taking the full input $[c, x_t, \mathrm{TimeEmb}(t)]$, where $\mathrm{TimeEmb}$ is a $32$-dimensional cosine-geometric time embedding. The output is a scalar velocity. We optimize with Adam (lr $3\!\times\!10^{-4}$, no weight decay) using batch size $1024$.

\subsubsection{Behavior Cloning} \label{app:bc_details}

We train a flow-matching behavior-cloning policy on the state-based variant of the PushT benchmark introduced by \cite{chi2023diffusion}, where a 2-D end-effector pushes a T-shaped block toward a fixed target pose. The observation is the $5$-dimensional state $[\text{agent}_{x},\text{agent}_{y},\text{block}_{x},\text{block}_{y},\text{block}_{\theta}]$ and the action is the $2$-dimensional target end-effector position. Following the standard receding-horizon BC setup, each example predicts an $8$-step action chunk from $2$ steps of observation history. Both observations and actions are MinMax-normalized over the training set. We use an $80-20$ train-val split. The condition $c \in \mathbb{R}^{10}$ is the flattened observation history, the target $x \in \mathbb{R}^{16}$ is the flattened action chunk and a $128$-dimensional sinusoidal time embedding is used. The $x_0$ distribution is a standard gaussian.

The velocity field is an $8$-block residual MLP with embedding width $512$, feed-forward expansion factor $4$ (i.e., a $2{,}048$-dim FFN), GELU activations, LayerNorm, and dropout $0.1$. Time is encoded with a $128$-dim sinusoidal positional embedding with frequencies linear in $[0, 100]$ followed by an MLP. No image encoder is used (state-based variant). We train using Adam with learning rate $10^{-4}$, weight decay $10^{-5}$, gradient clipping at norm $10$ and batch size $1024$.

\subsection{Determining the Underfitting Regions}
\label{app:underfitting_cutoffs}

We determine the shaded underfitting regions in Figures \ref{fig:coverage_antmaze}-\ref{fig:coverage_bc}, \ref{fig:coverage_cube}-\ref{fig:coverage_sciplex} using a patience-based early stopping rule
that detects sustained stagnation in the validation flow-matching
loss. We use the same procedure across tasks, with patience equal
to $10\%$ of the full training budget and a relative-improvement
threshold shared within each task type.

We mark the underfitting cutoff when the mean validation
flow-matching loss fails to improve by at least the specified
relative threshold for a patience period equal to $10\%$ of the
training run. Each qualifying improvement resets patience and
becomes the reference for subsequent improvements.

We use minimum relative improvements of $2\%$ for policy evaluation,
$0.01\%$ for cell perturbation, and $5\%$ for PushT
(Table~\ref{tab:underfitting_cutoff_settings}).
Training continues for the full budget; the detected cutoff
determines the shaded region in the figures. %

\begin{table}[t]
\centering
\small
\begin{tabular}{llrr}
\hline
Dataset & Training budget & Min.\ improvement & Patience ($10\%$ of each training budget) \\
\hline
Antmaze      & 35,000 steps & $2\%$    & 3,500 steps \\
Cube-Single  & 35,000 steps & $2\%$    & 3,500 steps \\
Humanoidmaze & 6,000 steps  & $2\%$    & 600 steps \\
Norman       & 400 epochs   & $0.01\%$ & 40 epochs \\
PBMC         & 2,000 epochs & $0.01\%$ & 200 epochs \\
SciPlex      & 2,000 epochs & $0.01\%$ & 200 epochs \\
PushT        & 75,000 steps & $5\%$    & 7,500 steps \\
\hline
\end{tabular}
\caption{Parameters used to determine the underfitting cutoffs.
Patience is $10\%$ of the full training budget for every dataset.}
\label{tab:underfitting_cutoff_settings}
\end{table}

\subsection{Additional Experiments}
\label{app:additional_coverage_exp}

\begin{figure}[htbp]
    \centering
    \includegraphics[width=1.0\textwidth]{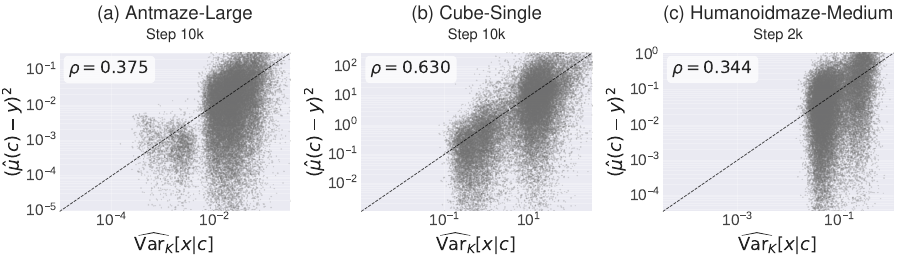}
    \caption{Pointwise conditional variance versus squared prediction error of the CFM mean on the three scalar-output policy-evaluation tasks. Each point corresponds to a conditioning input and its observed target; $\rho$ denotes Spearman correlation, and the dashed line denotes equality. The variance is larger precisely where the model's mean is less accurate, consistent with the interpretation that underfitting is expressed as conditional sample variance.}
    \label{fig:policy_variance_error}
\end{figure}

\paragraph{Policy Evaluation}

Section~\ref{sec:coverage} reports policy-evaluation results for \texttt{Antmaze-Large} in the main paper (Figure~\ref{fig:coverage_antmaze}). The same analysis on the remaining two OGBench environments, \texttt{Cube-Single} and \texttt{Humanoidmaze-medium}, is shown in Figures~\ref{fig:coverage_cube} and \ref{fig:coverage_hmmaze} below. Both figures use the layout described in Section~\ref{sec:coverage}: the top row shows training loss and flow coverage at three nominal levels, and the bottom row shows mean absolute calibration error (CE) along with mean interval widths for the flow, split conformal, and CQR. The qualitative findings discussed in Section~\ref{sec:coverage}, all hold consistently across these two additional environments.

\begin{figure}[h!]
    \centering
    \includegraphics[width=1.0\textwidth]{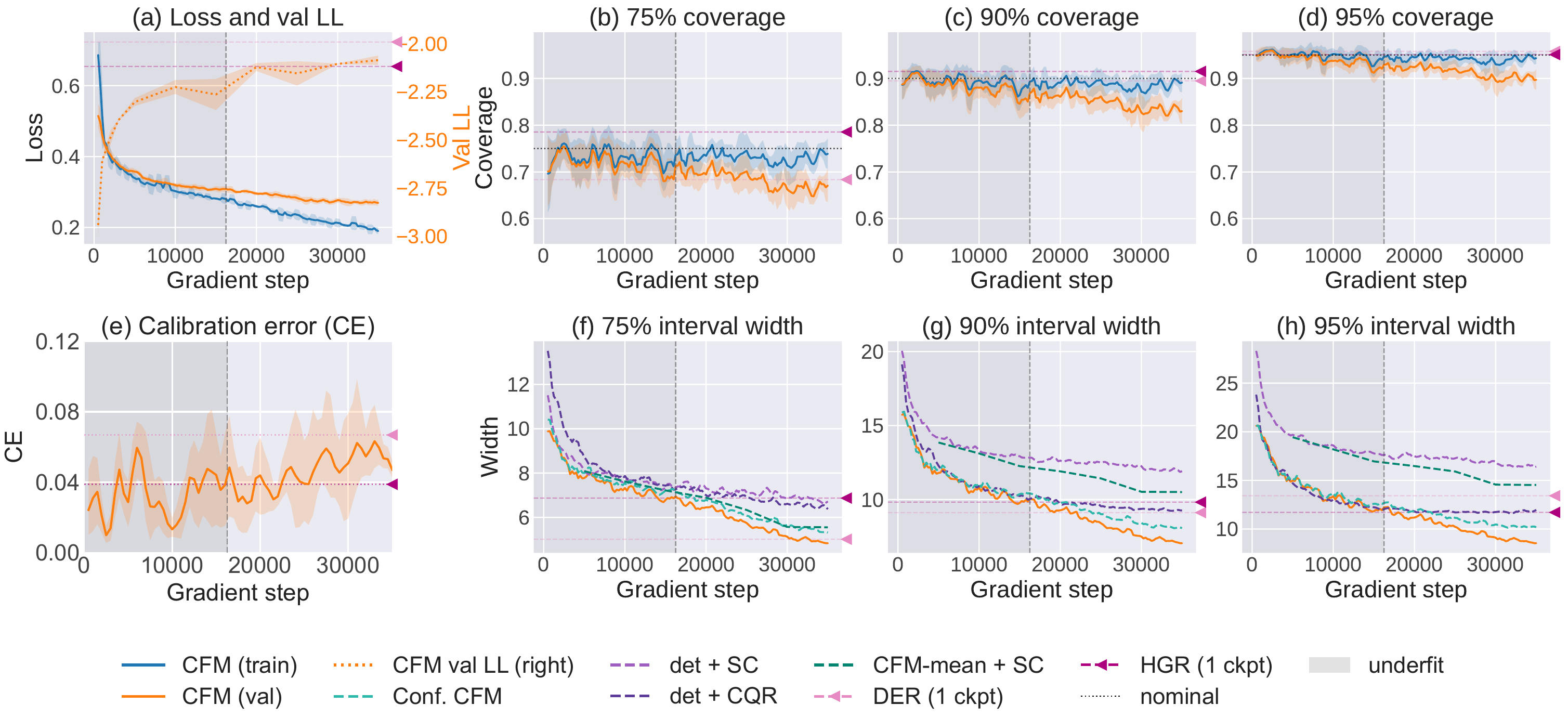}
    \caption{Coverage on \texttt{Cube-Single}. Layout matches Figure~\ref{fig:coverage_antmaze}.} 
    \label{fig:coverage_cube}
\end{figure}

\begin{figure}[h!]
    \centering
    \includegraphics[width=1.0\textwidth]{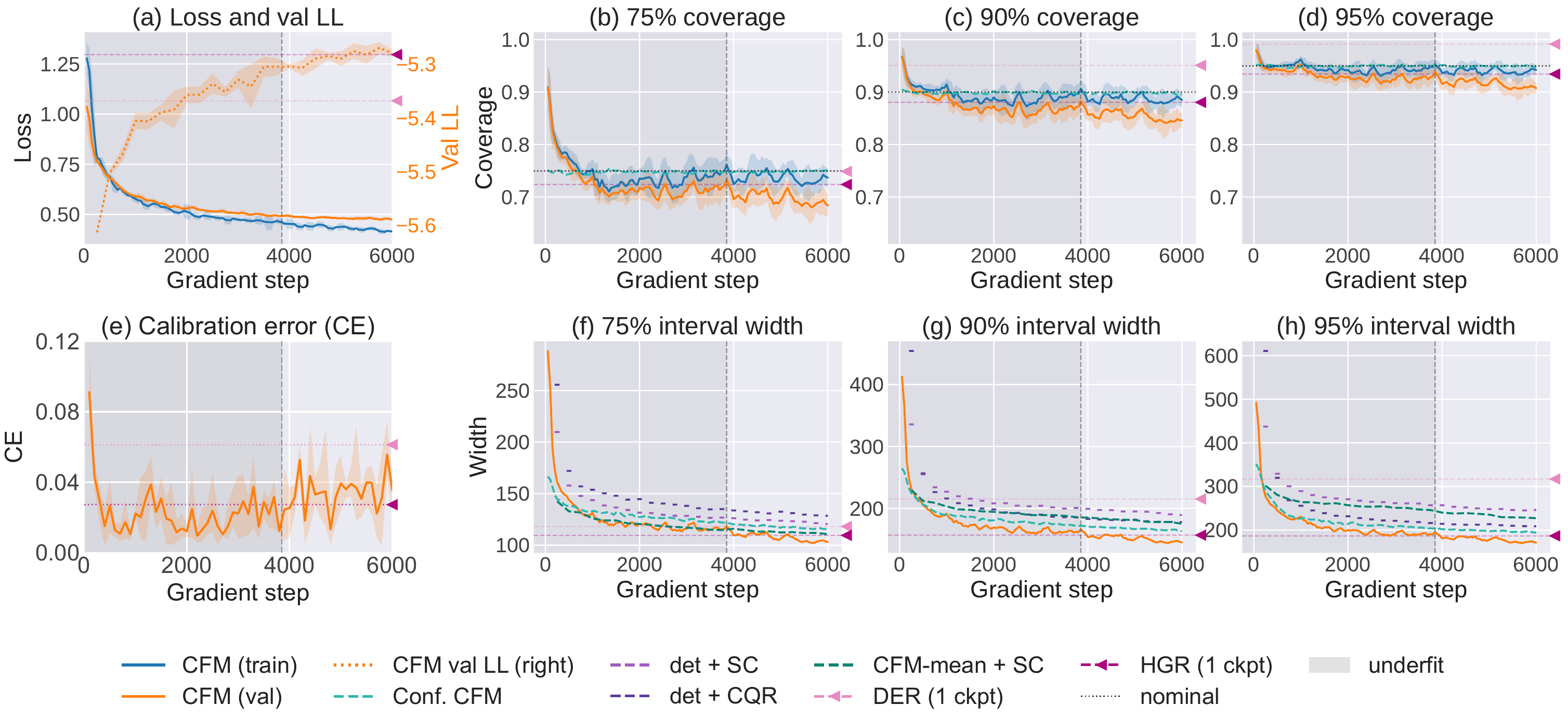}
    \caption{Coverage on \texttt{Humanoidmaze-medium}. Layout matches Figure~\ref{fig:coverage_antmaze}.} 
    \label{fig:coverage_hmmaze}
\end{figure}

\paragraph{Cell perturbation} \label{app:additional_bio}

Section~\ref{sec:coverage} reports coverage results for Norman in the main paper (Figure~\ref{fig:coverage_bio}). The same analysis on the remaining PBMC cytokines and scPerturb benchmark SciPlex3 (drug-dose perturbations), is shown in figures \ref{fig:coverage_pbmc} and \ref{fig:coverage_sciplex} respectively. The figures use the same layout as Figure~\ref{fig:coverage_bio}: training and validation flow-matching loss alongside KNN-based coverage at the three nominal levels.

The qualitative findings discussed in Section~\ref{sec:coverage} hold here as well: training coverage remains at or above nominal throughout training, and validation coverage is close to nominal during the underfitting regime. The flow's predictive sets are approximately calibrated on SciPlex3 and PBMC, consistent with the pattern observed on Norman.

\begin{figure}[h!]
    \centering
    \includegraphics[width=1.0\textwidth]{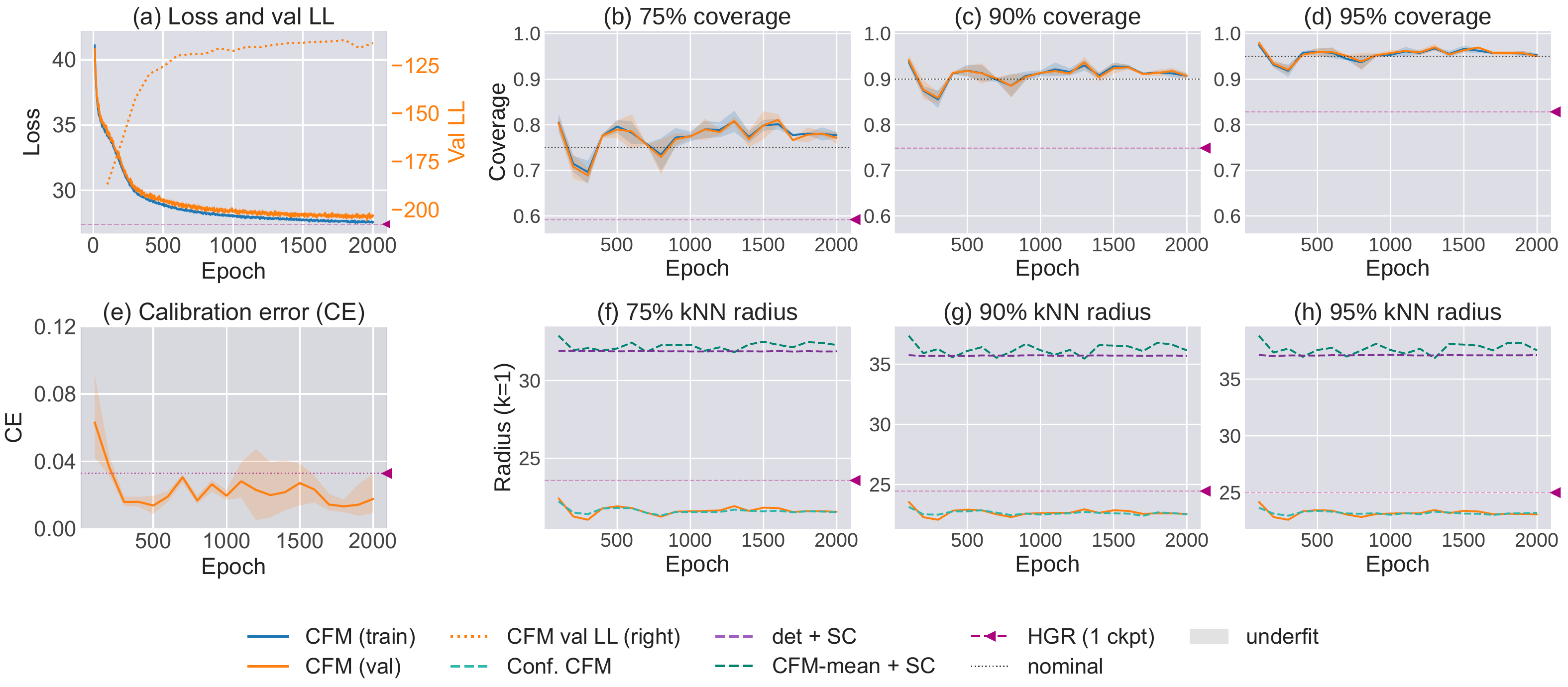}
    \caption{Coverage on PBMC. Layout matches Figure~\ref{fig:coverage_bio}.}
    \label{fig:coverage_pbmc}
\end{figure}

\begin{figure}[h!]
    \centering
    \includegraphics[width=1.0\textwidth]{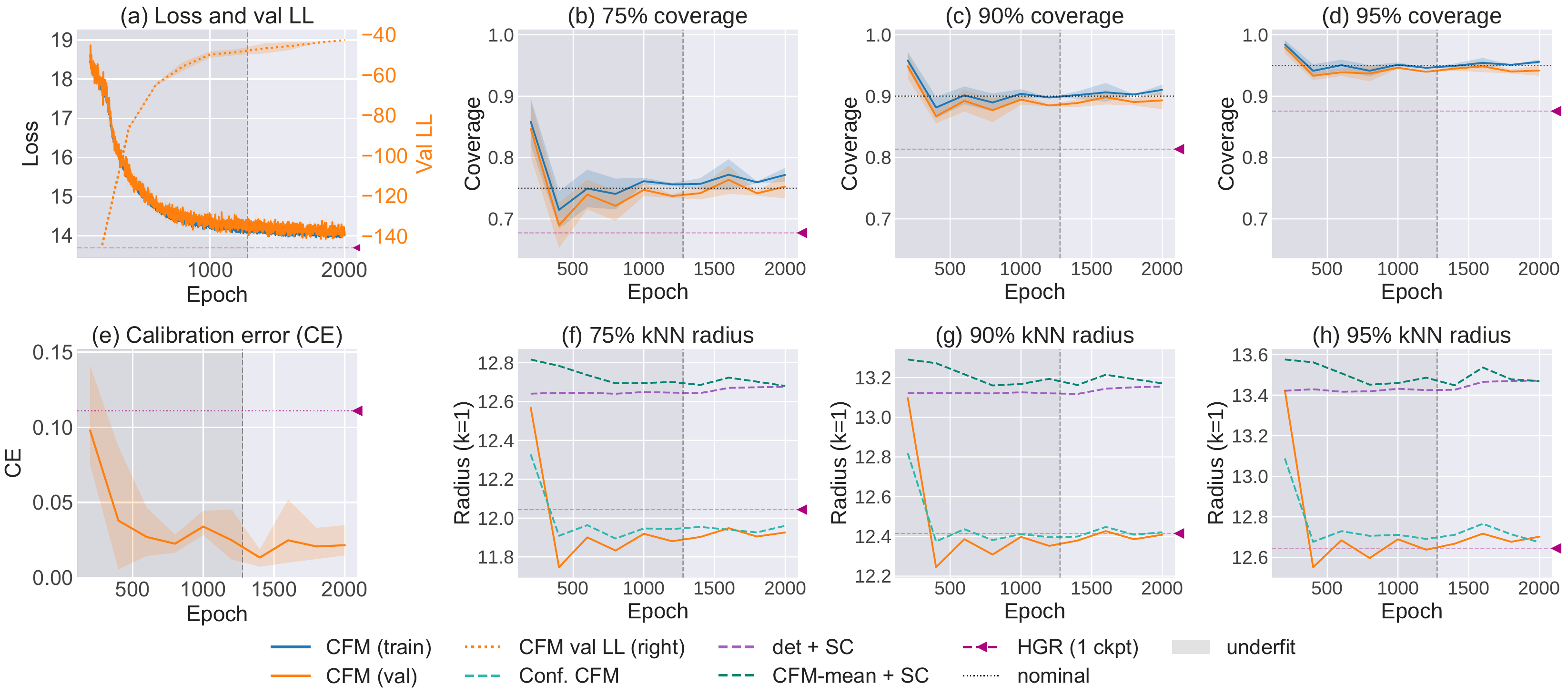}
    \caption{Coverage on SciPlex3 (drug-dose perturbations). Layout matches Figure~\ref{fig:coverage_bio}.}
    \label{fig:coverage_sciplex}
\end{figure}

%% file: body/app_variance_inflation.tex
\section{Variance Inflation} 

\subsection{Section \ref{sec:variance_inflation} Experimental details}
\label{app:synthetic_experimental_details}

\subsubsection{Flow models versus deterministic regression}
\label{app:synthetic_flow_vs_regression}

Figure~\ref{fig:flow_vs_reg} uses the following true data distribution:
\[
C \sim \operatorname{Uniform}[0,1],
\qquad
X_1 \mid C=c \sim
\mathcal{N}\!\left(f^\star(c),0.1^2\right),
\qquad
f^\star(c)=\sin(2\pi c)+\sin(20\pi c).
\]
We sample a fixed dataset of 128 observations and split it into 103 training observations and 25 validation observations. 

The deterministic regressor and neural velocity field are MLPs of width 128 and depth four, with SiLU activations. The regressor maps the scalar condition $c$ to a scalar prediction. The velocity field takes the three raw scalar inputs $(x_t,t,c)$ and predicts a scalar velocity.

Both models use full-batch Adam with learning rate $2\times10^{-3}$. Training duration is determined by early stopping. The stopping configurations are:
\begin{center}
\begin{tabular}{lrrl}
\hline
Model & Maximum epochs & Patience & Validation criterion \\
\hline
Deterministic regressor & 3,000 & 200 & Mean squared error \\
CFM & 6,000 & 400 & Flow-matching loss \\
\hline
\end{tabular}
\end{center}
Both use a minimum improvement threshold of $10^{-5}$. CFM samples are generated using 100 forward Euler steps.

\subsubsection{Variance inflation under different smoothing mechanisms}
\label{app:synthetic_smoothing_estimators}

Figure~\ref{fig:reg_trends} uses
\[
C \sim \operatorname{Uniform}[0,1],
\qquad
X_1 \mid C=c \sim
\mathcal{N}\!\left(g(c),0.5^2\right),
\qquad
g(c)=\sin(10c)+\sin(75c).
\]
We compare finite-time neural-network training, explicit conditioning-direction regularization, and kernel ridge regression.

\paragraph{Neural velocity fields.}
The first two rows use an MLP with four hidden layers of width 256 and SiLU activations. The network takes the raw scalar inputs $(c,x_t,t)$ and outputs a scalar velocity. Training uses Adam with initial learning rate $10^{-3}$, cosine annealing over the full run, and batch size 4,096.

A fresh independent batch of population observations is drawn at every update. We sample
\[
T \sim \operatorname{Uniform}[0.001,0.999],
\qquad
X_T=(1-T)X_0+TX_1.
\]

For the second row, the objective is,
\[
\mathcal{L}_{\lambda}(\theta)
=
\mathbb{E}\!\left[
\left(v_\theta(C,X_T,T)-\dot X_T\right)^2
\right]
+
\lambda\,
\mathbb{E}\!\left[
\left(
\left.
\frac{\partial v_\theta(c,x,t)}{\partial c}
\right|_{(c,x,t)=(C,X_T,T)}
\right)^2
\right].
\]

\paragraph{Kernel ridge regression.}
The third row uses a fixed regression dataset of 500 condition--observation pairs. Each pair is expanded with ten independent Gaussian base draws, giving 5,000 regression points at each time slice. We use an RBF kernel on $(c,x_t)$ with a shared lengthscale of 0.05 in both coordinates and ridge parameters in $\{5,1,0.05,0.01\}$. The velocity estimator is obtained by a closed-form ridge solve, recomputed at each ODE time.

\subsubsection{Comparison with probabilistic regression baselines}
\label{app:synthetic_probabilistic_baselines}

\paragraph{Data and training protocol.}
The experiments in Section \ref{sec:excess_var_structure} use the following true data distribution:
\[
C \sim \operatorname{Uniform}[0,1],
\qquad
X_1 \mid C=c \sim
\mathcal{N}\!\left(g(c),0.1^2\right),
\qquad
g(c)=\sin(10c)+\sin(75c).
\]
Every gradient step uses 256 fresh independent population observations. 

All networks have three hidden layers of width 128 with ReLU activations and use Adam with learning rate $10^{-3}$, without weight decay. The CFM takes the raw inputs $(c,x_t,t)$ and predicts a scalar velocity. It is compared against the following four baselines:

\paragraph{Gaussian NLL.}
The network predicts a conditional mean $\mu_\theta(c)$ and variance
\[
s_\theta^2(c)
=
\operatorname{softplus}(a_\theta(c))+10^{-6}.
\]
where \(a_\theta(c)\) is the network’s second, unconstrained output. It is trained using Gaussian negative log-likelihood,
\[
\ell_\theta(x,c)
=
\frac{1}{2}
\left[
\log s_\theta^2(c)
+
\frac{(x-\mu_\theta(c))^2}{s_\theta^2(c)}
\right].
\]

\paragraph{Monte Carlo Dropout.}
We add dropout with probability 0.1 after each hidden ReLU activation and train using Gaussian NLL. Dropout remains active during predictive evaluation, and we average 2,048 stochastic Gaussian predictions.

\paragraph{Deep ensemble.}
We train five independently initialized Gaussian-NLL networks. Every member receives the same fresh observation batches. The predictive distribution is the equally weighted mixture of the five members' Gaussian predictions.

\paragraph{Variational Bayesian neural network.}
Every weight and bias has an independent Gaussian variational posterior,
\[
q_\phi(w)
=
\prod_j
\mathcal{N}\!\left(
w_j;\mu_j,\operatorname{softplus}(\rho_j)^2
\right).
\]
The priors are zero-mean Gaussian, with weight standard deviation $1/\sqrt{d_{\mathrm{in}}}$ and bias standard deviation one. Posterior means start from the Gaussian network initialization, and posterior standard deviations start at 0.05 times the corresponding prior standard deviations.

Training minimizes expected batch-mean Gaussian NLL plus
$\operatorname{KL}(q_\phi\|p)/10{,}000$, with the KL summed over all weights and biases. One posterior network is sampled per optimizer update. Predictive evaluation averages 2,048 posterior network samples.

For all three mixture-based baselines, we retain the complete Gaussian-mixture predictive distribution.

\paragraph{Density comparison at comparable squared bias.}
In Figure \ref{fig:prob_methods_density}, for each method, we select a checkpoint from its training run with average squared conditional-mean bias around 0.45-0.5:
\[
\mathrm{Bias}^2
=
\int_0^1
\left(\widehat m(c)-g(c)\right)^2\,dc,
\qquad
\widehat m(c)
=
\mathbb{E}_{\widehat p}[X_1\mid C=c].
\]

The selected checkpoints are:
\begin{center}
\begin{tabular}{lrr}
\hline
Method & Gradient steps & Average squared bias \\
\hline
CFM & 5,000 & 0.457 \\
Gaussian NLL & 1,250 & 0.452 \\
Monte Carlo Dropout & 2,000 & 0.452 \\
Deep ensemble & 1,400 & 0.450 \\
Variational BNN & 29,250 & 0.498 \\
\hline
\end{tabular}
\end{center}

\paragraph{Coverage throughout training.}
For the coverage comparison in Figure \ref{fig:synth_coverage_comparison}, at nominal level $\alpha\in\{0.75,0.90,0.95\}$, we construct the central equal-tailed predictive interval
\[
C_{\theta,\alpha}(c)
=
[L_\alpha(c),U_\alpha(c)].
\]
Gaussian-NLL bounds are calculated analytically. The Monte Carlo Dropout, deep-ensemble, and BNN bounds are obtained by numerically inverting their full Gaussian-mixture CDFs.

For the CFM, we directly transport the six standard-normal quantiles defining these intervals using 500 forward Euler steps.

Because the true conditional distribution is known, population coverage is computed directly:
\[
\operatorname{Cov}_{\alpha}
=
\int_0^1
\left[
\Phi\!\left(
\frac{U_\alpha(c)-g(c)}{0.1}
\right)
-
\Phi\!\left(
\frac{L_\alpha(c)-g(c)}{0.1}
\right)
\right]dc.
\]

Curves show mean population coverage and $95\%$ confidence intervals computed using 5 seeds.

\subsection{Variance Inflation trends on additional functions}
\label{sec:add_synthetic}

In this section, we present evidence of variance inflation trends similar to Figure \ref{fig:reg_trends} across several more synthetic datasets. The experimental setup is the same as described in Appendix \ref{app:synthetic_smoothing_estimators}. Similar to Figure \ref{fig:reg_trends}, for all the experiments below, the learned conditional variance increases as the learned mean becomes smoother. The average learned conditional variance decreases monotonically as smoothing decreases across all three estimators.

\begin{figure}[h]
    \centering
    \begin{subfigure}[t]{0.9\textwidth}
        \centering
        \includegraphics[width=1.0\textwidth]{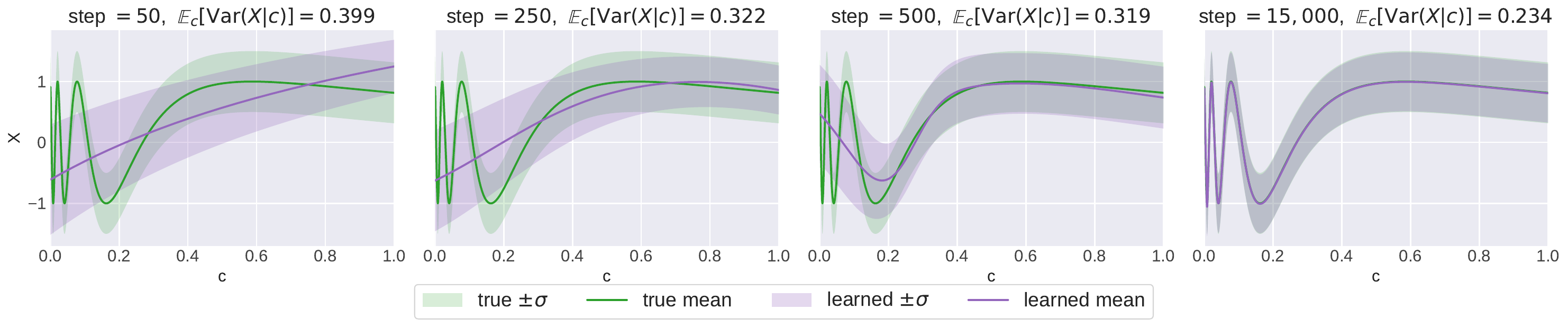}
        \caption{NN vector field parameterization: [Left to Right] Distributions learned with increasing training time}
    \end{subfigure}
    \begin{subfigure}[t]{0.9\textwidth}
        \centering
        \includegraphics[width=1.0\textwidth]{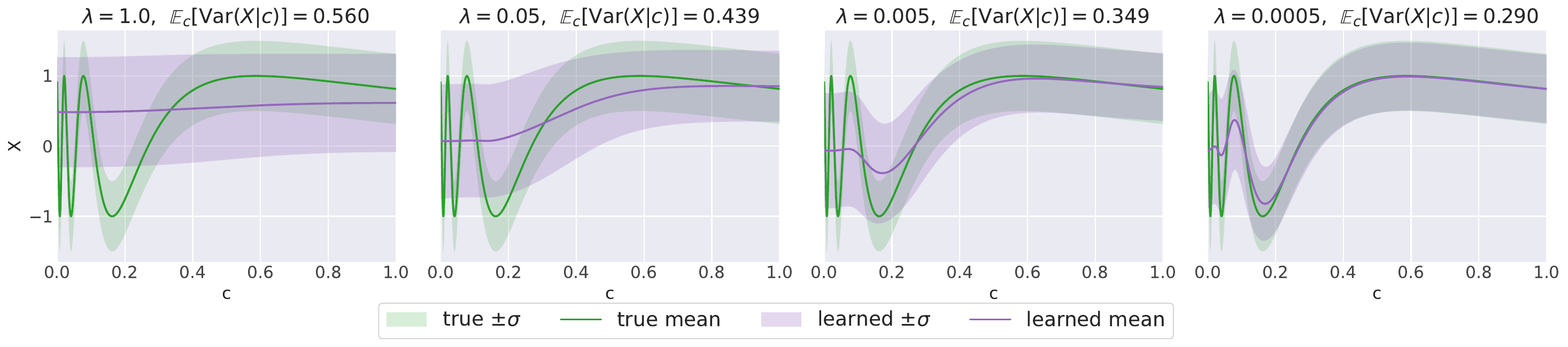}
        \caption{NN vector field parameterization with explicit regularization penalty $\lambda (\partial_c v )^2$}
    \end{subfigure}
    \begin{subfigure}[t]{0.9\textwidth}
        \centering
        \includegraphics[width=1.0\textwidth]{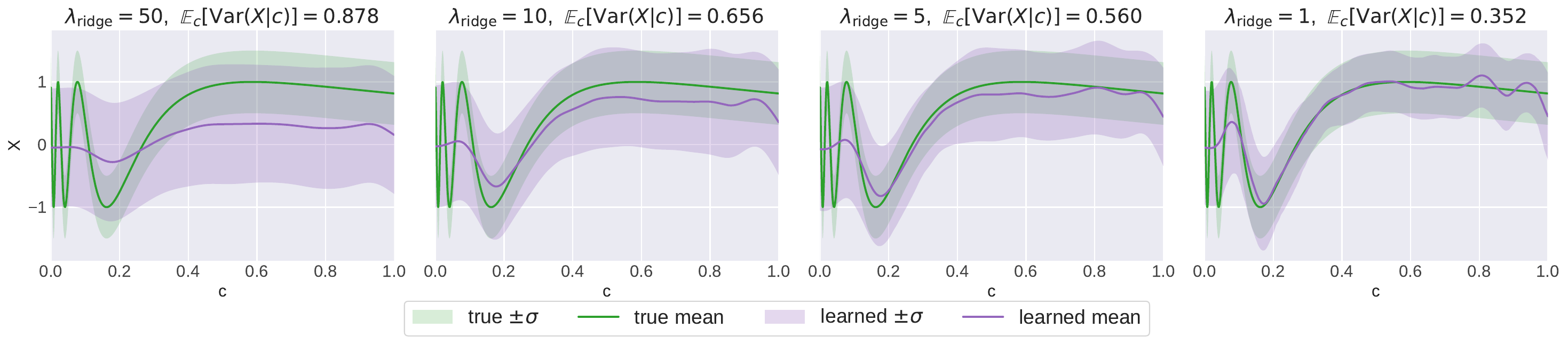}
        \caption{Vector field learning with kernel ridge regression with penalty $\lambda$}
    \end{subfigure}
    \caption{Training data: $C \sim \mathcal{U}([0,1]), \, X \mid C=c \sim \mathcal{N}\!\left(\sin(1/(c + 0.05)),\, \sigma^2 \right), \, \sigma=0.5$.}
    \label{fig:sin_inverse}
\end{figure}

\begin{figure}[h]
    \centering
    \begin{subfigure}[t]{0.9\textwidth}
        \centering
        \includegraphics[width=1.0\textwidth]{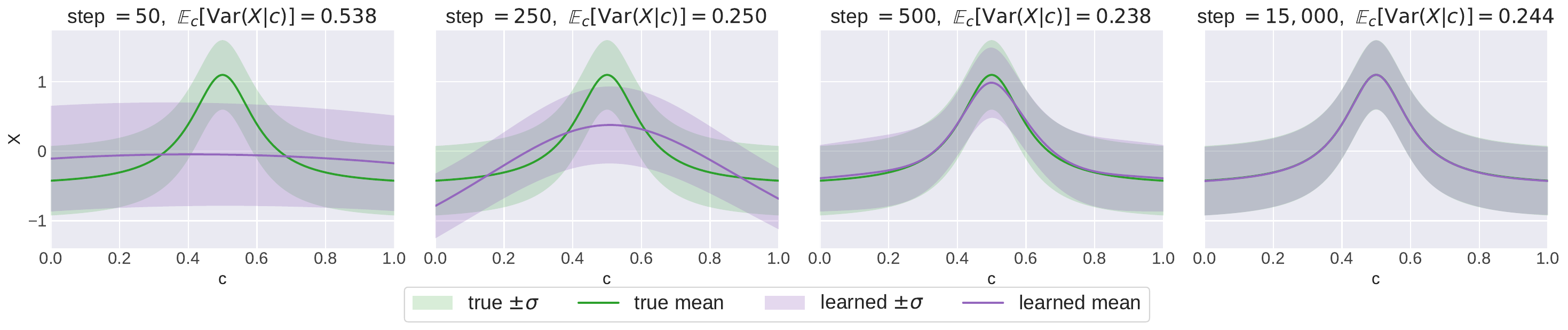}
        \caption{NN vector field parameterization: [Left to Right] Distributions learned with increasing training time}
    \end{subfigure}
    \begin{subfigure}[t]{0.9\textwidth}
        \centering
        \includegraphics[width=1.0\textwidth]{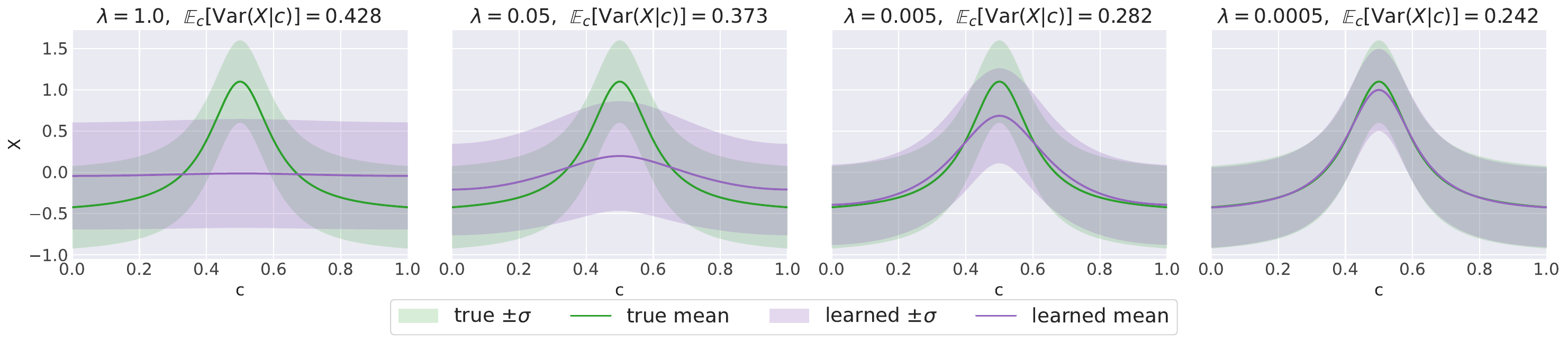}
        \caption{NN vector field parameterization with explicit regularization penalty $\lambda (\partial_c v )^2$}
    \end{subfigure}
    \begin{subfigure}[t]{0.9\textwidth}
        \centering
        \includegraphics[width=1.0\textwidth]{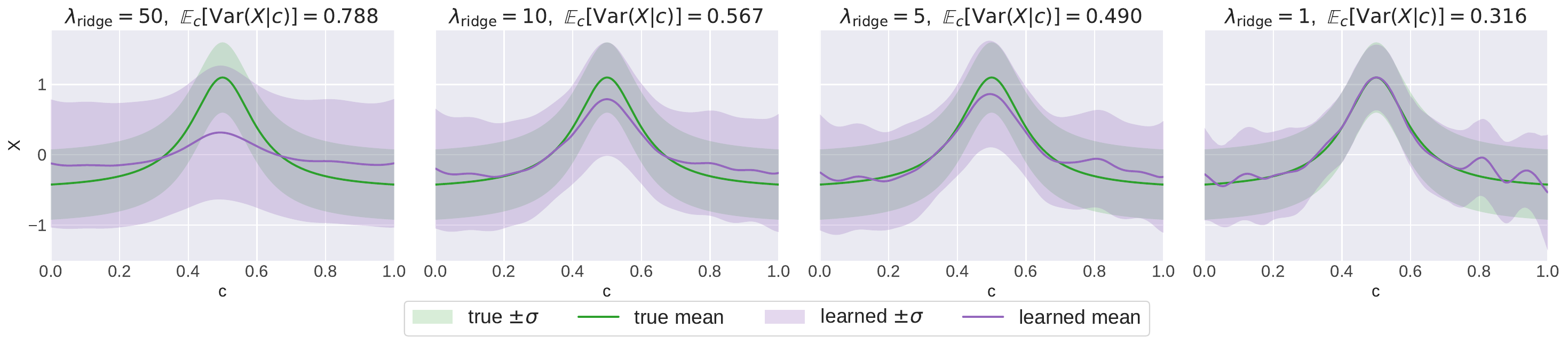}
        \caption{Vector field learning with kernel ridge regression with penalty $\lambda$}
    \end{subfigure}
    \caption{Training data: $C \sim \mathcal{U}([0,1]), \, X \mid C=c \sim \mathcal{N}\!\left(\frac{1.6}{1+80(c-0.5)^2}-0.5,\, \sigma^2 \right), \, \sigma=0.5$.}
    \label{fig:peak}
\end{figure}

\begin{figure}[h]
    \centering
    \begin{subfigure}[t]{0.9\textwidth}
        \centering
        \includegraphics[width=1.0\textwidth]{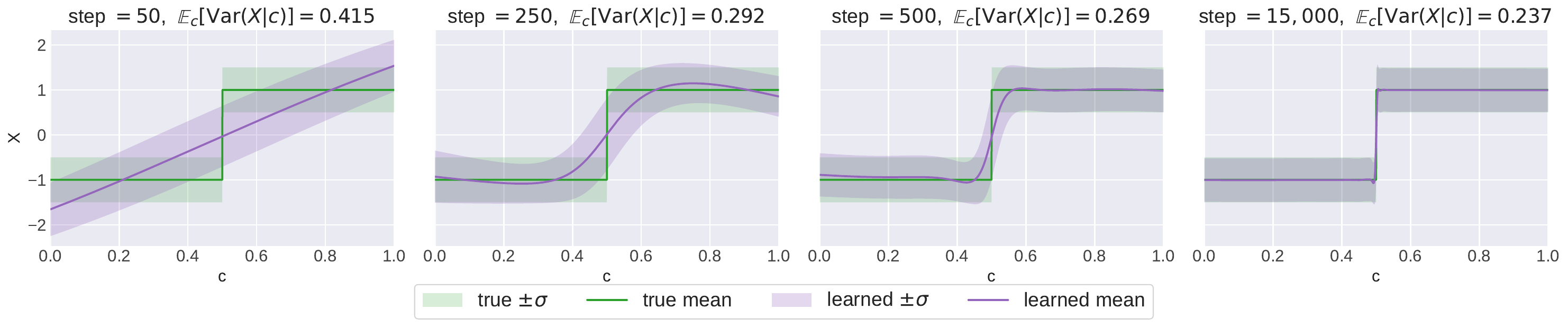}
        \caption{NN vector field parameterization: [Left to Right] Distributions learned with increasing training time}
    \end{subfigure}
    \begin{subfigure}[t]{0.9\textwidth}
        \centering
        \includegraphics[width=1.0\textwidth]{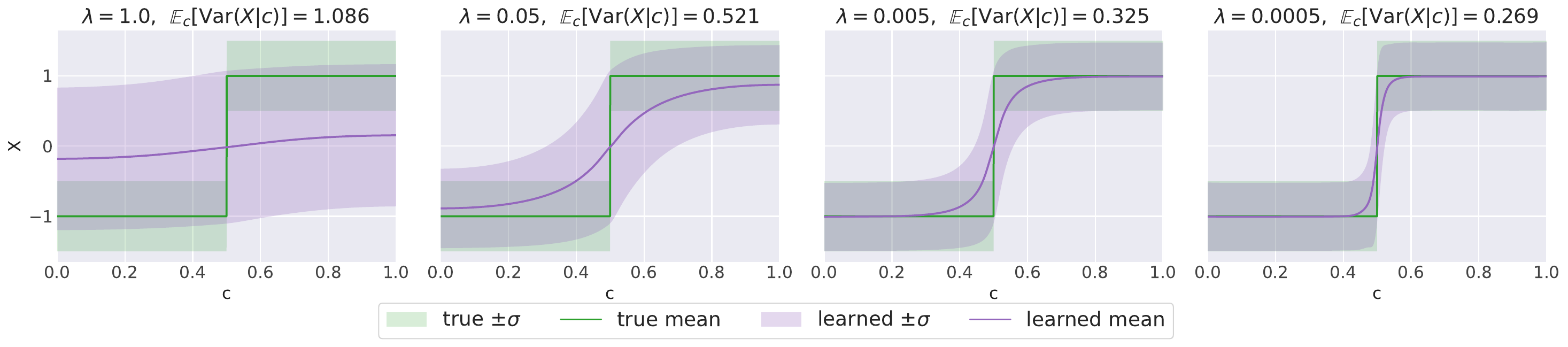}
        \caption{NN vector field parameterization with explicit regularization penalty $\lambda (\partial_c v )^2$}
    \end{subfigure}
    \begin{subfigure}[t]{0.9\textwidth}
        \centering
        \includegraphics[width=1.0\textwidth]{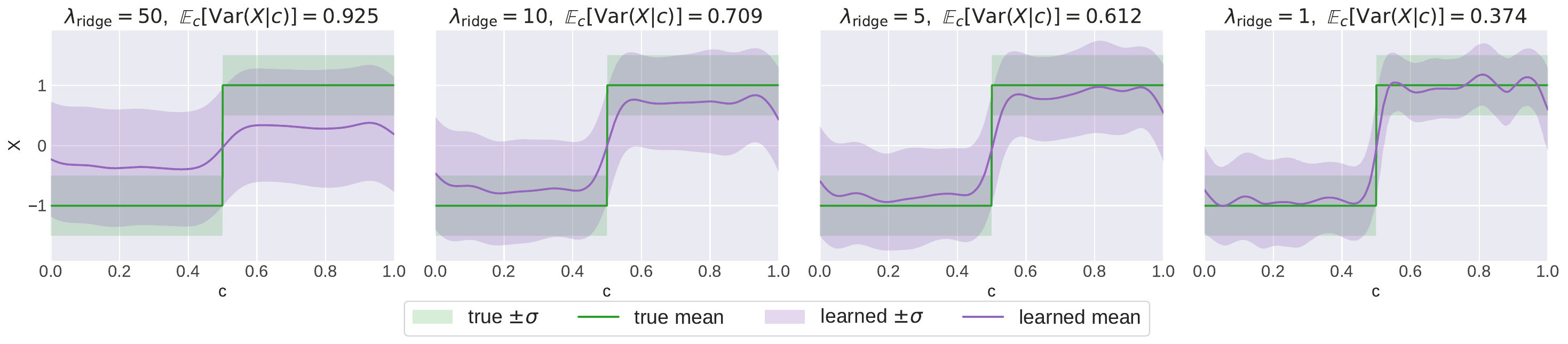}
        \caption{Vector field learning with kernel ridge regression with penalty $\lambda$}
    \end{subfigure}
    \caption{Training data: $C \sim \mathcal{U}([0,1]), \, X \mid C=c \sim \mathcal{N}\!\left(g(c),\, \sigma^2 \right), \, \sigma=0.5$, with $g(c)=1$ for $c>0.5$, and $g(c)=-1$ otherwise.}
    \label{fig:step}
\end{figure}
\FloatBarrier

\subsection{Variance Inflation beyond one-dimensional synthetic regression: Monte-Carlo Policy Evaluation} \label{sec:policy_eval_pca} 

In this section, we test whether the variance-inflation phenomenon persists when the conditioning variable is high-dimensional.
For this, we use a Policy Evaluation task, where we train a CFM to predict $Q$-values conditioned on state-action pairs $c=(s,a)$, using Monte-Carlo targets.
We evaluate on two OGBench \citep{park2025ogbenchbenchmarkingofflinegoalconditioned} environments, \texttt{Cube-Single} and \texttt{Antmaze-large}, where the conditioning variable is high-dimensional while the target remains scalar.

Figure \ref{fig:pca} visualizes the learned and empirical distributions against the first two principal components of the conditioning variable. At early training steps, the learned conditional distribution is noticeably more diffuse than the empirical target, but still covers the samples from the target distribution; at later checkpoints, the learned samples contract and align more closely with the target scatter. This mirrors the synthetic experiments: underfitting in the conditional vector field leads to excess spread, which decreases as training proceeds.

\begin{figure}[h!]
    \centering
    \begin{subfigure}[t]{0.9\textwidth}
        \centering
        \includegraphics[width=1.0\textwidth]{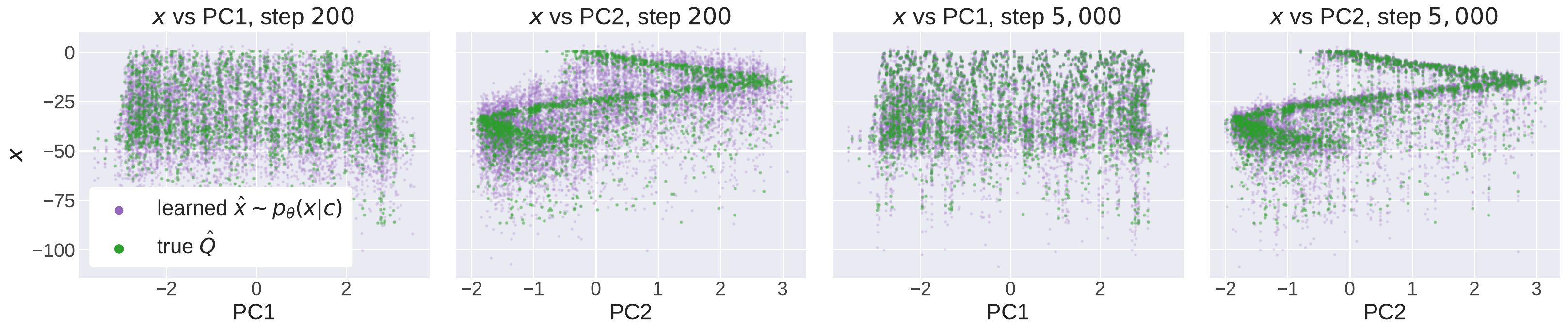}
        \caption{Cube Single}
    \end{subfigure}
    \begin{subfigure}[t]{0.9\textwidth}
        \centering
        \includegraphics[width=1.0\textwidth]{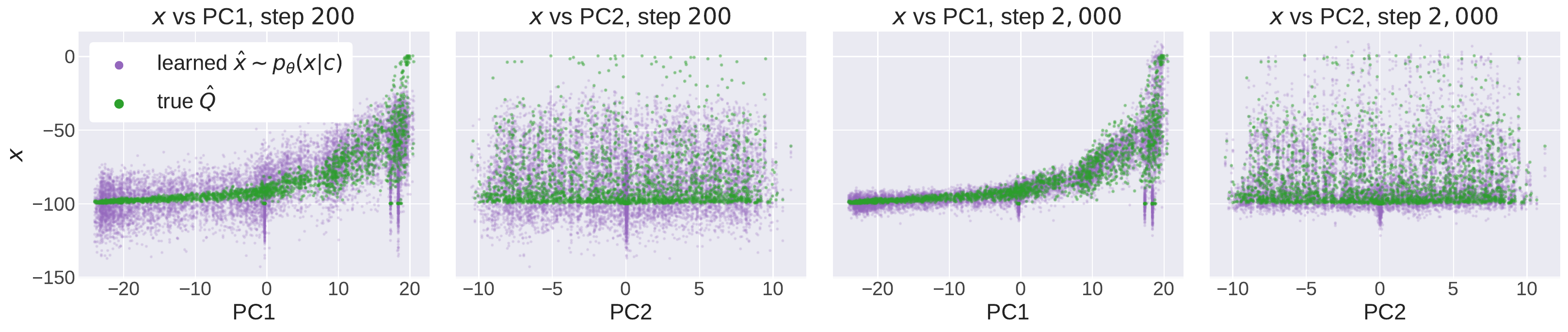}
        \caption{Antmaze}
    \end{subfigure}
    \caption{Policy Evaluation: learned (purple) and true (green) conditional distributions against the first two principal components of the state-action conditioning} \label{fig:pca} %
\end{figure}

Details of the Cube-Single and Antmaze-large policy-evaluation datasets are described in Appendix \ref{app:policy_eval_dataset_details}. Each contains 200,000 examples. The condition is the state--action pair $c=(s,a)$, and the scalar target averages discounted returns over 20 independent on-policy rollouts of a trained FloQ policy. Cube-Single has conditioning dimension 33, whereas for Antmaze-large it is 37.

The velocity field uses four hidden layers of width 512 with ReLU activations. Its inputs concatenate the full conditioning vector, current flow state, and a 32-dimensional cosine-geometric time embedding. Training uses Adam with learning rate $3\times10^{-4}$, batch size 1,024, and no weight decay.

\subsection{Comparison with other probabilistic methods: Additional experiments}
\label{app:prob_methods_additional_exp}

In Figure \ref{fig:prob_methods_density_detailed} we show more details of the experiment in Figure \ref{fig:prob_methods_density}, essentially instead of a single checkpoint, we show how the learned probability distributions of different methods evolve across training.

\begin{figure}[h]
    \centering
        \includegraphics[width=0.9\textwidth]{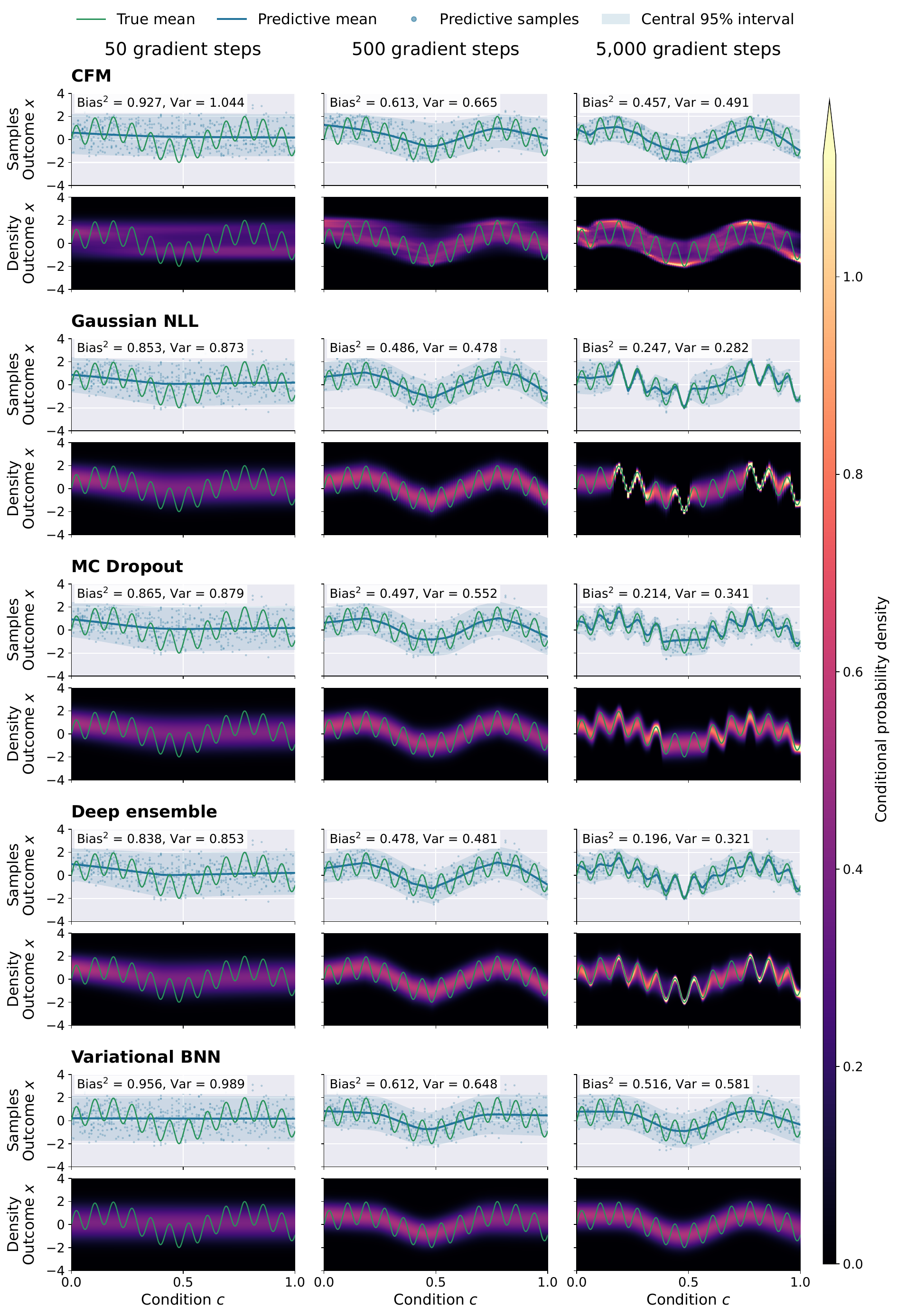}
        \caption{Evolution of learned probability distribution across training, of different methods.}
        \label{fig:prob_methods_density_detailed}
\end{figure}

Figure~\ref{fig:prob_methods_density_val_selected} examines the structure
of the learned distributions under fixed-dataset training and
validation-based checkpoint selection. We use a fixed dataset of
128 training observations and 64 validation observations, sampled from
\[
C \sim \operatorname{Uniform}[0,1],
\qquad
X_1\mid C=c
\sim
\mathcal{N}\!\left(\sin(10c)+\sin(150c),0.1^2\right).
\]

The CFM checkpoint selected minimizes validation flow-matching loss.
The regression baselines minimize validation predictive negative
log-likelihood, with each ensemble member selected separately.
For the BNN, the parameter KL term is divided by the training-set size,
128, instead of the effective sample size of 10,000 used in the
fresh-population experiments in Figure \ref{fig:prob_methods_density}. The architecture and other training hyperparameters remain the same (as detailed in Appendix \ref{app:synthetic_probabilistic_baselines}).

\begin{figure}[h]
    \centering
        \includegraphics[width=0.9\textwidth]{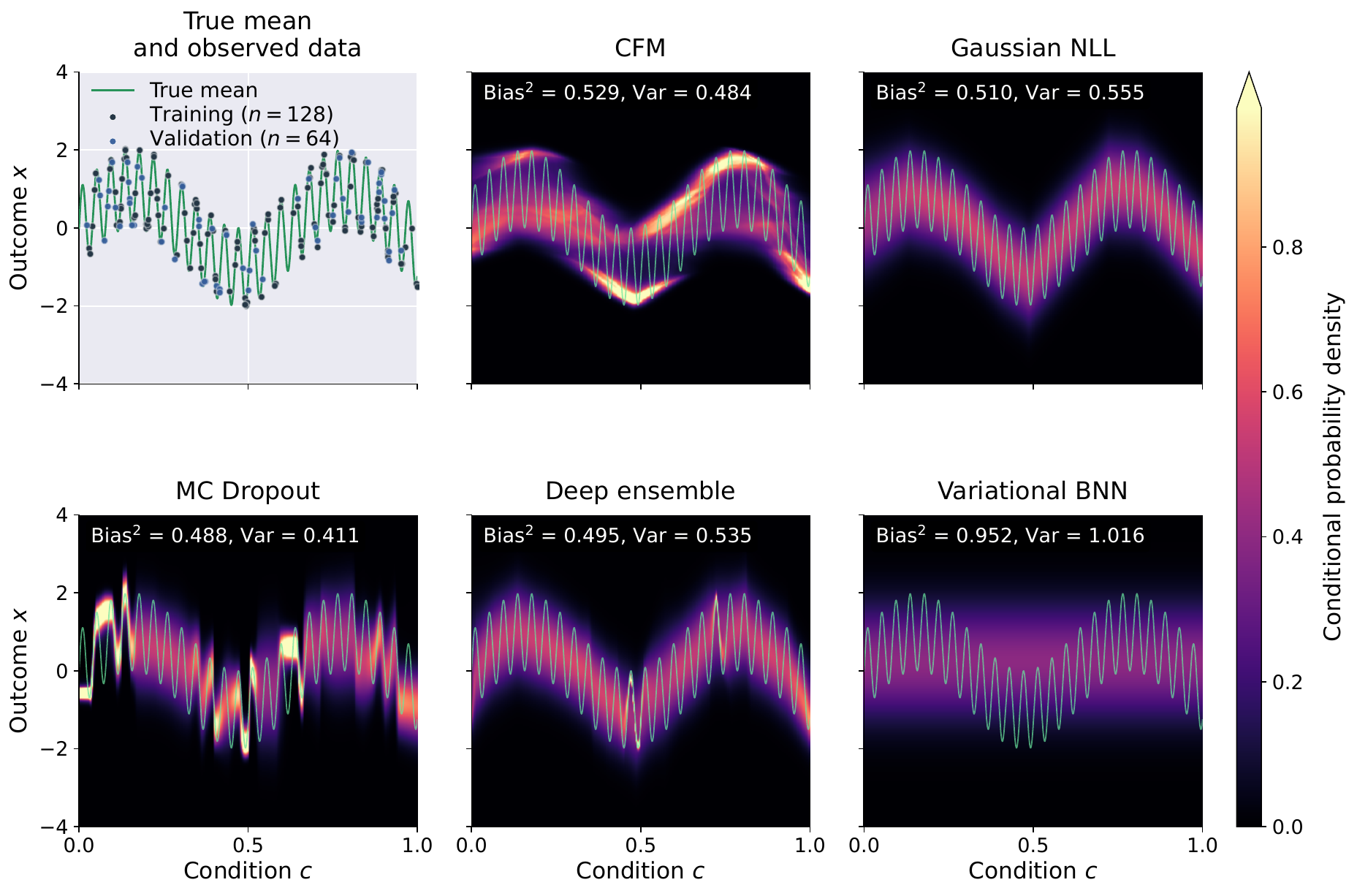}
        \caption{Conditional densities corresponding to minimum validation loss checkpoints.}
        \label{fig:prob_methods_density_val_selected}
\end{figure}

%% file: body/app_related_work.tex
\section{Relation to Prior Work on Inductive Biases of Flow and Diffusion Models}
\label{app:relation_prior_work}

Smoothing and underfitting biases in learned score and velocity fields have been studied extensively, but mostly for unconditional models. The question of why unconditional diffusion models can generate novel samples has been an extensive area of research, and prior work has linked this behavior to inductive biases and smoothing in the learned velocity field \citep{yoon2023diffusion, song2025selective, yi2023generalization, kamb2025analytic, niedoba2024towards, scarvelis2023closed}. Specifically, these works study how underfitting of the empirical score \citep{yoon2023diffusion, song2025selective, yi2023generalization}, architectural biases \citep{kamb2025analytic, niedoba2024towards}, and explicit score smoothing \citep{scarvelis2023closed} influence memorization, generalization, and the resulting unconditional distribution. 

However, how this smoothing bias affects the distributions learned by conditional flow models is still understudied. The conditional setting introduces an important distinction: the generated variable evolves under the learned velocity field, while the conditioning input remains fixed throughout generation. Smoothing across conditioning inputs therefore couples distinct conditional transport problems, whereas smoothing across the generated variable modifies the dynamics within a given transport. \cite{pfrommer2025conditionaldiffusionmodelactually} study how smoothing
across conditioning inputs affects the denoising path of diffusion models, and the resulting disagreement between sampling algorithms. Our work examines a complementary question: how the smoothing bias of neural networks across the conditioning variable affects the endpoint conditional distribution learned by these models, and uncovers one form of inductive bias in the learned distribution caused by vector field smoothing.

%% file: refs.bib
@misc{daniely2017sgdlearnsconjugatekernel,
      title={SGD Learns the Conjugate Kernel Class of the Network}, 
      author={Amit Daniely},
      year={2017},
      eprint={1702.08503},
      archivePrefix={arXiv},
      primaryClass={cs.LG},
      url={https://arxiv.org/abs/1702.08503}, 
}

@article{amini2020deep,
  title={Deep evidential regression},
  author={Amini, Alexander and Schwarting, Wilko and Soleimany, Ava and Rus, Daniela},
  journal={Advances in neural information processing systems},
  volume={33},
  pages={14927--14937},
  year={2020}
}

@inproceedings{yoon2023diffusion,
  title={Diffusion probabilistic models generalize when they fail to memorize},
  author={Yoon, TaeHo and Choi, Joo Young and Kwon, Sehyun and Ryu, Ernest K},
  booktitle={ICML 2023 workshop on structured probabilistic inference $\{$$\backslash$\&$\}$ generative modeling},
  year={2023}
}

@article{scarvelis2023closed,
  title={Closed-form diffusion models},
  author={Scarvelis, Christopher and Borde, Haitz S{\'a}ez de Oc{\'a}riz and Solomon, Justin},
  journal={arXiv preprint arXiv:2310.12395},
  year={2023}
}

@article{aithal2024understanding,
  title={Understanding hallucinations in diffusion models through mode interpolation},
  author={Aithal, Sumukh K and Maini, Pratyush and Lipton, Zachary C and Kolter, J Zico},
  journal={Advances in neural information processing systems},
  volume={37},
  pages={134614--134644},
  year={2024}
}

@article{song2025selective,
  title={Selective underfitting in diffusion models},
  author={Song, Kiwhan and Kim, Jaeyeon and Chen, Sitan and Du, Yilun and Kakade, Sham and Sitzmann, Vincent},
  journal={arXiv preprint arXiv:2510.01378},
  year={2025}
}

@misc{chisari2024learningroboticmanipulationpolicies,
      title={Learning Robotic Manipulation Policies from Point Clouds with Conditional Flow Matching}, 
      author={Eugenio Chisari and Nick Heppert and Max Argus and Tim Welschehold and Thomas Brox and Abhinav Valada},
      year={2024},
      eprint={2409.07343},
      archivePrefix={arXiv},
      primaryClass={cs.RO},
      url={https://arxiv.org/abs/2409.07343}, 
}

@article{yao2007early,
  title={On early stopping in gradient descent learning},
  author={Yao, Yuan and Rosasco, Lorenzo and Caponnetto, Andrea},
  journal={Constructive approximation},
  volume={26},
  number={2},
  pages={289--315},
  year={2007},
  publisher={Springer}
}

@misc{rouxel2024flowmatchingimitationlearning,
      title={Flow Matching Imitation Learning for Multi-Support Manipulation}, 
      author={Quentin Rouxel and Andrea Ferrari and Serena Ivaldi and Jean-Baptiste Mouret},
      year={2024},
      eprint={2407.12381},
      archivePrefix={arXiv},
      primaryClass={cs.RO},
      url={https://arxiv.org/abs/2407.12381}, 
}

@inproceedings{
    miller2024flowmm,
    title={Flow{MM}: Generating Materials with Riemannian Flow Matching},
    author={Benjamin Kurt Miller and Ricky T. Q. Chen and Anuroop Sriram and Brandon M Wood},
    booktitle={Forty-first International Conference on Machine Learning},
    year={2024},
    url={https://openreview.net/forum?id=W4pB7VbzZI}
}

@misc{zhang2024flowpolicyenablingfastrobust,
      title={FlowPolicy: Enabling Fast and Robust 3D Flow-based Policy via Consistency Flow Matching for Robot Manipulation}, 
      author={Qinglun Zhang and Zhen Liu and Haoqiang Fan and Guanghui Liu and Bing Zeng and Shuaicheng Liu},
      year={2024},
      eprint={2412.04987},
      archivePrefix={arXiv},
      primaryClass={cs.RO},
      url={https://arxiv.org/abs/2412.04987}, 
}

@inproceedings{vastola2025generalization,
  title = {Generalization through Variance: How Noise Shapes Inductive Biases in Diffusion Models},
  author = {Vastola, John J.},
  booktitle = {The Thirteenth International Conference on Learning Representations},
  year = {2025},
  url = {https://openreview.net/forum?id=7lUdo8Vuqa}
}

@inproceedings{kamb2025analytic,
  title = {An Analytic Theory of Creativity in Convolutional Diffusion Models},
  author = {Kamb, Mason and Ganguli, Surya},
  booktitle = {Proceedings of the 42nd International Conference on Machine Learning},
  year = {2025},
  url = {https://openreview.net/forum?id=ilpL2qACla}
}

@misc{pfrommer2025conditionaldiffusionmodelactually,
      title={Is Your Conditional Diffusion Model Actually Denoising?}, 
      author={Daniel Pfrommer and Zehao Dou and Christopher Scarvelis and Max Simchowitz and Ali Jadbabaie},
      year={2025},
      eprint={2512.18736},
      archivePrefix={arXiv},
      primaryClass={cs.LG},
      url={https://arxiv.org/abs/2512.18736}, 
}

@article{Lee_2020,
  title={Wide neural networks of any depth evolve as linear models under gradient descent},
  author={Lee, Jaehoon and Xiao, Lechao and Schoenholz, Samuel and Bahri, Yasaman and Novak, Roman and Sohl-Dickstein, Jascha and Pennington, Jeffrey},
  journal={Advances in neural information processing systems},
  volume={32},
  year={2019}
}

@misc{albergo2025stochasticinterpolantsunifyingframework,
      title={Stochastic Interpolants: A Unifying Framework for Flows and Diffusions}, 
      author={Michael S. Albergo and Nicholas M. Boffi and Eric Vanden-Eijnden},
      year={2025},
      eprint={2303.08797},
      archivePrefix={arXiv},
      primaryClass={cs.LG},
      url={https://arxiv.org/abs/2303.08797}, 
}

@article{heumos2024pertpy,                                            author  = {Heumos, Lukas and Ji, Yuge and May, Lilly and others},
      title   = {Pertpy: an end-to-end framework for perturbation analysis},                                   journal = {bioRxiv},                                      year    = {2024},                                         doi     = {10.1101/2024.08.04.606516}  
  }

@article{lin2023esm2,
    author  = {Lin, Zeming and Akin, Halil and Rao, Roshan and Hie, Brian and Zhu, Zhongkai and Lu, Wenting and Smetanin, Nikita and Verkuil, Robert and Kabeli, Ori and Shmueli, Yaniv and dos Santos Costa, Allan and Fazel-Zarandi, Maryam and Sercu, Tom and Candido, Salvatore and Rives, Alexander},
    title   = {Evolutionary-scale prediction of atomic-level protein structure},
    journal = {Science},
    volume  = {379},
    number  = {6637},
    pages   = {1123--1130},
    year    = {2023},
    doi     = {10.1126/science.ade2574}
}

@misc{parsebio_pbmc,                                                  title  = {{10 Million Human PBMCs in a Single Experiment}},
      author = {{Parse Biosciences}},                           year   = {2024},                                   
      note   = {Accessed: 2026-05-04},                          
      url    = {https://www.parsebiosciences.com/datasets/10-million-human-pbmcs-in-a-single-experiment/}
  }

@article{norman2019,
    author  = {Norman, Thomas M. and Horlbeck, Max A. and Replogle, Joseph M. and Ge, Alex Y. and Xu, Albert and Jost, Marco and Gilbert, Luke A. and Weissman, Jonathan S.},
    title   = {Exploring genetic interaction manifolds constructed from rich single-cell phenotypes},
    journal = {Science},
    volume  = {365},
    number  = {6455},
    pages   = {786--793},
    year    = {2019},
    doi     = {10.1126/science.aax4438}
}

@article{piran2024biolord,
    author  = {Piran, Zoe and Cohen, Niv and Hoshen, Yedid and Nitzan, Mor},
    title   = {Disentanglement of single-cell data with biolord},
    journal = {Nature Biotechnology},
    volume  = {42},
    pages   = {1678--1683},
    year    = {2024},
    doi     = {10.1038/s41587-023-02079-x}
}

@article{srivatsan2020,
    author  = {Srivatsan, Sanjay R. and McFaline-Figueroa, Jos\'e L. and Ramani, Vijay and Saunders, Lauren and Cao, Junyue and Packer, Jonathan and Pliner, Hannah A. and Jackson, Daniel L. and Daza, Riza M. and Christiansen, Lena and Zhang, Fan and Steemers, Frank and Shendure, Jay and Trapnell, Cole},
    title   = {Massively multiplex chemical transcriptomics at single-cell resolution},
    journal = {Science},
    volume  = {367},
    number  = {6473},
    pages   = {45--51},
    year    = {2020},
    doi     = {10.1126/science.aax6234}
}

@article{peidli2024scperturb,
  title={{scPerturb}: harmonized single-cell perturbation data},
  author={Peidli, Stefan and Green, Tessa D and Shen, Ciyue and Gross, Torsten and Min, Joseph and Garda, Samuele and Yuan, Bo and Schumacher, Linus J and Taylor-King, Jake P and Marks, Debora S and Luna, Augustin and Bl{\"u}thgen, Nils and Sander, Chris},
  journal={Nature Methods},
  volume={21},
  pages={531--540},
  year={2024}
}

@article{klein2025cellflow,
  title={{CellFlow} enables generative single-cell phenotype modeling with flow matching},
  author={Klein, Dominik and others},
  journal={bioRxiv},
  year={2025},
  doi={10.1101/2025.04.11.648220},
  note={Preprint, doi:10.1101/2025.04.11.648220}
}

@misc{agrawalla2025floqtrainingcriticsflowmatching,
      title={floq: Training Critics via Flow-Matching for Scaling Compute in Value-Based RL}, 
      author={Bhavya Agrawalla and Michal Nauman and Khush Agrawal and Aviral Kumar},
      year={2025},
      eprint={2509.06863},
      archivePrefix={arXiv},
      primaryClass={cs.LG},
      url={https://arxiv.org/abs/2509.06863}, 
}

@misc{wang2022probabilisticconformalpredictionusing,
      title={Probabilistic Conformal Prediction Using Conditional Random Samples}, 
      author={Zhendong Wang and Ruijiang Gao and Mingzhang Yin and Mingyuan Zhou and David M. Blei},
      year={2022},
      eprint={2206.06584},
      archivePrefix={arXiv},
      primaryClass={stat.ML},
      url={https://arxiv.org/abs/2206.06584}, 
}

@misc{park2025ogbenchbenchmarkingofflinegoalconditioned,
      title={OGBench: Benchmarking Offline Goal-Conditioned RL}, 
      author={Seohong Park and Kevin Frans and Benjamin Eysenbach and Sergey Levine},
      year={2025},
      eprint={2410.20092},
      archivePrefix={arXiv},
      primaryClass={cs.LG},
      url={https://arxiv.org/abs/2410.20092}, 
}

@article{silverman1984spline,
  title={Spline smoothing: The equivalent variable kernel method},
  author={Silverman, Bernard W},
  journal={The Annals of Statistics},
  volume={12},
  number={3},
  pages={898--916},
  year={1984},
  doi={10.1214/aos/1176346710}
}

@book{wahba1990spline,
  title={Spline Models for Observational Data},
  author={Wahba, Grace},
  series={CBMS-NSF Regional Conference Series in Applied Mathematics},
  volume={59},
  year={1990},
  publisher={SIAM},
  address={Philadelphia},
  doi={10.1137/1.9781611970128}
}

@book{green1994nonparametric,
  title={Nonparametric Regression and Generalized Linear Models: A Roughness Penalty Approach},
  author={Green, Peter J and Silverman, Bernard W},
  year={1994},
  publisher={Chapman and Hall},
  address={London}
}

@article{cucker2002mathematical,
  title={On the mathematical foundations of learning},
  author={Cucker, Felipe and Smale, Steve},
  journal={Bulletin of the American Mathematical Society},
  volume={39},
  number={1},
  pages={1--49},
  year={2002},
  doi={10.1090/S0273-0979-01-00923-5}
}

@article{caponnetto2007optimal,
  title={Optimal rates for the regularized least-squares algorithm},
  author={Caponnetto, Andrea and De Vito, Ernesto},
  journal={Foundations of Computational Mathematics},
  volume={7},
  number={3},
  pages={331--368},
  year={2007},
  doi={10.1007/s10208-006-0196-8}
}

@book{kress2014linear,
  title={Linear Integral Equations},
  author={Kress, Rainer},
  edition={3},
  series={Applied Mathematical Sciences},
  volume={82},
  year={2014},
  publisher={Springer},
  address={New York},
  doi={10.1007/978-1-4614-9593-2}
}

@book{engl1996regularization,
  title={Regularization of Inverse Problems},
  author={Engl, Heinz W and Hanke, Martin and Neubauer, Andreas},
  series={Mathematics and Its Applications},
  volume={375},
  year={1996},
  publisher={Kluwer Academic Publishers},
  address={Dordrecht}
}

@article{chi2023diffusion,
  title={Diffusion Policy: Visuomotor Policy Learning via Action Diffusion},
  author={Chi, Cheng and Feng, Siyuan and Du, Yilun and Xu, Zhenjia and Cousineau, Eric and Burchfiel, Benjamin and Song, Shuran},
  journal={arXiv preprint arXiv:2303.04137},
  year={2023}
}

@misc{lipman2023flow,
      title={Flow Matching for Generative Modeling}, 
      author={Yaron Lipman and Ricky T. Q. Chen and Heli Ben-Hamu and Maximilian Nickel and Matt Le},
      year={2023},
      eprint={2210.02747},
      archivePrefix={arXiv},
      primaryClass={cs.LG},
      url={https://arxiv.org/abs/2210.02747}, 
}

@inproceedings{albergo2023building,
  title={Building Normalizing Flows with Stochastic Interpolants},
  author={Albergo, Michael S. and Vanden-Eijnden, Eric},
  booktitle={International Conference on Learning Representations},
  year={2023}
}

@misc{liu2023flow,
      title={Flow Straight and Fast: Learning to Generate and Transfer Data with Rectified Flow}, 
      author={Xingchao Liu and Chengyue Gong and Qiang Liu},
      year={2022},
      eprint={2209.03003},
      archivePrefix={arXiv},
      primaryClass={cs.LG},
      url={https://arxiv.org/abs/2209.03003}, 
}

@inproceedings{chen2018neural,
  title={Neural ordinary differential equations},
  author={Chen, Ricky T. Q. and Rubanova, Yulia and Bettencourt, Jesse and Duvenaud, David K.},
  booktitle={Advances in Neural Information Processing Systems},
  volume={31},
  year={2018}
}

@inproceedings{rahaman2019spectral,
  title={On the Spectral Bias of Neural Networks},
  author={Rahaman, Nasim and Baratin, Aristide and Arpit, Devansh and Draxler, Felix and Lin, Min and Hamprecht, Fred A. and Bengio, Yoshua and Courville, Aaron},
  booktitle={International Conference on Machine Learning},
  pages={5301--5310},
  year={2019},
  organization={PMLR}
}

@article{xu2020frequency,
  title={Frequency Principle: {F}ourier Analysis Sheds Light on Deep Neural Networks},
  author={Xu, Zhi-Qin John and Zhang, Yaoyu and Luo, Tao and Xiao, Yanyang and Ma, Zheng},
  journal={Communications in Computational Physics},
  volume={28},
  number={5},
  pages={1746--1767},
  year={2020}
}

@inproceedings{tancik2020fourier,
  title={Fourier features let networks learn high frequency functions in low dimensional domains},
  author={Tancik, Matthew and Srinivasan, Pratul and Mildenhall, Ben and Fridovich-Keil, Sara and Raghavan, Nithin and Singhal, Utkarsh and Ramamoorthi, Ravi and Barron, Jonathan and Ng, Ren},
  booktitle={Advances in Neural Information Processing Systems},
  volume={33},
  pages={7537--7547},
  year={2020}
}

@inproceedings{jacot2018neural,
  title={Neural Tangent Kernel: Convergence and Generalization in Neural Networks},
  author={Jacot, Arthur and Gabriel, Franck and Hongler, Cl{\'e}ment},
  booktitle={Advances in Neural Information Processing Systems},
  volume={31},
  year={2018}
}

@inproceedings{scholkopf2001generalized,
  title={A Generalized Representer Theorem},
  author={Sch{\"o}lkopf, Bernhard and Herbrich, Ralf and Smola, Alex J.},
  booktitle={Computational Learning Theory},
  pages={416--426},
  year={2001},
  publisher={Springer}
}

@article{benton2023error,
  title={Error bounds for flow matching methods},
  author={Benton, Joe and Deligiannidis, George and Doucet, Arnaud},
  journal={arXiv preprint arXiv:2305.16860},
  year={2023}
}

@inproceedings{blundell2015weight,
  title={Weight uncertainty in neural network},
  author={Blundell, Charles and Cornebise, Julien and Kavukcuoglu, Koray and Wierstra, Daan},
  booktitle={International conference on machine learning},
  pages={1613--1622},
  year={2015},
  organization={PMLR}
}

@article{chung2021uncertainty,
  title={Uncertainty toolbox: an open-source library for assessing, visualizing, and improving uncertainty quantification},
  author={Chung, Youngseog and Char, Ian and Guo, Han and Schneider, Jeff and Neiswanger, Willie},
  journal={arXiv preprint arXiv:2109.10254},
  year={2021}
}

@article{cui2020calibrated,
  title={Calibrated reliable regression using maximum mean discrepancy},
  author={Cui, Peng and Hu, Wenbo and Zhu, Jun},
  journal={Advances in Neural Information Processing Systems},
  volume={33},
  pages={17164--17175},
  year={2020}
}

@inproceedings{gal2016dropout,
  title={Dropout as a bayesian approximation: Representing model uncertainty in deep learning},
  author={Gal, Yarin and Ghahramani, Zoubin},
  booktitle={international conference on machine learning},
  pages={1050--1059},
  year={2016},
  organization={PMLR}
}

@article{lakshminarayanan2017simple,
  title={Simple and scalable predictive uncertainty estimation using deep ensembles},
  author={Lakshminarayanan, Balaji and Pritzel, Alexander and Blundell, Charles},
  journal={Advances in neural information processing systems},
  volume={30},
  year={2017}
}

@article{niedoba2024towards,
  title={Towards a mechanistic explanation of diffusion model generalization},
  author={Niedoba, Matthew and Zwartsenberg, Berend and Murphy, Kevin and Wood, Frank},
  journal={arXiv preprint arXiv:2411.19339},
  year={2024}
}

@inproceedings{shen2023non,
  title={Non-autoregressive conditional diffusion models for time series prediction},
  author={Shen, Lifeng and Kwok, James},
  booktitle={International Conference on Machine Learning},
  pages={31016--31029},
  year={2023},
  organization={PMLR}
}

@article{tashiro2021csdi,
  title={Csdi: Conditional score-based diffusion models for probabilistic time series imputation},
  author={Tashiro, Yusuke and Song, Jiaming and Song, Yang and Ermon, Stefano},
  journal={Advances in neural information processing systems},
  volume={34},
  pages={24804--24816},
  year={2021}
}

@article{yi2023generalization,
  title={On the generalization of diffusion model},
  author={Yi, Mingyang and Sun, Jiacheng and Li, Zhenguo},
  journal={arXiv preprint arXiv:2305.14712},
  year={2023}
}
